\documentclass[11pt,twoside]{article} 

\usepackage[utf8]{inputenc} 
\usepackage[T1]{fontenc}    
\usepackage{booktabs}       
\usepackage{amsfonts}       
\usepackage{nicefrac}       
\usepackage{microtype}      

\usepackage{subfigure}
\usepackage{epsf}
\usepackage{epsfig}
\usepackage{fancyhdr}
\usepackage{graphics}
\usepackage{graphicx}
\usepackage{psfrag}
\usepackage{fullpage}
\usepackage{pdfpages}

\usepackage{microtype}
\usepackage{graphicx}
\usepackage{subfigure}
\usepackage{booktabs} 
\usepackage{amsmath}
\usepackage{amsfonts} 

\usepackage{epsf}
\usepackage{fancyhdr}
\usepackage{graphics}
\usepackage{graphicx}
\usepackage{microtype}
\usepackage{subfigure}
\usepackage{mathtools, algorithmic}

\usepackage[linesnumbered,ruled]{algorithm2e}
\DontPrintSemicolon
\usepackage{color}
\usepackage{amsthm}
\usepackage{amsfonts}
\usepackage{amsmath}
\usepackage{amssymb,bbm}
\usepackage{caption}
\usepackage{sidecap}
\sidecaptionvpos{figure}{c}

\usepackage{url}
\usepackage[colorlinks=True,linkcolor=magenta,citecolor=blue,urlcolor=blue,pagebackref=true,backref=true]
{hyperref}
\renewcommand*{\backref}[1]{\ifx#1\relax \else Page #1 \fi}
\renewcommand*{\backrefalt}[4]{%
    \ifcase #1 \footnotesize{(Not cited.)}%
    \or        \footnotesize{(Cited on page~#2.)}%
    \else      \footnotesize{(Cited on pages~#2.)}%
    \fi}
    
\newcommand{\ones}{\mathbf{1}}
\newcommand{\zeros}{\mathbf{0}}

\newcommand{\KL}{\mathbf{KL}}
\newtheorem{assumption}{Assumption}
\newtheorem{remark}{Remark}

\newtheorem{lemma}{Lemma}
\newtheorem{theorem}{Theorem}

\newtheorem{definition}{Definition}

\DeclareMathOperator*{\argmin}{arg\,min}

\newcommand{\Br}{\mathbb{R}}
\newcommand{\one}{\mathbf{1}}

\def\RR{\mathbb{R}}

\def\Dd{\mathcal{D}}

\def\Ii{\mathcal{I}}
\def\Ww{\mathcal{W}}

\def\Nn{\mathcal{N}}
\def\Xx{\mathcal{X}}
\def\Zz{\mathcal{Z}}

\def\Mm{\mathcal{M}}

\def\Uu{\mathcal{U}}

\def\aA{\mathbf{a}}
\def\bB{\mathbf{b}}
\def\xX{\mathbf{x}}
\def\yY{\mathbf{y}}
\def\pP{\mathbf{p}}

\def\uU{\mathbf{u}}
\def\vV{\mathbf{v}}
\def\yY{\mathbf{y}}

\def\pP{\mathbf{p}}

\newcommand{\br}{\mathbb{R}}

\newcommand{\ba}{\begin{array}}
\newcommand{\ea}{\end{array}}

\newcommand{\Rspace}{\mathbb{R}}
\newcommand{\bigO}{\mathcal{O}}
\newcommand{\bigOtil}{\widetilde{\mathcal{O}}}

\newcommand{\Rrot}{R}
\newcommand{\Urot}{U_{\textnormal{rot}}}

\newcommand{\Rrsot}{R}
\newcommand{\Ursot}{U_{\textnormal{rsot}}}

\newcommand{\Rrsbp}{R_{\textnormal{rsbp}}}
\newcommand{\Ursbp}{U_{\textnormal{rsbp}}}

\newcommand{\arsots}{a_{\textnormal{rsot}}^*}
\newcommand{\arsotk}{a_{\textnormal{rsot}}^k}

\newcommand{\arsotsi}{(a_{\textnormal{rsot}}^*)_i}
\newcommand{\arsotki}{(a_{\textnormal{rsot}}^k)_i}

\newcommand{\brsots}{b_{\textnormal{rsot}}^*}
\newcommand{\brsotk}{b_{\textnormal{rsot}}^k}

\newcommand{\brsotsj}{(b_{\textnormal{rsot}}^*)_j}
\newcommand{\brsotkj}{(b_{\textnormal{rsot}}^k)_j}

\newcommand{\buots}{b_{\textnormal{uot}}^*}
\newcommand{\buotk}{b_{\textnormal{uot}}^k}
\newcommand{\buotsj}{(b_{\textnormal{uot}}^*)_j}
\newcommand{\buotkj}{(b_{\textnormal{uot}}^k)_j}

\newcommand{\ubki}{u^k_i}
\newcommand{\ubsi}{u^*_i}

\newcommand{\ubsij}{(u^*_i)_j}
\newcommand{\ubkij}{(u^k_i)_j}

\newcommand{\vbki}{v^k_i}
\newcommand{\vbsi}{v^*_i}

\newcommand{\vbsil}{(v^*_i)_l}
\newcommand{\vbkil}{(v^k_i)_l}

\newcommand{\Xbki}{X^k_i}
\newcommand{\Xbsi}{X^*_i}
\newcommand{\Xbhi}{\widehat{X}_i}
\newcommand{\Xbpki}{\bar{X}^{k}_i}
\newcommand{\Xbpsi}{\bar{X}^{*}_i}

\newcommand{\Xbk}{\bold{X}^k}
\newcommand{\Xbs}{\bold{X}^*}
\newcommand{\Xbh}{\widehat{\bold{X}}}
\newcommand{\Xbpk}{\bar{\bold{X}}^{k}}
\newcommand{\Xbps}{\bar{\mathbf{X}}^{*}}

\newcommand{\xbpk}{\bar{x}^{k}}
\newcommand{\xbps}{\bar{x}^{*}}

\newcommand{\absij}{(a^*_i)_j}
\newcommand{\abkij}{(a^k_i)_j}

\newcommand{\ursot}{u_{\textnormal{rsot}}}
\newcommand{\ursots}{u_{\textnormal{rsot}}^*}
\newcommand{\ursotk}{u_{\textnormal{rsot}}^k}
\newcommand{\ursotko}{u_{\textnormal{rsot}}^{k + 1}}
\newcommand{\ursotsi}{(u_{\textnormal{rsot}}^*)_i}
\newcommand{\ursotki}{(u_{\textnormal{rsot}}^k)_i}
\newcommand{\ursotkoi}{(u_{\textnormal{rsot}}^{k+1})_i}

\newcommand{\vrsot}{v_{\textnormal{rsot}}}
\newcommand{\vrsots}{v_{\textnormal{rsot}}^*}
\newcommand{\vrsotk}{v_{\textnormal{rsot}}^k}
\newcommand{\vrsotko}{v_{\textnormal{rsot}}^{k + 1}}
\newcommand{\vrsotsj}{(v_{\textnormal{rsot}}^*)_j}
\newcommand{\vrsotkj}{(v_{\textnormal{rsot}}^k)_j}

\newcommand{\uuots}{u_{\textnormal{uot}}^*}
\newcommand{\uuotk}{u_{\textnormal{uot}}^k}

\newcommand{\vuots}{v_{\textnormal{uot}}^*}
\newcommand{\vuotk}{v_{\textnormal{uot}}^k}

\newcommand{\vuotsj}{(v_{\textnormal{uot}}^*)_j}
\newcommand{\vuotkj}{(v_{\textnormal{uot}}^k)_j}

\newcommand{\Xrot}{X_{\textnormal{rot}}}
\newcommand{\Xroth}{\widehat{X}_{\textnormal{rot}}}
\newcommand{\Xrots}{X_{\textnormal{rot}}^*}
\newcommand{\Xrotk}{X_{\textnormal{rot}}^k}

\newcommand{\Xrsot}{X_{\textnormal{rsot}}}
\newcommand{\Xrsoth}{\widehat{X}_{\textnormal{rsot}}}
\newcommand{\Xrsots}{X_{\textnormal{rsot}}^*}
\newcommand{\Xrsotk}{X_{\textnormal{rsot}}^k}

\newcommand{\Xrsotsij}{(X_{\textnormal{rsot}}^*)_{ij}}

\newcommand{\Xuots}{X_{\textnormal{uot}}^*}
\newcommand{\Xuotk}{X_{\textnormal{uot}}^k}
\newcommand{\Xuoth}{\widehat{X}_{\textnormal{uot}}}

\newcommand{\xuots}{x_{\textnormal{uot}}^*}
\newcommand{\xuotk}{x_{\textnormal{uot}}^k}

\newcommand{\frot}{f_{\textnormal{rot}}}
\newcommand{\grot}{g_{\textnormal{rot}}}

\newcommand{\frsot}{f_{\textnormal{rsot}}}
\newcommand{\grsot}{g_{\textnormal{rsot}}}
\newcommand{\hrsot}{h_{\textnormal{rsot}}}

\newcommand{\frsbp}{f_{\textnormal{rsbp}}}
\newcommand{\grsbp}{g_{\textnormal{rsbp}}}
\newcommand{\hrsbp}{h_{\textnormal{rsbp}}}

\newcommand{\drsotk}{\Delta_{\text{rsot}}^k}

\newcommand{\Nystrom}{\text{Nystr\"{o}m} }

\newcommand{\niter}{n_{\textnormal{iter}}}

\newcommand{\mnorm}[1]{\|{#1}\|_{\scriptscriptstyle \infty}}
\newcommand{\onorm}[1]{\|#1\|_{1}}

\begin{document}

\begin{center}

{\bf{\LARGE{On Robust Optimal Transport: Computational Complexity and Barycenter Computation}}}
  
\vspace*{.2in}
{\large{
\begin{tabular}{cccc}
Khang Le$^{\star, \dagger}$ & Huy Nguyen$^{\star, \diamond}$  & Quang Minh Nguyen$^{\flat}$ & Tung Pham$^{\diamond}$
\end{tabular}
\begin{tabular}{cc}
Hung Bui$^{\diamond}$ & Nhat Ho$^{\dagger}$
\end{tabular}
}}

\vspace*{.2in}

\begin{tabular}{c}
VinAI Research, Vietnam$^\diamond$; Massachusetts Institute of Technology$^\flat$; \vspace*{-2mm}\\University of Texas, Austin$^{\dagger}$\\
\end{tabular}


\vspace*{.2in}

\begin{abstract}
We consider robust variants of the standard optimal transport, named robust optimal transport, where marginal constraints are relaxed via Kullback-Leibler divergence. We show that Sinkhorn-based algorithms can approximate the optimal cost of robust optimal transport in $\widetilde{\mathcal{O}}(\frac{n^2}{\varepsilon})$ time, in which $n$ is the number of supports of the probability distributions and $\varepsilon$ is the desired error. Furthermore, we investigate a fixed-support robust barycenter problem between $m$ discrete probability distributions with at most $n$ number of supports and develop an approximating algorithm based on iterative Bregman projections (IBP). For the specific case $m = 2$, we show that this algorithm can approximate the optimal barycenter value in $\widetilde{\mathcal{O}}(\frac{mn^2}{\varepsilon})$ time, thus being better than the previous complexity $\widetilde{\mathcal{O}}(\frac{mn^2}{\varepsilon^2})$ of the IBP algorithm for approximating the Wasserstein barycenter.
\end{abstract}
\let\thefootnote\relax\footnotetext{$\star$ Khang Le and Huy Nguyen
  contributed equally to this work.}
 
\end{center}
\section{Introduction}
The recent advance in computation with optimal transport (OT) problem~\cite{Cuturi-2013-Sinkhorn, Altschuler-2017-Near, Dvurechensky-2018-Computational, Blanchet-2018-Towards, lin2019efficient, peyre2019computational, Lahn-2019-Graph} has led to a surge of interest in using that tool in various domains of machine learning and statistics. The range of its applications is broad, including deep generative models~\cite{Arjovsky-2017-Wasserstein, Gulrajani_2017, tolstikhin2018wasserstein}, scalable Bayes~\cite{Srivastava-2015-WASP, Srivastava-2018-Scalable}, mixture and hierarchical models~\cite{Nguyen-13}, and other applications~\cite{Solomon_2015, Rolet-2016-Fast, Courty-2017-Optimal, Ho-ICML-2017, xu2019scalable, titouan2020co, chen2020graph}.

The goal of optimal transport is to find a minimal cost of moving masses between (supports of) probability distributions. It is known that the estimation of transport cost is not robust when there are  outliers. To deal with this  issue, \cite{Matran_2008} proposed a trimmed version of optimal transport. In particular, they search for truncated probability distributions  such that the transport cost between them is minimized. However, their trimmed optimal transport is non-trivial to compute, which hinders its usage in practical applications. Another line of works proposed using unbalanced optimal transport (UOT) to solve the sensitivity of optimal transport to outliers~\cite{Balaji_Robust, Peyre_2020}. More specifically, their idea is to assign as small as possible masses to outliers  by  relaxing the marginal constraints of OT through a penalty function such as the Kullback-Leibler (KL) divergence. This direction of robust optimal transport has been shown to have good performance in generative models and domain adaptation~\cite{Balaji_Robust}. Although this approach achieved considerable success, the full picture of its computational complexity has remained missing.

\textbf{Our Contribution:} In the paper, we provide a comprehensive study of the computational complexity of robust optimal transport and its corresponding barycenter problem when the probability distributions are discrete and have at most $n$ components. Our contribution is twofold and can be summarized as follows:

\begin{itemize}
    \item[(1)] \textbf{On robust optimal transport,} we consider two versions corresponding to two ways of relaxing marginal constraints in the standard optimal transport problem via the KL divergence. We show that two scaling algorithms computing these robust formulations have the complexities $\bigOtil(n^2/ \varepsilon)$, where $\varepsilon$ denotes the desired error for the computed cost. These complexities are lower than the complexity of the Sinkhorn algorithm for solving the optimal transport problem, which is $\bigOtil(n^2/\varepsilon^2)$~\cite{Dvurechensky-2018-Computational}, and match the complexity of the Sinkhorn algorithm that solves the UOT problem~\cite{pham2020unbalanced}. Furthermore, we show how the above complexity can be improved by utilizing the low-rank approximation method to speed up the matrix-vector computations in the loop similar to \cite{Altschuler-2018-Massively}, and obtain the improved computing time of $\widetilde{O}(nr^2 + \frac{nr}{\varepsilon})$, where $r$ is the approximated rank.
    \item[(2)] \textbf{On robust barycenter problem,} where the goal is to determine a probability measure that minimizes its robust optimal cost to a given set of $m \ge 2$ probability measures, we propose \textsc{RobustIBP} algorithm for solving the robust barycenter problem, which is inspired by the iterative Bregman projection (IBP) algorithm for solving the traditional barycenter problem~\cite{Benamou-2015-Iterative}. We show that when $m = 2$, the complexity of \textsc{RobustIBP} algorithm is at the order of $\bigOtil(mn^2/ \varepsilon)$, better than that of the IBP algorithm for solving the traditional barycenter problem~\cite{Kroshnin-2019-Complexity}, which is $\bigOtil(mn^2/\varepsilon^2)$. To the best of our knowledge, the \textsc{RobustIBP} is also the first practical algorithm obtaining the near-optimal complexity $\bigOtil(mn^2/\varepsilon)$ for solving the barycenter problem even under only the setting $m = 2$.
\end{itemize}

\textbf{Organization:} The paper is organized as follows. In Section~\ref{sec:preliminary}, we provide the background on the optimal transport problem and some of its variants that have robust effects. In Section~\ref{sec:robust_optimal_transport}, we discuss in-depth the variant where only one marginal constraint is relaxed, study the computational complexity of a Sinkhorn-based algorithm that solves it, and then briefly introduce the fully-relaxed formulation. We also establish the complexities of these algorithms after applying \Nystrom method. Subsequently, we present our study of the robust barycenter problem in Section~\ref{sec:robust_semi_barycenter}. In Section~\ref{sec:experiment}, we carry out empirical studies to illustrate the theories before concluding with a few discussions in Section~\ref{sec:conclusion}. The proofs of our theoretical results are in the supplementary material.

\textbf{Notation:} We let $[n]$ stand for the set $\{1, 2, \ldots, n\}$ while $\Rspace^n_+$ indicates the set of all vectors with non-negative entries. For a vector $x \in \Rspace^n$ and $p\in [1,\infty)$, we denote $\|x\|_p$ as its $\ell_p$-norm and $\text{diag}(x)$ as the diagonal matrix with $x$ on the diagonal. The natural logarithm of a vector $\aA = (a_1,..., a_n) \in \mathbb{R}^n_+$ is denoted by $\log \aA = (\log a_1,..., \log a_n)$, $\one_n$ stands for a vector of length $n$ that all of its entries 
 equal to $1$, and $\partial_x f$ refers to the partial differentiation of function  $f$ with respect to $x$. For any given space $\mathcal{X} \subset \mathbb{R}^{d}$, we denote by $\mathcal{P}(\mathcal{X})$ the space of all probability measures on $\mathcal{X}$. Given an integer $n>0$ and a real number $\varepsilon>0$, the notation $a = \bigO\left(b(n,\varepsilon)\right)$ means that $a \leq C \cdot b(n, \varepsilon)$ where $C$ is independent of $n$ and $\varepsilon$. Meanwhile, the notation $a = \bigOtil(b(n, \varepsilon))$ indicates the previous inequality may depend on a logarithmic function of $n$ and $\varepsilon$. For any two probability measures $\xX = (x_1,\ldots,x_n)$ and $\yY = (y_1, \ldots, y_n)$ with the same supports, the generalized Kullback-Leibler divergence is defined as $ \KL(\xX \|\yY) = \sum_{i=1}^n \big[x_i \log\big(\frac{x_i}{y_i} \big) - x_i + y_i \big]$. Finally, the entropy of a matrix $X$ is given by $H(X) = \sum_{i, j=1}^n - X_{ij} (\log X_{ij} - 1)$.
\section{Background on Optimal Transport}
\label{sec:preliminary}
In this section, we review optimal transport and its unbalanced formulation, then from that deriving formulations for robust optimal transport. For any $P$ and $Q$ in $\mathcal{P}(\mathcal{X})$ for a space $\mathcal{X}$, the OT distance between $P$ and $Q$ takes the following form
\begin{align}
    \text{OT}(P, Q) : = \min_{\pi \in \Pi(P, Q)} \int c(x, y) d \pi(x, y), \label{eq:OT_formulation}
\end{align}
where $\Pi(P, Q)$ is the set of joint probability distributions in $\mathcal{X} \times \mathcal{X}$ such that their marginal distributions are $P$ and $Q$, and  $c: \Xx \times \Xx \to [0,\infty)$ is a cost function.

\textbf{Unbalanced Optimal Transport:} When $P$ or $Q$ is not a probability distribution, the OT formulation between $P$ and $Q$ in equation~\eqref{eq:OT_formulation} is no longer valid. One solution to this issue is using the unbalanced optimal transport (UOT) \cite{chizat2018scaling}, which is given by:
\begin{align}
    \text{UOT}(P, Q) : = \min_{\pi \in \Mm_+(\mathcal{X} \times \mathcal{X})} \int c(x, y) d \pi(x, y) + \tau_{1} \KL(\pi_{1}\| P) + \tau_{2} \KL(\pi_{2}\| Q), \label{eq:UOT_formulation}
\end{align}
where $\Mm_+(\mathcal{X} \times \mathcal{X})$ denotes the set of joint non-negative measures on the space $\mathcal{X} \times \mathcal{X}$; $\pi_{1}, \pi_{2}$ are the marginal distributions of $\pi$ and respectively correspond to $P$ and $Q$; $\tau_{1}, \tau_{2}$ are regularized positive parameters. Note that, we can replace the KL divergence in equation~\eqref{eq:UOT_formulation} by any Csisz\'{a}r-divergence \cite{Csiszar-67}. However, we only consider the case of KL divergence in this work.

\textbf{Robust Optimal Transport:} Optimal transport is well-known for not being robust in the present of outliers. A way to deal with this issue is using the approach of  unbalanced optimal transport (UOT), which has demonstrated favorable practical performance in generative models and domain adaptation ~\cite{Balaji_Robust}. More specifically, when $P$ and $Q$ are probability distributions in $\mathcal{X}$, the \textit{\textbf{R}obust Unconstrained \textbf{O}ptimal \textbf{T}ransport (ROT)} admits the following form
\begin{align}
    \text{ROT}(P, Q) : = \inf_{P_{1}, Q_{1} \in \mathcal{P}(\mathcal{X})} \min_{\pi \in \Pi(P_{1}, Q_{1})} \int c(x, y) d \pi(x, y) + \tau_{1} \KL(P_{1}\| P) + \tau_{2} \KL(Q_{1}\| Q), \label{eq:uncons_robust_OT}
\end{align}
where $\tau_{1}, \tau_{2} > 0$ are some given regularized parameters.
The reason to name it robust unconstrained optimal transport is that instead of looking for an optimal transport plan moving masses from $P$ to $Q$, we seek another plan that optimally transports masses between their approximations, which are probability measures $P_1$ and $Q_1$, under the KL divergence. This formulation is closely related to the ones studied in \cite{Balaji_Robust} and \cite{pmlr-v139-mukherjee21a}: the former used $\chi^2$-divergence for the relaxation and the latter used total variation distance (note that those three divergences all together belong to the family of $f$-divergence).

By  relaxing only one marginal constraint regarding (presumably) on $P$, we have another version of ROT, named \textit{\textbf{R}obust \textbf{S}emi-constrained \textbf{O}ptimal \textbf{T}ransport (RSOT)}, which is given by
\begin{align}
    \text{RSOT}(P, Q) : = \inf_{P_{1} \in \mathcal{P}(\mathcal{X})} \min_{\pi \in \Pi(P_{1}, Q)} \int c(x, y) d \pi(x, y) + \tau \KL(P_{1}\| P), \label{eq:semicons_robust_OT}
\end{align}
where $\tau > 0$ is a regularized parameter. We could also define $\textnormal{RSOT}(Q,P) $ similarly 
with a remark that although $\text{RSOT}(P, Q)$ can be different from $\text{RSOT}(Q, P)$, the techniques for obtaining the computational complexity of  both are similar.

\textbf{UOT vs ROT/RSOT:} Though the formulations ROT/RSOT and UOT seem to be similar, they serve different purposes. The goal of UOT is to deal with unbalanced measures, thus there is no condition on the ``transport plan''. Hence, the meaning of the optimal plan $\pi$ of UOT problem is dependent on the interpreter. For example, in applications such as  \cite{schiebinger2017reconstruction}, the UOT is used to figure out the developmental trajectory of cells. Meanwhile, ROT/RSOT aim to seek an accurate transport plan between two possibly corrupted probability  distributions. The toy example in Figure \ref{figure:rot_and_uot} illustrates this difference. In particular, the marginals of the ``transport plan" obtained by the latter (see plots $(b), (d)$) are very different from the two original probability measures $\aA, \bB$. On the other hand, the solution of the former leads to good approximations of $\aA$ and $\bB$ (see plots $(a), (c)$) while removing some bumps in both tails which are presumably outliers.

\begin{figure}[!t]
    \centering
    \includegraphics[width=\linewidth]{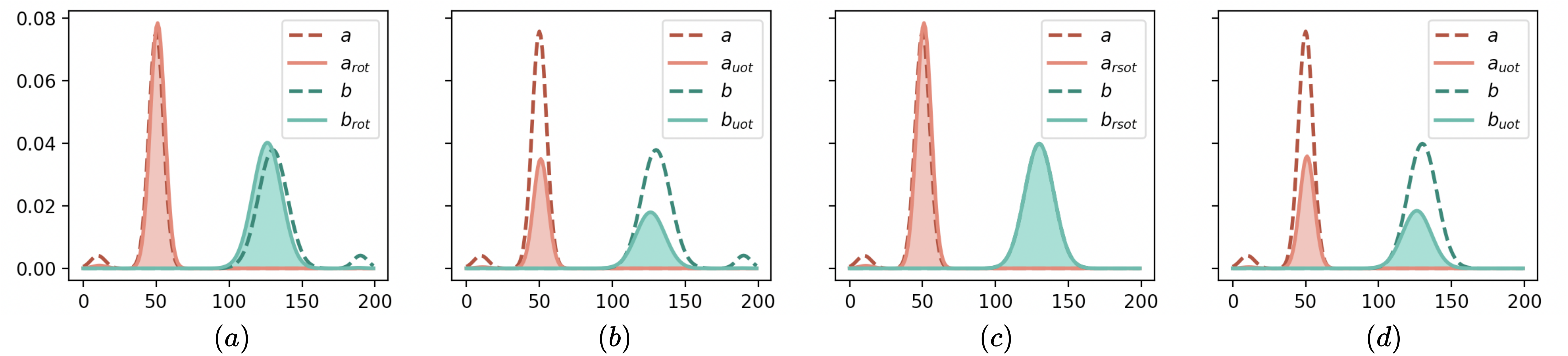}  
    \caption{Comparison on two marginals induced by ROT/RSOT solutions and UOT solutions. Here $\aA, \bB$ are two (possibly corrupted) 1-D Gaussian distributions on which we compute the optimal transport, and $a_{[problem]}, b_{[problem]}$ represent two marginals (with respect to $a$ and $b$ respectively) of the optimal solution for the corresponding $[problem]$. In plots $(a), (b)$, we compare ROT and UOT where both $\aA$ and $\bB$ contain ($10\%$) outliers from other Gaussians, while in plots $(c), (d)$ we investigate RSOT and UOT where only $\aA$ is corrupted.}
    \label{figure:rot_and_uot}
\end{figure}

\section{Discrete Robust Optimal Transport and its Computational Complexity}
\label{sec:robust_optimal_transport}

When $P$ and $Q$ are discrete measures, the KL penalties in equations~\eqref{eq:uncons_robust_OT} and~\eqref{eq:semicons_robust_OT} suggest that the probability distributions $P_{1}$ and $Q_{1}$  need to share the same set of supports as that of $P$ and $Q$, respectively. Therefore, throughout this section, we implicitly require this condition in our formulations of RSOT and ROT and we denote the masses of $P$ and $Q$ by $\aA$ and $\bB$, respectively. 

\subsection{Robust Semi-constrained Optimal Transport}
\label{sec:semi_robust}
Assume that the marginal constraint associating with $Q$ is kept and that of $P$ is relaxed and $P_1$ and $P$ share the same set of supports, the  formulation of RSOT in equation~\eqref{eq:semicons_robust_OT} can be rewritten as follows
\begin{align}
 \min_{X \in \Br_{+}^{n \times n}, X^{\top}\one_{n} = \bB} f_{\text{rsot}}(X):= \left\langle C, X\right\rangle + \tau \KL(X \one_{n} || \aA),
 \label{eq:semi_robust_OT}
\end{align}
where $\aA, \bB$ are the masses of $P$ and $Q$ respectively, and $C$ is the cost matrix whose entries are distances between the supports of these distributions.
Solving directly problem~\eqref{eq:semi_robust_OT} by traditional linear programming solvers can be expensive and not scalable in terms of $n$. Therefore, we utilize the entropic regularization approach proposed by~\cite{Cuturi-2013-Sinkhorn} to the objective function of RSOT, leading to
\begin{align}
\min_{\substack{X \in \Br_{+}^{n \times n}, X^{\top}\one_{n} = \bB}} g_{\text{rsot}}(X) : = f_{\text{rsot}}(X) - \eta H(X). \label{eq:semi_robust_OT_entropic}
\end{align}
Here, $\eta>0$ is a given regularization parameter, and we refer the problem~\eqref{eq:semi_robust_OT_entropic} to as \emph{entropic RSOT}. The dual problem of entropic RSOT is
\begin{align}
     \min_{u, v \in \Br^{n}}h_{\text{rsot}}(u,v):= \eta \onorm{B(u, v)}
     + \tau \big \langle e^{-u/ \tau}, \aA \big \rangle - \big \langle v, \bB \big \rangle,
     \label{eq:semi_robust_ot_entropic_dual}
\end{align}
where $B(u, v)$ is defined as a matrix of size $n\times n$ with entries $[B(u, v)]_{ij} := e^{(u_{i} + v_{j} - C_{ij})/ \eta}$. 
Since equation~\eqref{eq:semi_robust_ot_entropic_dual} is an unconstrained convex optimization problem, we can perform alternating minimization for $u$ and $v$ by setting $\partial h(u, v) / \partial u = 0$ and $\partial h(u, v) / \partial v = 0$, resulting in closed-form updates of a Sinkhorn-like procedure (see~\cite{Cuturi-2013-Sinkhorn}) in Algorithm \ref{algorithm:rsot}. This procedure is known to converge to the optimal solution $(u^*, v^*) := \argmin \hrsot(u, v)$. As strong duality holds for the convex optimization problem \eqref{eq:semi_robust_OT_entropic}, the optimal transport plan of the entropic RSOT is exactly $B(u^*,v^*)$.
\begin{algorithm}[t!]
\caption{\textsc{Robust-SemiSinkhorn}} \label{algorithm:rsot}
\begin{algorithmic}
\STATE \textbf{Input:} $C, \aA, \bB, \eta, \tau, \niter$
\STATE \textbf{Initialization:} $u^0 = v^0 = 0$, $k = 0$
\WHILE {$k < \niter$}
\STATE $a^k \gets B(u^k,v^k) \ones_n, \quad b^k \gets \big(B(u^k,v^k)\big)^{\top} \ones_n$
\IF {$k$ is even}
\STATE $u^{k+1} \leftarrow \frac{\eta \tau}{\eta + \tau} \big[\frac{u^k}{\eta} + \log(\aA) - \log(a^k)\big]$
\STATE $v^{k+1} \leftarrow v^k$
\ELSE
\STATE $u^{k+1} \leftarrow u^k$
\STATE $v^{k+1} \leftarrow \eta \big[\frac{v^k}{\eta} + \log(\bB) - \log(b^k)\big]$
\ENDIF
\STATE $k \gets k + 1$
\ENDWHILE
\STATE \textbf{return} $B(u^k,v^k)$
\end{algorithmic}
\end{algorithm}

Since no assumptions are made on the cost matrix, except its entries are non-negative, closed-form solutions of OT and UOT generally do not exist. Therefore, we introduce the definition of an \textit{$\varepsilon$-approximation} solution of an optimization problem, which will be used for all the subsequent complexity analyses.
\begin{definition}[\textbf{$\varepsilon$-approximation}]
For any $\varepsilon>0$, a transportation plan $X$ is called  an $\varepsilon$-approximation of the minimizer $\widehat{X}$ of some objective function $f$ if $f(X) \le f(\widehat{X }) + \varepsilon$.
\end{definition}
Based on this concept, we then state our main theorem on the runtime complexity of Algorithm \ref{algorithm:rsot} in solving the RSOT problem~\eqref{eq:semi_robust_OT}.
\begin{theorem} 
\label{theorem:rsot}
For $\Ursot:=\max \{3\log(n), \varepsilon / \tau\}$ and $\eta=\varepsilon/ \Ursot$, Algorithm \ref{algorithm:rsot} returns an $\varepsilon$-approximation of the optimal solution $\Xrsoth$ of the problem \eqref{eq:semi_robust_OT} in time 
\begin{align*}
    \bigO\left(\frac{\tau n^2}{\varepsilon}\log(n)\left[\log\left(\frac{\tau\mnorm{C}}{\varepsilon}\right)+\log(\log(n))\right]\right).
\end{align*}
\end{theorem}
\begin{proof}[Proof Sketch]
The full proof of Theorem~\ref{theorem:rsot} is in Appendix~\ref{appendix:rsot}. Note that, this result is not achieved by directly applying Theorem 2 in \cite{pham2020unbalanced} with $\tau_2 \to \infty$ as the nature of the dual function changes in that limit, invalidating many previous results. Let $\Xrsotk$ be the output of Algorithm \ref{algorithm:rsot} at the $k$-th step while $\Xrsoth$ and $\Xrsots$ denotes the minimizers of equations~\eqref{eq:semi_robust_OT} and \eqref{eq:semi_robust_OT_entropic}, respectively. The goal is to find $k$ that guarantees $\frsot(\Xrsotk) - \frsot(\Xrsoth) \le \varepsilon = \eta \Ursot$. We start by decomposing
\begin{align*}
    \underbrace{\frsot(\Xrsotk)}_{\grsot(\Xrsotk) + \eta H(\Xrsotk)} - \underbrace{\frsot(\Xrsoth)}_{\grsot(\Xrsoth) + \eta H(\Xrsoth)} \le \left[\grsot(\Xrsotk)-\grsot(\Xrsots)\right] + \eta \left[H(\Xrsotk)-H(\Xrsoth)\right],
\end{align*}
and try to bound each term by a linear function of $\eta$. Dealing with the entropy term is simple as the $\eta$ factor is already presented, and the entropy difference can be bounded by a constant due to the fact that $1 \le H(X) \le 2\log(n) + 1$ for all $X \in \RR_+^{n \times n}, \|X\|_1 = 1$. The non-trivial part is bounding the difference between $\grsot$ values, which hinges upon two results. The first one is the value of $\grsot$ at optimality:
\begin{align}
    \label{equation:main_text:grsot_optimal}
    \grsot(\Xrsots) &= -\eta-\tau(1-\alpha) + \langle \vrsots, \brsots \rangle.
\end{align}
The second result is the geometric convergence rate of the updates on $u$ and $v$ (Lemma \ref{lemma:rsot:uv_dual:convergence_rate} in Appendix~\ref{appendix:rsot}):
\begin{align*}
    \max\Big\{\|u^{k+1} - u^*\|_{\infty}, \|v^{k+1} - v^*\|_{\infty} \Big\}\le (\text{const}) \Big(\dfrac{\tau}{\tau+\eta}\Big)^{k/2} =: \Delta^k.
\end{align*}
The final step is using equation \eqref{equation:main_text:grsot_optimal} to tailor the $\grsot$ difference to be bounded by a linear function of $\Delta^k$, which is an exponential function of $k$, then solving for the minimum $k$ at which this exponential function is small enough compared to $\eta$. The main technical difficulty here is to deal with the unknown term $\langle \vrsots, \brsots \rangle$ in equation~\eqref{equation:main_text:grsot_optimal}, which causes the deviation from the previous techniques.
\end{proof}
\begin{remark}
The result of Theorem~\ref{theorem:rsot} indicates that the complexity of \textsc{Robust-SemiSinkhorn} algorithm for computing RSOT is at the order of $\bigOtil(\frac{n^2}{\varepsilon})$. This complexity is near-optimal and faster than the complexity of the standard Sinkhorn algorithm for computing the optimal transport problem~\cite{Dvurechensky-2018-Computational, lin2019efficient}, which is at the order of $\bigOtil(\frac{n^2}{\varepsilon^2})$.
\end{remark}

\subsection{Robust Unconstrained Optimal Transport}
\label{sec:rot}
In this section, we briefly present another version of robust optimal transport, abbreviated by ROT, when two distributions are contaminated. We first show that the approach of using the duality of the objective function of ROT problem with entropic regularizer does not produce a Sinkhorn algorithm as in the cases of RSOT and UOT.  However, a second thought of the problem finds an interesting link between the optimal solutions of ROT and UOT, which results in a nice algorithm for the ROT. We also discuss some technical difficulties when analysing the complexity for the ROT problem. At the end of this section, we show that the result could be extended to the case of low-rank cost matrix, which will significantly reduce the computation.

Recall that the masses of $P$ and $Q$ are $\aA$ and $\bB$, respectively, the ROT problem \eqref{eq:uncons_robust_OT} becomes 
\begin{align}
    \label{eq:robust_OT}
    \min_{X \in \Br_{+}^{n \times n}, \onorm{X} = 1}f_{\text{rot}}(X):= \langle C, X\rangle + \tau\KL(X\one_n||\aA) + \tau\KL(X^{\top}\one_n||\bB).
\end{align}
Here we set $\tau_1 = \tau_2 = \tau$ for the sake of simplicity, since there are no more technical difficulties to work with finite $\tau_1\neq \tau_2$. As noted  in Section~\ref{sec:preliminary}, the formulation~\eqref{eq:robust_OT} bears some resemblance to the unbalanced optimal transport problem studied in \cite{pham2020unbalanced}, except the additional norm condition forcing $X$ to be a transportation plan (i.e., a joint probability distribution), which shows the different nature of two problems. Following the approach of using the Sinkhorn algorithm of UOT,  the duality of formulation \eqref{eq:robust_OT} has the form
\begin{align*}
    \eta \log \|B(u,v)\|_1 + \tau \big\{\langle e^{u/\tau},\aA\rangle + \langle e^{v/\tau},\bB\rangle \big\}.
\end{align*}
By taking derivatives of the above function with respect to $u$ and $v$ and set the derivatives to be zero, we obtain
\begin{align*}
    \frac{B(u,v) \mathbf{1}_n}{\|B(u,v)\|_1} = e^{-u/\tau} \odot \aA,\quad \frac{B(u,v)^{\top}\mathbf{1}_n}{\|B(u,v)\|_1} = e^{-v/\tau} \odot \bB,
\end{align*}
where $\odot$ denotes element-wise multiplication. 
Unfortunately, the above equations do not have closed-form solutions to produce update as the Sinkhorn algorithms do because of the term $\|B(u,v)\|_1$ in the denominator. However, the objective function of UOT is not homogeneous with respect to $X$, but could be written as a linear function of ROT and another function of $\|X\|_1$ due to some special properties of the KL divergence. This observation leads to the interesting result summarized in the below lemma.

\begin{lemma} [\textbf{Connections with UOT}]
\label{lemma:rot:Xrots}
The optimal solution of problem \eqref{eq:robust_OT}, denoted $\Xrots$, is the normalized version of $\Xuots$ which is the minimizer of UOT in entropic formulation. More specifically, we have $\Xrots = \frac{\Xuots}{\onorm{\Xuots}}$.
\end{lemma}
The proof of Lemma~\ref{lemma:rot:Xrots} is in Appendix~\ref{appendix:rot:proofs}. Based on this result, we can utilize the Sinkhorn algorithm that solves UOT (see \cite{pham2020unbalanced}) with a normalizing step at the end to produce a solution for the ROT. Although the normalizing step is convenient in finding ROT's solution, it introduces new challenge in the proof compared to that of UOT since the normalizing constant does not have a lower bound. Even so, we are still able to obtain an $\varepsilon$-approximation solution for the ROT in $\widetilde{\mathcal{O}}(n^2/\varepsilon)$ time without any additional constraints on the setting. For more technical details, please refer to Appendix~\ref{appendix:rot:proofs}.

\textbf{Further Improving Complexities by Low-Rank Approximation:}
As a consequence of our complexity analysis, we can show that by using low-rank approximation method studied in \cite{Altschuler-2018-Massively} to the kernel matrix $K:=  \exp (-C /\eta)$, we could further reduce the complexities  of both robust semi/un-constrained optimal transport problem to  $\widetilde{O}(nr^2 + nr / \varepsilon)$ time, given the same $\varepsilon$-approximation and  the approximated-rank $r$. This result is essentially different from the complexity studied in \cite{Altschuler-2018-Massively}, where the $\varepsilon$-approximation is considered regarding the optimal value of the entropic-regularized problem, not the original one in our analysis. For a more detailed discussion, please refer to Appendix \ref{section:nystrom:proofs}.


\section{The Robust Barycenter Problem}
\label{sec:robust_semi_barycenter}
In this section, we consider the problem of computing the barycenter of a set of possibly corrupted probability measures. The semi-constrained formulation arises as a natural candidate for this goal, when potential outliers only appear in the given probability measures and the desired barycenter is the barycenter of the uncontaminated probability measures. In particular, assume that we have $m \geq 2$ discrete probability measures $P_{1}, \ldots, P_{m}$: each has at most $n$ fixed support points and the associated positive weights are given by $\omega_{1}, \ldots, \omega_{m}$ ($\sum_{i = 1}^{m} \omega_{i} = 1$). The barycenter problem then aims to find the probability measure that minimizes $\sum_{i=1}^m \omega_i \text{RSOT}(P_i, P)$, which is a linear combination of RSOT divergence from the barycenter to all given probability measures. We refer it  as \textit{\textbf{R}obust \textbf{S}emi-constrained \textbf{B}arycenter \textbf{P}roblem (RSBP)}. 
In this work, we consider the fixed-support settings where all the probability measures $P_i$ share the same set of support points. 
  This setting had been widely used in the previous works to study the computational complexity of Wasserstein barycenter problem~\cite{Kroshnin-2019-Complexity, Lin-2020-Fast}. Let $\pP_i$ be the mass of probability measure $P_{i}$ for $i \in [m]$, the discrete RSBP reads
\begin{align*}
    \min_{\pP \in \RR_+^n, \onorm{\pP} = 1} \quad \sum_{i = 1}^m \omega_i \Big[ \min_{X_i \in \RR_+^{n \times n} , X_i^\top \ones_n = \pP} \langle C_i, X_i \rangle + \tau \KL(X_i \ones_n \| \pP_i) \Big],
\end{align*}
which is equivalent to
\begin{align}
    \label{problem:rsbp}
    \min_{\mathbf{X} \in \Dd_1(\mathbf{X})} \quad \frsbp(\mathbf{X}):=\sum_{i=1}^m\omega_i\big[\langle C_i, X_i\rangle + \tau \KL(X_i\one_n\|\pP_i)\big],
\end{align}
where $\Dd_1(\mathbf{X}) := \big\{(X_1,\ldots,X_m) : \ X_i \in \RR_+^{n \times n} \ \text{and} \ \|X_{i}\|_{1} = 1  \ \forall i \in [m]; \ X_{i}^{\top} \one_{n} = X_{i + 1}^{\top} \one_{n} \ \forall i \in [m - 1] \big\}$. 
Note that the objective function of RSBP is different from that of Wasserstein barycenter \cite{Kroshnin-2019-Complexity}: here we relax the marginal constraints $X_{i} \one_{n} = \pP_{i}$ by using the KL divergence to deal with the contaminated $P_i$. 
Finally, the constraints $X_{i}^{\top} \one_{n} = X_{i + 1}^{\top} \one_{n} = \pP$ are to guarantee that the transportation plans $X_i$ have one common marginal which turns out to be a feasible barycenter $\pP$.
Similar to RSOT, we consider an entropic-regularized formulation of~\eqref{problem:rsbp}, named \emph{entropic RSBP}:
\begin{align}
    \label{problem:entropic_RSBP}
    \min_{\mathbf{X} \in \Dd_1(\mathbf{X})} g_{\text{rsbp}}(\mathbf{X}) := \sum_{i = 1}^{m} \omega_{i} g_{\text{rsot}}(X_i; \pP_i,C_i).
\end{align}
Since some functions like $\grsot(X)$, depends on some parameters like $C_i$ and $\pP_i$, we sometimes abuse the notation by including these parameters next to variables, e.g., $g_{\text{rsot}}(X_i; C_i, \pP_i)$. 
A general approach to deal with \eqref{problem:entropic_RSBP} is to consider its dual function, which admits the following form: 
\begin{align}
    \min_{\substack{\uU = (u_{1}, \ldots, u_{m}), \vV = (v_{1}, \ldots, v_{m}) \\ \sum_{i = 1}^m \omega_i v_i = 0}} h_{\text{rsbp}}(\uU, \vV):= \sum_{i = 1}^{m} \omega_{i} \big[ \eta \log \onorm{B(u_i, v_i;C_i)} + \tau \big \langle e^{-u_{i}/ \tau}, \pP_{i} \big \rangle \big].
    \label{problem:dual_entropic_rsbp}
\end{align}
We could use the alternating minimization method to find the minimizer of \eqref{problem:dual_entropic_rsbp}. In particular, starting at an initialization $\uU^0$ and $\vV^0$, we update them alternatively as follows: 
\begin{align}
    \label{equation:dual_update}
    \uU^{k + 1} = \argmin_{\uU} \hrsbp(\uU, \vV^{k}), \quad \vV^{k + 1} = \argmin_{\vV: \sum_{i = 1}^m \omega_i v_i = 0} \hrsbp(\uU^{k + 1},\vV).
\end{align}
In some problems (e.g.,~RSOT), closed-form updates can be acquired if the system of equations $\partial \hrsbp(\uU, \vV^{k})/\partial \uU = \mathbf{0} $ and $\partial \hrsbp(\uU^{k},\vV)/\partial \vV = \mathbf{0}$ could be solved exactly by some simple formulas. However, this is not the case with the formulation of $\hrsbp$ in equation~\eqref{problem:dual_entropic_rsbp} because the logarithmic term leads to an intractable system of equations of the partial derivative of $\hrsbp$. Instead, we propose to solve the optimization problem \eqref{problem:entropic_RSBP} via another objective function, whose dual form can be solved effectively by alternating minimization. 
\begin{algorithm}[t!]
\caption{\textsc{RobustIBP}} \label{algorithm:barycenter_semiOT}
\begin{algorithmic}
\STATE \textbf{Input:} $\{C_{i}\}_{i = 1}^{m}, \{\pP_{i}\}_{i = 1}^{m}, \tau, \eta, \niter$
\STATE \textbf{Initialization:} $u_{i}^0 = v_{i}^0 = \zeros_n$ for $i \in [m]$, $k = 0$
\WHILE {$k < \niter$}
\STATE $a_{i}^k \gets B(u_{i}^k,v_{i}^k;C_i) \ones_n; \quad b_{i}^k \gets \big(B(u_{i}^k,v_{i}^k;C_i)\big)^{\top} \ones_n \quad \forall i \in [m]$
\IF {$k$ is even}
\STATE $u_{i}^{k+1} \leftarrow \frac{\eta \tau}{\eta + \tau} \big[\frac{u^k_i}{\eta} + \log(\mathbf{p}_{i}) - \log(a_{i}^k)\big] \quad \forall i \in [m]$
\STATE $v_{i}^{k+1} \leftarrow v_{i}^k \quad \forall i \in [m]$
\ELSE
\STATE $u_{i}^{k+1} \leftarrow u_{i}^k \quad \forall i \in [m]$
\STATE $v_{i}^{k+1} \leftarrow \eta  \left[ \frac{v^k_i}{\eta} - \log(b_{i}^{k}) - \sum_{t = 1}^{m} \omega_{t} (\frac{v^k_t}{\eta} - \log(b_{t}^{k}) ) \right] \quad \forall i \in [m]$
\ENDIF
\STATE $k \gets k + 1$
\ENDWHILE
\STATE $X_{i}^k \leftarrow B(u_{i}^k,v_{i}^k;C_i) \quad \forall i \in [m]$
\STATE \textbf{return} $(X_1^k, \dots, X_m^k)$ for equation~\eqref{problem:entropic_rsbp_no_norm} $\quad$ or $\quad$ $\big(\frac{X^k_{1}}{\onorm{X^k_{1}}}, \ldots, \frac{X^k_{m}}{\onorm{X^k_{m}}}\big)$ for equation~\eqref{problem:entropic_RSBP}.
\end{algorithmic}
\end{algorithm}
\subsection{\textsc{RobustIBP} Algorithm}
We consider a similar problem to the entropic RSBP in~\eqref{problem:entropic_RSBP}, with its feasible set $\Dd(\mathbf{X}) := \{(X_1,\ldots,X_m) : X_i \in \RR_+^{n \times n}, \forall i \in [m]; X_{i}^{\top} \one_{n} = X_{i + 1}^{\top} \one_{n} \forall i \in [m - 1] \}$ which does not have the norm constraint. The primal objective function and its dual are as follows: 
\begin{align}
    &\textbf{Primal:} \quad \min_{\mathbf{X} \in \Dd(\mathbf{X})} g_{\text{rsbp}}(\mathbf{X}) := \sum_{i = 1}^{m} \omega_{i} g_{\text{rsot}}(X_i; \pP_i,C_i), \label{problem:entropic_rsbp_no_norm} \\
    &\textbf{Dual:} \quad \min_{\uU, \vV:\sum_{i = 1}^m \omega_i v_i = \mathbf{0}} \bar{h}_{\text{rsbp}}(\uU, \vV):= \sum_{i = 1}^{m} \omega_{i} \big[ \eta \onorm{B(u_i, v_i; C_i)} +\tau \big \langle e^{-u_{i}/ \tau}, \pP_{i} \big \rangle \big]. \label{problem:dual_entropic_rsbp_no_norm}
\end{align}
The dual formulation \eqref{problem:dual_entropic_rsbp_no_norm} has a closed form updates for $\uU$ and $\vV$.   
Based on these, we develop Algorithm \ref{algorithm:barycenter_semiOT}, namely \textsc{RobustIBP}, since this procedure resembles the iterative Bregman projections studied in \cite{Benamou-2015-Iterative} and \cite{Kroshnin-2019-Complexity}. The updates of $\uU$ and $\vV$ are known to converge to the optimal solution $(\uU^{*}, \vV^{*})$ of the problem \eqref{problem:dual_entropic_rsbp_no_norm}, and strong duality suggests that $\mathbf{X}^* = (B(u_{i}^{*}, v_{i}^{*}; C_i) )_{i = 1}^m$ is the optimal solution of the problem \eqref{problem:entropic_rsbp_no_norm}. Furthermore, there is an intriguing relation between the optimal solution of the problem \eqref{problem:entropic_rsbp_no_norm} to that of the problem \eqref{problem:entropic_RSBP}, presented in the following lemma.
\begin{lemma}
\label{lemma:barycenter_normalize}
Let $\Xbps=(\bar{X}^{*}_1,\ldots,\bar{X}^{*}_m)$ and $\mathbf{X}^* = (X_1^*, \dots, X_n^*)$ be the optimizers of $\grsbp$ with the feasible set $\Dd(\mathbf{X})$ and with the feasible set $\Dd_1(\mathbf{X})$, respectively. Then, $\Xbsi= \dfrac{\Xbpsi}{\|\Xbpsi\|_1}$ for all $i\in [m]$.
\end{lemma}
The proof of Lemma \ref{lemma:barycenter_normalize} is in Appendix \ref{sec:proofs:robust_barycenter}. This result indicates that we can approximate the solution of equation~\eqref{problem:entropic_RSBP} by the solution of equation~\eqref{problem:entropic_rsbp_no_norm}, using the same Algorithm \ref{algorithm:barycenter_semiOT} with an additional normalizing step at the end.

\subsection{Complexity Analysis}
In this section, we provide the analysis of \textsc{RobustIBP} algorithm for obtaining an $\varepsilon$-approximation of the robust semi-constrained barycenter problem \eqref{problem:entropic_RSBP} when $m = 2$. We also discuss the challenges of extending the current proof technique to $m \geq 3$ at the end of this section. First, we present the complexity of the \textsc{RobustIBP} algorithm in the following theorem.
\begin{theorem}
\label{theorem:barycenter}
For $m = 2$ and $\eta=\varepsilon \Ursbp^{-1}$ where $\Ursbp := \max\{2 + 2\log(n),2\varepsilon,3\varepsilon \log(n) / \tau \}$, the \textsc{RobustIBP} algorithm returns an $\varepsilon$-approximation of the optimal solution $(\widehat{X}_1,\ldots,\widehat{X}_m)$ of the RSBP \eqref{problem:rsbp} in time $\displaystyle \bigO\Big(\frac{\tau n^2}{\varepsilon}\log(n)\Big[\log\Big(\tau\sum_{i=1}^m\mnorm{C_i}\Big)+\log\Big(\frac{\log(n)}{\varepsilon}\Big)\Big]\Big)$.
\end{theorem}

\begin{remark}
The complexity $\bigOtil(n^2/\varepsilon)$ of \textsc{RobustIBP} algorithm is near-optimal and better than that of IBP algorithm for solving the Wasserstein barycenter problem, which is $\bigOtil (n^2 / \varepsilon^2)$ when $m = 2$ in ~\cite{Kroshnin-2019-Complexity}. It is also better than the complexity of \textsc{FASTIBP} algorithm in~\cite{Lin-2020-Fast}, which is $\bigOtil (n^{7/3} / \varepsilon^{4/3})$. To the best of our knowledge, the \textsc{RobustIBP} is also the first practical algorithm obtaining the near-optimal complexity $\bigOtil(n^2/\varepsilon)$ for solving the barycenter problem under the setting $m = 2$.
\end{remark}

The main ingredient in the proof of Theorem \ref{theorem:barycenter} is the convergence rate of  vectors $\uU$ and $\vV$ of the problem \eqref{problem:dual_entropic_rsbp_no_norm}, which is captured as follows:
\begin{align}
    \label{convergence:main:rsbp}
    \max \Big\{ \sum_{i = 1}^m \mnorm{\Delta u_i^{k + 1}}, \sum_{i = 1}^m \mnorm{\Delta v_i^{k + 1}} \Big\} \le (\text{constant}) \Big(\frac{\tau}{\tau + \eta}\Big)^{k/2},
\end{align}
where $\Delta u_i^{k} := u_{i}^{k + 1}- u_{i}^{*}$ and $\Delta v_i^{k} := v_{i}^{k + 1}- v_{i}^{*}$. The result can be achieved by alternatively applying two following inequalities. 

For the first inequality, with even $k$,  from the update of $\uU^{k+1}$ in the Algorithm \ref{algorithm:barycenter_semiOT}, we obtain  $\mnorm{\Delta u_i^{k + 1}} \le \frac{\tau}{\tau + \eta} \mnorm{\Delta v_i^k}$.

The second inequality is obtained from the update of $\vV^{k}$ in Algorithm  \ref{algorithm:barycenter_semiOT} as follows:
\begin{align*}
     \sum_{i = 1}^m \mnorm{\Delta v_i^{k}}
    \le   \sum_{i = 1}^m ((m - 2) \omega_i + 1) \mnorm{\Delta u_i^{k - 1}}.
\end{align*}
Thus, when $m = 2$, we can achieve inequality~\eqref{convergence:main:rsbp}, though this approach is inapplicable for the case $m > 2$. For a formal statement regarding the above convergence rate, please refer to Lemma~\ref{lemma:geometric_converge_IBP} in Appendix~\ref{sec:proofs:robust_barycenter}. Note that for $m \ge 3$, the result of Theorem \ref{theorem:barycenter} still holds if $\uU^k$ and $\vV^k$ 
converge at the rate of the order $(\frac{\tau}{\tau + \eta})^{k/2}$. So next we will take a closer look at this case to see whether the rate remains geometric.

\begin{figure}[!t]
    \centering
    \includegraphics[width=\linewidth]{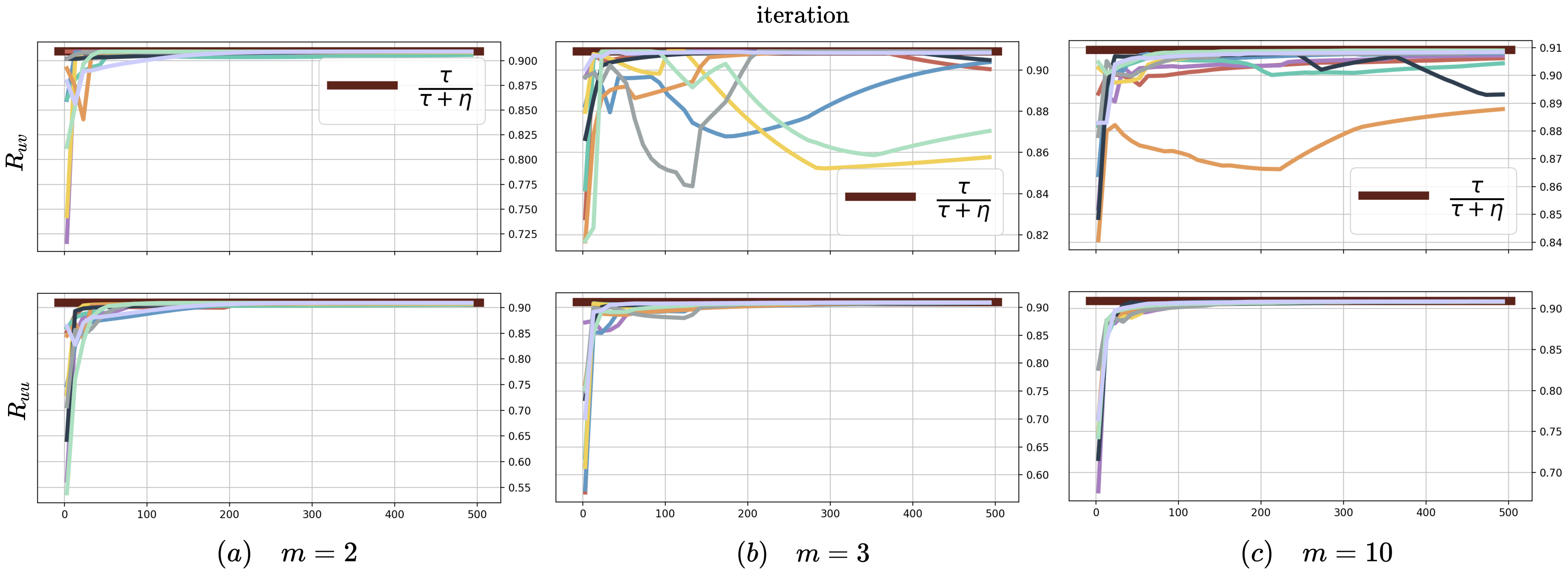}  
    \caption{On the convergence rate of RSBP dual variables when $m \in \{2, 3, 10\}$. Lines with different colors present different runs (with the same values of $\tau = 0.1$ and $\eta = 0.01$). Other parameters are set as follows: $n = 10, C_{i} \sim \Uu[0.01, 1]^{n \times n}$.}
    \label{figure:rsbp:convergence_analysis}
\end{figure}


\textbf{On $m \geq 3$:} In Figure \ref{figure:rsbp:convergence_analysis}, we plot the values of two ratios: $R_{uv} := \frac{\sum_{i = 1}^m \mnorm{\Delta u_i^{k + 1}}}{\sum_{i = 1}^m \mnorm{\Delta v_i^{k}}}$ and $R_{uu} := \frac{\sum_{i = 1}^m \mnorm{\Delta u_i^{k + 1}}}{\sum_{i = 1}^m \mnorm{\Delta u_i^{k - 1}}}$. When $k$ is even, we have that $R_{uu} \le \frac{\tau}{\tau + \eta}$ for all $m$, while the inequality $R_{uv} \le \frac{\tau}{\tau + \eta}$ was only proved for the case $m = 2$. From this figure, both these bounds are true in all considered cases. However, while the bound on $R_{uv}$ (which is theoretically true for all $m$) is only tight when $m = 2$ and seems to be loose in several trials with larger values of $m$, the bound $R_{uu}$ (which is only showed for the case $m = 2$) appears to be tight in all reported scenarios. Thus, we conjecture that the geometric convergence rate at equation~$\eqref{convergence:main:rsbp}$ may still hold for $m$ greater than $2$. 
We leave the case $m \geq 3$ for the future work.
\section{Experiments}
\label{sec:experiment}
In this section, we provide numerical evidences regarding our presented complexities for \textsc{Robust-SemiSinkhorn} and \textsc{Robust-IBP} algorithms. We put additional experiments (including the runtime comparison of ROT/RSOT on synthetic and real datasets, as well as some applications for the studied robust formulations) in Appendix~\ref{sec:additional_experiments}. All the optimal solutions for convex problems in the following part are computed using the \textbf{cvxpy} library~\cite{cvxpy_rewriting}. All the experiments are conducted on a server with 32 GB RAM, 8 cores Intel(R) Core(TM) i7-9700K and 1 GeForce RTX 2080 GPU.
\begin{figure}[!t]
    \centering
    \includegraphics[width=0.99 \linewidth]{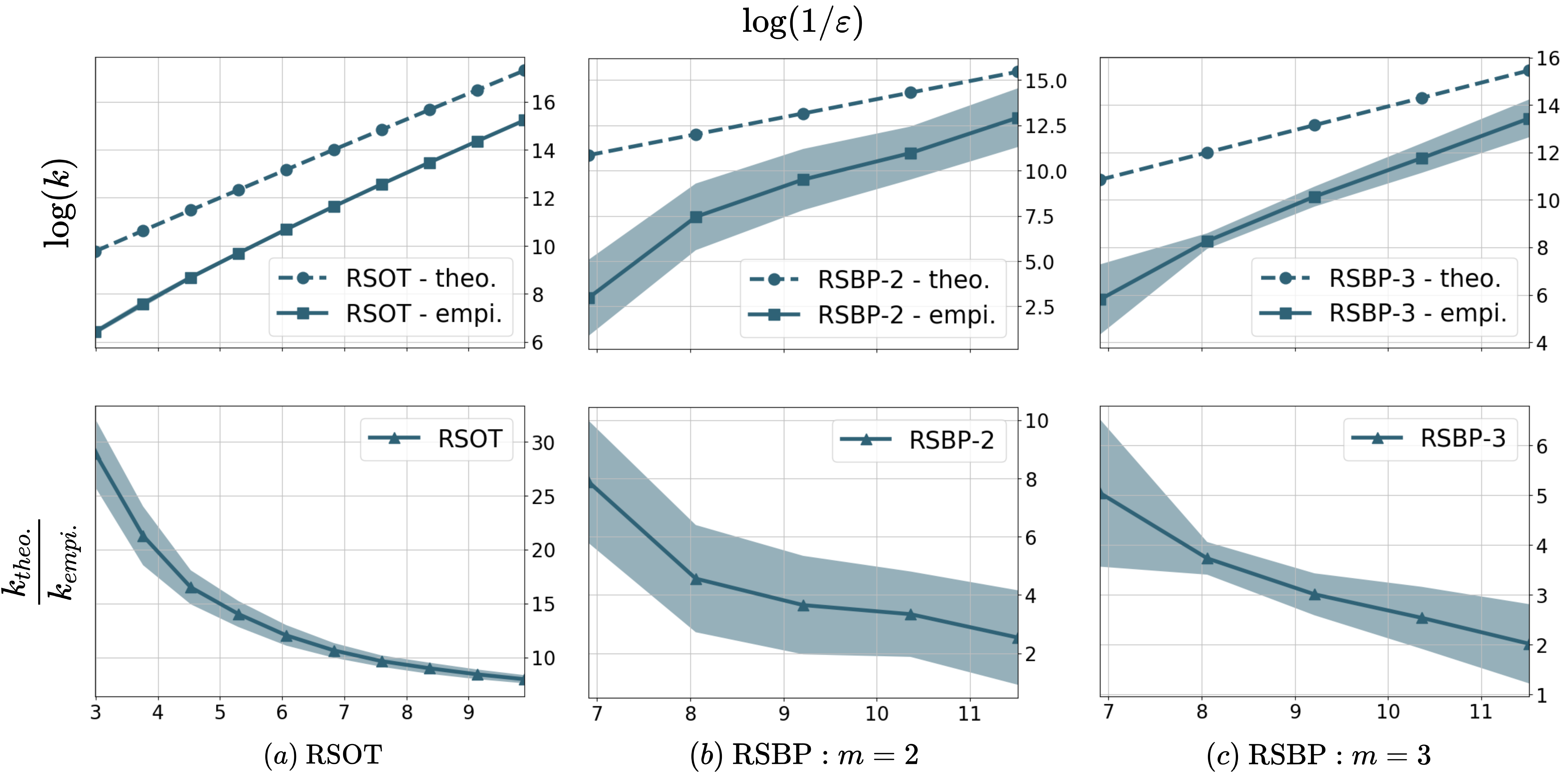}  
    \caption{Runtime demonstration for $(a)$ \textsc{Robust-SemiSinkhorn} and $(b), (c)$ \textsc{Robust-IBP} algorithms. \textit{\textbf{Top}} The log value of the number of iterations computed in our theorems (dashed lines with circle marker) and the true number of iterations at which the algorithms achieve $\varepsilon$-approximations (solid lines with square marker). \textit{\textbf{Bottom:}} The ratio between two values of the upper figures. Both the number of iterations (on the left) and $\varepsilon$ are plotted in the log domain, while the ratios (on the right) are computed with the original values.}
    \label{figure:rot:runtime}
\end{figure}

\textbf{Runtime Demonstration:} For each algorithm, we investigate the number of iterations required to obtain an $\varepsilon$-approximation. We compare the theoretical values in Theorems \ref{theorem:rsot} and \ref{theorem:barycenter} with the empirical values computed by running the corresponding algorithms to obtain the first iterations from where the algorithm always returns an $\varepsilon$-approximation.

\textit{For RSOT}, we let $n = 100, \tau = 1$, generate entries of $C$ uniformly from the interval $[1, 50]$ and draw entries $a, b$  uniformly from $[0.1, 1]$ then normalizing them to form probability vectors. $\eta$ is set according to Theorem \ref{theorem:rsot}. For each $\varepsilon$ varying from $5\times 10^{-2} $ to $5 \times 10^{-5}$, we calculate the number of theoretical and empirical iterations described above, as well as their ratio. This experiment is run $10$ times and we report their mean and standard deviation values in Figure \ref{figure:rot:runtime} $(a)$. We also carry out a similar experiment on MNIST data, which is reported in the Appendix~\ref{sec:additional_experiments}.

\textit{For RSBP}, we run the \textsc{RobustIBP} algorithm with the following setup: $n = 10; \tau = 1$; $\pP_1, \dots$, $\pP_m$, $[\omega_1, \dots, \omega_m]$ are randomly-initialized probability vectors; $\{ C_i \}_{i = 1}^{m}$ is a set of $n \times n$ matrices whose entries drawn uniformly in $[0.01, 0.1]$; five chosen values of $\varepsilon$ vary from $10^{-3}$ to $10^{-5}$ (which are relatively small compared to the optimal cost $\frsbp(\mathbf{X}^*)$ is about $0.019 \pm 0.001$ when $m = 2$ and is about $0.021 \pm 0.001$ when $m = 3$); and the corresponding values of $\eta$ are set according to Theorem \ref{theorem:barycenter}. The results are shown in Figure \ref{figure:rot:runtime} $(b)$ and $(c)$. Note that the complexity for the case $m \ge 3$ is still an open problem, and we use the formula in Theorem \ref{theorem:barycenter} to compute the (hypothetical) theoretical number of iterations in that case.

In all three experiments, it is noticeable that the ratios between theoretical and empirical values decrease as $\varepsilon \to 0$, indicating the our complexity bounds get tighter.

\section{Conclusion}
\label{sec:conclusion}
In the paper, we study the complexity of Sinkhorn-based algorithms for approximately solving robust versions of optimal transport between two discrete probability measures with at most $n$ components, and show that they return $\varepsilon$-approximated solutions in $\bigOtil(n^2 / \varepsilon)$ time. Low-rank approximation technique is also analysed to further reduce the dependency of these complexities on $n$, resulting in $\bigOtil(n r^2 + nr / \varepsilon)$ complexities. Finally, we investigate a robust barycenter problem between $m$ probability measures and develop the IBP-based algorithm for solving it. When $m = 2$, the complexity of the \textsc{RobustIBP} algorithm is proved to be at the order of $\bigOtil(mn^2 / \varepsilon)$, while in the case $m \geq 3$ we believe that a novel proof technique needs to be developed to establish the geometric convergence of the updates from the algorithm. We leave this direction for the future work.

\newpage
\appendix
\begin{center}
    \textbf{\Large{Supplement to ``On Robust Optimal Transport: Computational Complexity and Barycenter Computation''}}
\end{center}

In this supplementary material, we collect several proofs and remaining materials that are deferred from the main paper. In Appendix~\ref{sec:Notations}, we introduce and recall necessary notations for the supplementary material. In Appendix~\ref{appendix:rsot}, we provide key lemmas and proofs for the computational complexity of robust semi-constrained optimal transport (RSOT), and those regarding ROT are in Appendix~\ref{appendix:rot:proofs}. Appendix~\ref{sec:proofs:robust_barycenter} is devoted to the lemmas and proofs for the computational complexity of robust semi-constrained barycenter (RSBP). We provide the proof for computational complexity of robust Sinkhorn algorithms via $\Nystrom$ approximation in Appendix~\ref{section:nystrom:proofs}. Finally, 
we present additional experiment studies with the proposed robust algorithms in Appendix~\ref{sec:additional_experiments}.
\section{Notations}
\label{sec:Notations}
This appendix aims to introduce some notations that will be used intensively in the subsequent parts of the appendix. We start with the meaning of notations for the general case, and those for remaining cases follow similarly (see the table content). First, we denote $f$ and $g$ to be the original objective and the corresponding entropic-regularized objective, respectively, and let $\widehat{X} := \argmin f(X), X^* := \argmin g(X)$. The sum of all elements in $X$ is $x := \onorm{X}$ (similarly, $x^* := \onorm{X^*}$). Regarding Sinkhorn algorithm, $u^k, v^k$ are the updates of the $k$-th iteration. The converged values for $u^k$ and $v^k$ (if exist) are denoted $u^*$ and $v^*$ respectively, i.e. $u^* := \lim_{k \to \infty} u^k, v^* := \lim_{k \to \infty} v^k$. Finally, for the ease of presentation, let us denote some quantites which will be frequently used in our proofs: $\Delta^k := \max\{\mnorm{u^k-u^*},\mnorm{v^k-v^*}\}$, $\Rrsot := \max\{\mnorm{\log(\aA)},\mnorm{\log(\bB)}\} + \max\left\{\log(n),\dfrac{1}{\eta}\mnorm{C}-\log(n)\right\}$, $\alpha:=\onorm{\aA}$, $\beta:=\onorm{\bB}$ and $\rho_i=\onorm{\pP_i}$ for all $i\in [m]$. 
\begin{table}[ht]
    \centering
    \resizebox{\linewidth}{!}{%
    \begin{tabular}[t]{| c | c | c | c | c | c|}
        \hline
        General & Robust Semi-OT & Unbalanced OT & Robust OT & Non-normalized RSBP & RSBP \\ 
        \hline \hline
        $f: f(X) := \langle C, X \rangle + \tau \times \text{regularization} $ & $\frsot$ & \multicolumn{2}{|c|}{$\frot$} & \multicolumn{2}{|c|}{$\frsbp$}\\ \hline
        $g: g(X) := f(X) - \eta H(X)$ & $\grsot$ & \multicolumn{2}{|c|}{$\grot$} & \multicolumn{2}{|c|}{$\grsbp$} \\ \hline
        $\widehat{X} := \argmin f(X)$ & $\Xrsoth$ & $\Xuoth$ & $\Xroth$ &  & $\Xbh$ \\ \hline
        $X^* := \argmin g(X)$ & $\Xrsots$ & $\Xuots$ & $\Xrots$ & $\Xbps$ & $\Xbs$ \\ \hline
        $x^* := \onorm{X^*}$ & 1 & $\xuots$ & 1 & $\xbps$ & 1 \\ \hline
        $u^k, v^k$ ($k$-th Sinkhorn/IBP update) & $\ursotk , \vrsotk$ & \multicolumn{2}{|c|}{$\uuotk, \vuotk$} & \multicolumn{2}{|c|}{$\uU^k = (u_1^k, \dots, u_m^k), \vV^k = (v_1^k, \dots, v_m^k) $} \\ \hline
        $(u^*, v^*) := \lim_{k \to \infty} (u^k, v^k)$ & $\ursots, \vrsots$ & \multicolumn{2}{|c|}{$\uuots, \vuots$} & \multicolumn{2}{|c|}{$\uU^* = (u_1^*, \dots, u_m^*), \vV^* = (v_1^*, \dots, v_m^*) $} \\ \hline        
        $\Delta^k := \max\{\mnorm{u^k-u^*},\mnorm{v^k-v^*}\}$ & $\Delta^k_{\text{rsot}}$ & \multicolumn{2}{|c|}{$\Delta^k_{\text{uot}}$} & \multicolumn{2}{|c|}{$\boldsymbol{\Delta}^k = (\Delta^k_1,\ldots,\Delta^k_m) $} \\ \hline
        $X^k := B(u^k, v^k)$ & $\Xrsotk$ & $\Xuotk$ & $\Xrotk$ & $\Xbpk$ & $\Xbk$ \\ \hline
        $x^k := \onorm{X^k}$ & 1 (if $k$ is even) & $\xuotk$ & 1 & $\xbpk$ & 1\\ \hline
        \bottomrule
    \end{tabular}}
    \caption{Key notations for technical results and proofs in the supplementary material. When a term has a constant value (e.g. $1$), we provide that value instead of the corresponding notation.}
\end{table}%

\section{Robust Semi-Constrained Optimal Transport: Omitted Proofs}
\label{appendix:rsot}
This appendix is devoted to provide the lemmas and proofs for the computational complexity of robust semi-constrained optimal transport. 
\subsection{Useful Lemmas}
We first start with the following useful lemmas for the proof of Theorem~\ref{theorem:rsot}.
\begin{lemma}
\label{lemma:6_UOT}
The following inequalities are true for all positive $x_i,y_i,x,y$.
\begin{itemize}
    \item[(a)] $\displaystyle \min_{1\le i\le n}\dfrac{x_i}{y_i}\le\dfrac{\sum_{i=1}^nx_i}{\sum_{i=1}^ny_i}\le\max_{1\le i\le n}\dfrac{x_i}{y_i}$,
    \item[(b)] If $\max\Big\{\dfrac{x}{y},\dfrac{y}{x}\Big\}\le 1+\delta$, then $|x-y|\le\delta\min\{x,y\}$,
    \item[(c)] $\left(1+\dfrac{1}{x}\right)^{x+1}\ge e$.
\end{itemize}
\end{lemma}
\begin{proof}[Proof of Lemma \ref{lemma:6_UOT}]
\textbf{}\\
\textbf{(a)} It follows from the assumption $x_i$ and $y_i$ are positive that
\begin{equation*}
    y_j \min_{1\le i\le n} \Big( \dfrac{x_i}{y_i} \Big) \le x_j\le y_j \max_{1\le i\le n} \Big( \dfrac{x_i}{y_i} \Big).
\end{equation*}
Taking the sum over $j$,
\begin{equation*}
    \sum_{j=1}^n y_j \min_{1\le i\le n}\Big(\dfrac{x_i}{y_i} \Big) \le \sum_{j=1}^nx_j\le \sum_{j=1}^n y_j \max_{1\le i\le n} \Big(\dfrac{x_i}{y_i} \Big).
\end{equation*}
This directly leads to the conclusion.\\
\textbf{(b)} WLOG assume that $x>y$, then
\begin{equation*}
    \dfrac{x}{y}\le 1+\delta\Rightarrow x\le y+y\delta\Rightarrow |x-y|\le y\delta.
\end{equation*}
\textbf{(c)} For the fourth inequality, taking the log of both sides, it is equivalent to
\begin{align*}
    (x+1)\left[\log(x+1)-\log(x)\right]\ge 1.
\end{align*}
By the mean value theorem, there exists a number $y$ between $x$ and $x+1$ such that $\log(x+1)-\log(x)=1/y$, then $(x+1)/y\ge 1$.
\end{proof}
\begin{lemma}
\label{lemma:2_UOT}
Let $\arsots=\Xrsots\one_n, \arsotk=\Xrsotk\one_n$ and $\brsots=(\Xrsots)^{\top}\one_n,\brsotk=(\Xrsotk)^{\top}\one_n$. Then,
\begin{itemize}
    \item[(i)] $\displaystyle \left|\log\Big(\dfrac{\arsotsi}{\arsotki}\Big)-\dfrac{\ursotsi-\ursotki}{\eta}\right|\le\max_{1\le j\le n}\dfrac{\vrsotsj-\vrsotkj}{\eta}$,
    \item[(ii)] $\displaystyle \left|\log\Big(\dfrac{\brsotsj}{\brsotkj}\Big)-\dfrac{\vrsotsj-\vrsotkj}{\eta}\right|\le\max_{1\le i\le n}\dfrac{\ursotsi-\ursotki}{\eta}$.
\end{itemize}
\end{lemma}
\begin{proof}[Proof of Lemma \ref{lemma:2_UOT}]
\textbf{}\\
\textbf{(i)} From the definitions of $\arsotki$ and $\arsotsi$, we have
\begin{equation*}
    \log\left(\dfrac{\arsotsi}{\arsotki}\right)=\left(\dfrac{\ursotsi-\ursotki}{\eta}\right)+\log\left(\dfrac{\sum_{j=1}^n\exp\Big(\frac{\vrsotsj-C_{ij}}{\eta}\Big)}{\sum_{j=1}^n\exp\Big(\frac{\vrsotkj-C_{ij}}{\eta}\Big)}\right).
\end{equation*}
The desired inequalities are equivalent to upper and lower bounds for the second term of the RHS. Applying part (a) of Lemma \ref{lemma:6_UOT}, we obtain
\begin{equation*}
    \min_{1\le j\le n}\dfrac{\vrsotsj-\vrsotkj}{\eta}\le\log\left(\dfrac{\arsotsi}{\arsotki}\right)-\dfrac{\ursotsi-\ursotki}{\eta}\le\max_{1\le j\le n}\dfrac{\vrsotsj-\vrsotkj}{\eta}.
\end{equation*}
\textbf{(ii)} Part (ii) are done similarly.
\end{proof}
\begin{lemma}
\label{lemma:rsot:uv_dual:bound_norm} 
We have following upper bounds for the optimal solutions of RSOT's dual form, which is useful for the derivation of the convergence rate:
\begin{align*}
    \max \{\|\ursots\|_{\infty},\|\vrsots\|_{\infty}\}\le(2\tau+\eta)\Rrsot.
\end{align*}
\end{lemma}

\begin{proof}[Proof of Lemma \ref{lemma:rsot:uv_dual:bound_norm}]
First, we will show that
\begin{equation}
    \|\ursots\|_{\infty}\Big(\dfrac{1}{\tau}+\dfrac{1}{\eta}\Big)\le \dfrac{\|\vrsots\|_{\infty}}{\eta} + \Rrsot.
\end{equation}
Since $\ursots$ is a fixed point of the update in Algorithm \ref{algorithm:rsot}, we get
\begin{align}\label{ustar}
    \dfrac{\ursots}{\tau}=\log(\aA)-\log(\arsots).
\end{align}
Then, 
\begin{equation*}
    \dfrac{\ursotsi}{\tau}=\log(\aA_i)-\log\left[\sum_{j=1}^n \exp\Big(\dfrac{\ursotsi+\vrsotsj-C_{ij}}{\eta}\Big)\right],
\end{equation*}
which is equivalent to
\begin{equation*}
    \ursotsi\Big(\dfrac{1}{\tau}+\dfrac{1}{\eta}\Big)=\log(\aA_i)-\log\Big[\sum_{j=1}^n \exp\Big(\dfrac{\vrsotsj-C_{ij}}{\eta}\Big)\Big].
\end{equation*}
The second term can be bounded as follows
\begin{align*}
    \log\Big[\sum_{j=1}^n\exp{\Big(\dfrac{\vrsotsj-C_{ij}}{\eta}\Big)}\Big]\ge \log(n)+\min_{1\le j\le n}\Big\{\dfrac{\vrsotsj-C_{ij}}{\eta}\Big\}\ge \log(n) - \dfrac{\|\vrsots\|_{\infty}}{\eta}-\dfrac{\|C\|_{\infty}}{\eta},
\end{align*}
and
\begin{align*}
    \log\Big[\sum_{j=1}^n\exp{\Big(\dfrac{\vrsotsj-C_{ij}}{\eta}\Big)}\Big]\le \log(n)+\max_{1\le j\le n}\Big\{\dfrac{\vrsotsj-C_{ij}}{\eta}\Big\}\le \log(n) + \dfrac{\|\vrsots\|_{\infty}}{\eta},
\end{align*}
thus leading to
\begin{equation}\label{eq:rsot:log}
    \Big|\log\Big[\sum_{j=1}^n\exp{\Big(\dfrac{\vrsotsj-C_{ij}}{\eta}\Big)}\Big]\Big|\le\dfrac{\|\vrsots\|_{\infty}}{\eta}+\max\Big\{\log(n),\dfrac{\|C\|_{\infty}}{\eta}-\log(n)\Big\}.
\end{equation}
Hence,
\begin{align*}
    |\ursotsi|\Big(\dfrac{1}{\eta}+\dfrac{1}{\tau}\Big)\le |\log(\aA_i)|+\dfrac{\|\vrsots\|_{\infty}}{\eta} +\max\Big\{\log(n),\dfrac{\|C\|_{\infty}}{\eta}-\log(n)\Big\}.
\end{align*}
Choosing $i$ such that $|\ursotsi|=\|\ursots\|_{\infty}$, combining with the fact that
\begin{align*}
    |\log(\aA_i)|\le\max\{\|\log(\aA)\|_{\infty},\|\log(\bB)\|_{\infty}\},
\end{align*}
we have
\begin{equation}\label{rsot1}
    \|\ursots\|_{\infty}\Big(\dfrac{1}{\tau}+\dfrac{1}{\eta}\Big)\le \dfrac{\|\vrsots\|_{\infty}}{\eta} + \Rrsot.
\end{equation}
Next, we will prove that $$\|\vrsots\|_{\infty}\le \|\ursots\|_{\infty}+  \eta \Rrsot.$$
Notice that $\vrsots$ is a fixed point of the update in Algorithm \ref{algorithm:rsot}, we get
$\vrsots = \eta\Big[\dfrac{\vrsots}{\eta}+\log(\bB)-\log(\brsots)\Big]$, which implies that $\log(\brsots)=\log(\bB)$. Therefore,
\begin{align*}
    \log(\bB_j)&=\log\Big[\sum_{i=1}^n\exp{\Big(\dfrac{\ursotsi+\vrsotsj-C_{ij}}{\eta}\Big)}\Big]\\
    &= \dfrac{\vrsotsj}{\eta}+\log\Big[\sum_{i=1}^n\exp{\Big(\dfrac{\ursotsi-C_{ij}}{\eta}\Big)}\Big],
\end{align*}
or equivalently,
$$\dfrac{\vrsotsj}{\eta}=\log(\bB_j)-\log\left[\sum_{i=1}^n\exp{\Big(\dfrac{\ursotsi-C_{ij}}{\eta}\Big)}\right].$$
Using the same arguments as for deriving equation \eqref{eq:rsot:log}, we obtain
\begin{align*}
    \Big|\log\Big[\sum_{i=1}^n\exp{\Big(\dfrac{\ursotsi-C_{ij}}{\eta}\Big)}\Big]\Big|\le\dfrac{\|\ursots\|_{\infty}}{\eta}+\max\Big\{\log(n),\dfrac{\|C\|_{\infty}}{\eta}-\log(n)\Big\}.
\end{align*}
It follows that 
\begin{align*}
    \dfrac{1}{\eta}|\vrsotsj|\le |\log(\bB_j)|+\dfrac{\|\ursots\|_{\infty}}{\eta}+\max\Big\{\log(n),\dfrac{\|C\|_{\infty}}{\eta}-\log(n)\Big\}.
\end{align*}
Choosing $j$ such that $|\vrsotsj|=\|\vrsots\|_{\infty}$, and making use of the fact that
\begin{align*}
|\log(\bB_j)|\le\max\{\|\log(\aA)\|_{\infty},\|\log(\bB)\|_{\infty}\},
\end{align*}
we have
\begin{equation}\label{rsot2}
    \|\vrsots\|_{\infty}\le\|\ursots\|_{\infty}+\eta \Rrsot.
\end{equation}
From equations \eqref{rsot1} and \eqref{rsot2}, we get
\begin{align*}
    \|\ursots\|_{\infty}\Big(\dfrac{1}{\tau}+\dfrac{1}{\eta}\Big)\le \dfrac{\|\vrsots\|_{\infty}}{\eta} + \Rrsot\le \dfrac{\|\ursots\|_{\infty}}{\eta} + 2\Rrsot,
\end{align*}
which implies that
\begin{equation}\label{rsot3}
    \|\ursots\|_{\infty}\le 2\tau \Rrsot\le (2\tau +\eta)\Rrsot.
\end{equation}
Therefore,
\begin{equation}\label{rsot4}
    \|\vrsots\|_{\infty}\le\|\ursots\|_{\infty}+\eta R \le (2\tau + \eta)\Rrsot.
\end{equation}
Combining equation \eqref{rsot3} with equation \eqref{rsot4}, the proof is completed.
\end{proof}

\begin{lemma}
\label{lemma:rsot:uv_dual:convergence_rate}
For any $k\ge 0$, the update $(\ursotko,\vrsotko)$ from Algorithm \ref{algorithm:rsot} satisfies the following bound
\begin{equation}
    \max\Big\{\|\ursotko-\ursots\|_{\infty}, \|\vrsotko - \vrsots\|_{\infty} \Big\}\le \Big(\dfrac{\tau}{\tau+\eta}\Big)^{k/2}\times (2\tau+\eta)\Rrsot.
\end{equation}
This establishes a geometric convergence rate for the dual variables in Algorithm \ref{algorithm:rsot}.
\end{lemma}

\begin{proof}[Proof of Lemma \ref{lemma:rsot:uv_dual:convergence_rate}]
We first consider the case when $k$ is even. From the update of $\ursotko$ in Algorithm \ref{algorithm:rsot}, we have
\begin{align*}
    \ursotkoi &= \dfrac{\eta\tau}{\tau+\eta}\Big[\dfrac{\ursotki}{\eta}+\log(\aA_i)-\log(\arsotki)\Big] \\
    &=\dfrac{\eta\tau}{\tau+\eta}\left\{\dfrac{\ursotki}{\eta}+\big[\log(\aA_i)-\log(\arsotsi)\big]+\big[\log(\arsotsi)-\log(\arsotki)\big]\right\}.
\end{align*}
Using equation \eqref{ustar}, the above equality is equivalent to
\begin{equation*}
    \ursotkoi-\ursotsi=\Big[\eta\log\Big(\dfrac{\arsotsi}{\arsotki}\Big)-(\ursotsi-\ursotki)\Big]\dfrac{\tau}{\tau+\eta}.
\end{equation*}
Applying Lemma \ref{lemma:2_UOT}, we get
\begin{equation*}
    |\ursotkoi-\ursotsi|\le\max_{1\le j\le n}|\vrsotkj-\vrsotsj|\dfrac{\tau}{\tau+\eta},
\end{equation*}
which implies that
\begin{equation}\label{rsot5}
    \|\ursotko-\ursots\|_{\infty}\le\dfrac{\tau}{\tau+\eta}\|\vrsotk-\vrsots\|_{\infty}.
\end{equation}
From the update of $\vrsotk$ in Algorithm \ref{algorithm:rsot}, we have
\begin{align*}
    \vrsotkj = (v^{k-1}_{\text{rsot}})_j + \eta \log\Big(\dfrac{\bB_j}{(b^{k-1}_{\text{rsot}})_j}\Big)
    = (v^{k-1}_{\text{rsot}})_j + \eta \log\Big(\dfrac{\brsotsj}{(b^{k-1}_{\text{rsot}})_j}\Big).
\end{align*}
Subtracting $\vrsotsj$ from both sides and applying Lemma \ref{lemma:2_UOT}, one gets
\begin{align*}
    |\vrsotkj - \vrsotsj| &= \eta \Big|\log\Big(\dfrac{\brsotsj}{(b^{k-1}_{\text{rsot}})_j}\Big)-\dfrac{\vrsotsj-(v^{k-1}_{\text{rsot}})_j}{\eta}\Big|\le \|u^{k-1}_{\text{rsot}}-\ursots\|_{\infty}.
\end{align*}
This leads to
\begin{equation}\label{rsot6}
    \|\vrsotk-\vrsots\|_{\infty}\le\|u^{k-1}_{\text{rsot}}-\ursots\|_{\infty}.
\end{equation}
Combining the two inequalities \eqref{rsot5} and \eqref{rsot6} yields
\begin{equation*}
    \|\ursotko-\ursots\|_{\infty}\le\dfrac{\tau}{\tau+\eta}\|u^{k-1}_{\text{rsot}}-\ursots\|_{\infty}.
\end{equation*}
Repeating all the above arguments alternatively, we have
\begin{align*}
    \|\ursotko-\ursots\|_{\infty}\le\Big(\dfrac{\tau}{\tau+\eta}\Big)^{k/2}\|\ursot^1-\ursots\|_{\infty}&\le\Big(\dfrac{\tau}{\tau+\eta}\Big)^{k/2+1}\|\vrsot^0-\vrsots\|_{\infty} \\
    &=\Big(\dfrac{\tau}{\tau+\eta}\Big)^{k/2+1}\|\vrsots\|_{\infty}.
\end{align*}
Note that $\vrsotko=\vrsotk$ for $k$ even. Therefore, it is clear from \eqref{rsot6} that
\begin{align*}
    \mnorm{\vrsotko-\vrsots}\le\mnorm{u^{k-1}_{\text{rsot}}-\ursots}\le\Big(\dfrac{\tau}{\tau+\eta}\Big)^{k/2}\max\{\|\ursots\|_{\infty},\|\vrsots\|_{\infty}\}.
\end{align*}
Thus,
\begin{equation*}
    \max\Big\{\|\ursotko-\ursots\|_{\infty},\|\vrsotko - \vrsots\|_{\infty} \Big\}\le\Big(\dfrac{\tau}{\tau+\eta}\Big)^{k/2}\max\{\|\ursots\|_{\infty},\|\vrsots\|_{\infty}\}.
\end{equation*}
Similarly, the above result also holds for $k$ odd. Finally, applying Lemma \ref{lemma:rsot:uv_dual:bound_norm}, we obtain the conclusion.
\end{proof}
\subsection{Detailed Proof of Theorem~\ref{theorem:rsot}}
Denoting
\begin{align*}
    k_1 := \log\left(\frac{8\Rrsot(2\tau+\eta)}{3\eta}\right)\Big/\log\left(\frac{\tau+\eta}{\tau}\right), \qquad k_2 := \left(1+\dfrac{\tau}{\eta}\right)\log\left(\dfrac{3\tau\Rrsot[2(\eta+\tau)+3\Rrsot(2\tau+\eta)]}{\eta^2\log(n)}\right),
\end{align*}
we will show that for all $k\ge 1+2\max\{k_1,k_2\}$ and $\eta = \varepsilon / \Ursot$, $\Xrsotk$ is an $\varepsilon$-approximation of the optimal solution $\Xrsoth$, that is
\begin{align*}
    \frsot(\Xrsotk) - \frsot(\Xrsoth) \le \varepsilon = \eta \Ursot.
\end{align*}
First, we can bound the above difference in the following way
\begin{align*}
    \underbrace{\frsot(\Xrsotk)}_{\grsot(\Xrsotk) + \eta H(\Xrsotk)} - \underbrace{\frsot(\Xrsoth)}_{\grsot(\Xrsoth) + \eta H(\Xrsoth)} \le \left[\grsot(\Xrsotk)-\grsot(\Xrsots)\right] + \eta \left[H(\Xrsotk)-H(\Xrsoth)\right],
\end{align*}
where the inequality comes from $\grsot(\Xrsoth) \ge \grsot(\Xrsots)$ which is the optimal value of the entropic ROT. Subsequently, the two terms in the right-hand side can be bounded separately as follows.\\

\paragraph{Upper bound of $H(\Xrsotk)-H(\Xrsoth)$.}
The upper bound is obtained from the following inequalities for the entropy under the constraint  $X \in \RR_+^{n \times n}$ satisfying $\|X\|_1 = 1$,
\begin{equation}\label{eq:entropy}
    1 \le H(X) \le 2\log(n) + 1.
\end{equation}
Since $\Xrsoth$ is the optimal solution for RSOT, $\onorm{\Xrsoth} = 1$. To derive the needed upper bound, we will show that $\onorm{\Xrsotk} = 1$ for even $k$. Notice that when $k$ is even, at step $k-1$ of Algorithm \ref{algorithm:rsot} we update $v$, thus  
\begin{align*}
    \vrsotk = \argmin_{v} \hrsot(\ursot^{k-1}, v) = \argmin_{v} \hrsot(\ursotk, v), \qquad (\text{because } \ursotk = \ursot^{k-1})
\end{align*}
indicating that
\begin{align*}
    \Xrsotk = \argmin_{\substack{X \in \Br_{+}^{n \times n}, X^{\top}\one_{n} = \bB}} \grsot^k(X):=\langle C,X\rangle - \eta H(X) + \tau\KL(X\one_n||\aA^k), 
\end{align*}
where $\aA^k := \exp\left(\dfrac{\ursotk}{\tau}\right)\odot (X\one_n)$ with $\odot$ denoting the element-wise multiplication. As a result, we have $\|\Xrsotk\|_1=1$ which leads to the following inequality
\begin{equation}
    \label{proof:rsot:bound_h}
    H(\Xrsotk)-H(\Xrsoth) \le 2\log(n).
\end{equation}
\paragraph{Upper bound of $\grsot(\Xrsotk) - \grsot(\Xrsots)$.}
The main idea for deriving this bound comes from the geometric convergence rate (i.e. Lemma \ref{lemma:rsot:uv_dual:convergence_rate}). First, we represent the above difference by other quantities that are straightforward to bound. Reusing the definition of $\grsot^k$ above, we utilize the following result regarding the optimal value of entropic RSOT
\begin{align}
    \grsot(\Xrsots) &= -\eta-\tau(1-\alpha) + \langle \vrsots, \brsots \rangle, \label{eq:rsot_v_b_s} \\
    \grsot^k(\Xrsotk) &= -\eta-\tau(1-\alpha^k) + \langle \vrsotk, \brsotk \rangle,\label{eq:rsot_v_b_k}
\end{align}
where $\alpha^k := \onorm{\aA^k}$. We can see that these two equations have a similar form, and we can prove the first one by simple algebraic derivations as follows
\begin{align*}
    \eta H(\Xrsots) 
    &= - \eta \Big[ \sum_{i,j=1}^n \Xrsotsij\log \Xrsotsij + 1 \Big]\\
    &= - \eta \sum_{i,j=1}^n \Xrsotsij \dfrac{\ursotsi+\vrsotsj-C_{ij}}{\eta} + \eta \qquad  \\
    &= - \langle \arsots, \ursots \rangle - \langle \brsots, \vrsots \rangle + \left\langle C, \Xrsots\right\rangle + \eta.
\end{align*}
The second equation comes from the fact that  $\Xrsotsij = \exp\Big\{\dfrac{\ursotsi+\vrsotsj-C_{ij}}{\eta}\Big\}$. Then, we have
\begin{align*}
    \left\langle C, \Xrsots\right\rangle - \eta H(\Xrsots) &= -\eta + \langle \arsots, \ursots \rangle + \langle \brsots, \vrsots \rangle. \\
    \empty & \empty\\
    \tau\KL (\underbrace{\Xrsots\one_n}_{\arsots}||\aA)
    &= -\tau + \tau\alpha -\tau \Big\langle \arsots, \log \Big( \dfrac{\arsots}{\aA} \Big) \Big\rangle \\
    &= -\tau(1-\alpha) - \langle \arsots, \ursots \rangle, 
\end{align*}
because  $\ursots$ satisfies the fixed-point equation: $\dfrac{\ursots}{\tau} = \log \Big( \dfrac{\arsots}{\aA} \Big)$.
The equation for $\grsot(\Xrsots)$ comes straight from adding the above two equations. Then the difference of interest can be written as
\begin{align}
    g(\Xrsotk)-g(\Xrsots) &= \left[g(\Xrsotk)-g^k(\Xrsotk)\right]+\left[g^k(\Xrsotk)-g(\Xrsots)\right] \nonumber \\
    &= \tau \Big\langle \arsotk, \log \Big( \frac{\aA^k}{\aA} \Big) \Big\rangle + \big( \langle \vrsotk, \brsotk \rangle - \langle \vrsots, \brsots \rangle \big). \label{eq:key_bound_rsot}
\end{align}
Both terms above can be bounded with regards to $\drsotk := \max\left\{\mnorm{\ursotk-\ursots},\mnorm{\vrsotk-\vrsots}\right\}$.
\textit{On the first term in equation~\eqref{eq:key_bound_rsot}.} From the fixed-point result for $u$-updates, we have
\begin{align}
    \Big\| \log\left(\frac{\aA^k}{\aA}\right) \Big\|_\infty 
    &= \Big\| \frac{\ursotk - \ursots}{\tau}-\log\left(\frac{\arsots}{\arsotk}\right) \Big\|_\infty  \nonumber \\ 
    &\le \frac{1}{\tau} \mnorm{\ursotk-\ursots} + \Big\| \log\left(\frac{\arsots}{\arsotk}\right) \Big\|_\infty \nonumber \\
    &\le \frac{1}{\tau} \mnorm{\ursotk-\ursots} + \frac{1}{\eta} (\mnorm{\ursotk-\ursots} + \mnorm{\vrsotk-\vrsots}) \nonumber \\
    &\le \Big( \frac{1}{\tau} + \frac{2}{\eta} \Big) \drsotk, \nonumber \\
    \tau \Big\langle \arsotk, \log \Big( \frac{\aA^k}{\aA} \Big) \Big\rangle
    &\le \tau \underbrace{\onorm{\arsotk}}_{=1} \Big\| \log\left(\frac{\aA^k}{\aA}\right) \Big\|_\infty \le \frac{2\tau + \eta}{\eta} \drsotk. \label{proof:rsot:bound_g:first_term}
\end{align}
\textit{On the second term in equation~\eqref{eq:key_bound_rsot}.} We find that
\begin{align*}
    \langle \vrsotk, \brsotk \rangle - \langle \vrsots, \brsots \rangle 
    &= \langle \vrsotk - \vrsots, \brsotk \rangle - \langle \vrsots, \brsots - \brsotk \rangle \\
    &\le \underbrace{\onorm{\brsotk}}_{=1} \underbrace{\mnorm{\vrsotk - \vrsots}}_{\le \drsotk} + \underbrace{\mnorm{\vrsots}}_{\le (2\tau + \eta) \Rrsot} \onorm{\brsots - \brsotk}.
\end{align*}
Thus, we need an upper bound for $\onorm{\brsots - \brsotk}$, i.e., $\ell_1$-norm of the difference between $\brsots$ and $\brsotk$. Note that we have the following bound on their ratio 
(which is a direct result of Lemma \ref{lemma:2_UOT})
\begin{align*}
    \max_j \Big\{\frac{\brsotsj}{\brsotkj}, \frac{\brsotkj}{\brsotsj}\Big\}
    &\le \exp \left( \frac{\mnorm{\ursotk - \ursots} + \mnorm{\vrsotk - \vrsots}}{\eta} \right) \le \exp \left(\frac{2 \drsotk}{\eta} \right).
\end{align*}

Applying part (b) of Lemma \ref{lemma:6_UOT}, we obtain
\begin{align*}
    \big|\brsotsj - \brsotkj\big|
    &\le \left[ \exp \left(\frac{2 \drsotk}{\eta} \right) - 1 \right] \min_j \left\{ \brsotkj, \brsotsj \right\}. \\
    \sum_{j=1}^n |\brsotkj - \brsotsj| &\le \left[ \exp \left(\frac{2 \drsotk}{\eta} \right) - 1 \right] \underbrace{\sum_{j=1}^n\min \left\{ \brsotkj, \brsotsj \right\}}_{\le \onorm{\brsots} = 1} \le \exp \left(\frac{2 \drsotk}{\eta} \right) - 1.
    \end{align*}
Hence, 
    \begin{align*}
    \onorm{\brsots - \brsotk} \le \sum_{j=1}^n |\brsotkj - \brsotsj| \le \exp \left(\frac{2 \drsotk}{\eta} \right) - 1. 
\end{align*}
To remove the exponential operator, noting that for $k\ge 1+2k_1$, we have $\frac{\drsotk}{\eta} \le \frac{3}{8}$. Thus, $\exp \left(\frac{2 \drsotk}{\eta} \right) - 1 \le  \frac{3\drsotk}{\eta}$, and consequently $\onorm{\brsots - \brsotk} \le  \frac{3\drsotk}{\eta}$. Having this bound on $\onorm{\brsots - \brsotk}$, we can completely bound the second term of interest as follows
\begin{align}
    \langle \vrsotk, \brsotk \rangle - \langle \vrsots, \brsots \rangle &\le \left[1+\dfrac{3}{\eta}(2\tau+\eta)\Rrsot\right] \drsotk.
    \label{proof:rsot:bound_g:second_term}
\end{align}
Plugging the bounds~\eqref{proof:rsot:bound_g:first_term} and~\eqref{proof:rsot:bound_g:second_term} to equation~\eqref{eq:key_bound_rsot}, we obtain
\begin{align*}
    \grsot(\Xrsotk)-\grsot(\Xrsots)&\le \left[1+\frac{2\tau+\eta}{\eta}+\dfrac{3}{\eta}(2\tau+\eta)\Rrsot\right] \drsotk.
\end{align*}
From this bound, we will show that
\begin{align}
    \label{proof:rsot:bound_g}
    \grsot(\Xrsotk) - \grsot(\Xrsots)\le \eta\log(n).
\end{align}
From Lemma \ref{lemma:rsot:uv_dual:convergence_rate} we have $\drsotk \le 3\tau\left(\frac{\tau}{\tau+\eta}\right)^{(k-1)/2} \Rrsot$. Thus, we only need to prove that for $k \ge 2 k_2 + 1$,
\begin{equation*}
    3\tau \left(\dfrac{\tau}{\tau+\eta}\right)^{(k-1)/2} \cdot \Rrsot \cdot \left[1+\frac{2\tau+\eta}{\eta}+\dfrac{3}{\eta}(2\tau+\eta)\Rrsot\right]\le \eta \log(n).
\end{equation*}
This form of inequality can be represented through the following lemma.
\begin{lemma}
\label{proof:bound_g:bound_k}
For $0 < s < 1$, if $D \ge s^2$ and $\kappa \ge (1 + \frac{1}{s}) \log \Big( \frac{D}{s^2} \Big)$, then $D \le s^2 (1 + s)^\kappa$.
\end{lemma}
\begin{proof}[Proof of Lemma \ref{proof:bound_g:bound_k}]
The statement comes directly  from a chain of inequalities using Lemma \ref{lemma:6_UOT}c for $x = \frac{1}{s}$:
\begin{align*}
    s^2(1+s)^\kappa 
    &\ge s^2(1+s)^{(1 + \frac{1}{s}) \log \big( \frac{D}{s^2} \big)}\\
    &\ge s^2\exp \Big\{\log \Big( \frac{D}{s^2} \Big) \Big\}= D.
\end{align*}
\end{proof}
\noindent Applying Lemma \ref{proof:bound_g:bound_k} for $s = \frac{\eta}{\tau} \in (0, 1), D = \frac{3\Rrsot}{\tau \log(n)}\left[2(\tau+\eta)+3\Rrsot(2\tau+\eta)\right]$ and $\kappa = \frac{k - 1}{2}$, we get the inequality \eqref{proof:rsot:bound_g}.
Combining the bounds \eqref{proof:rsot:bound_h} and \eqref{proof:rsot:bound_g}, we obtain
\begin{align*}
    \frsot(\Xrsotk) - \frsot(\Xrsoth) \le \eta \log(n) + 2\eta\log(n) = 3\eta\log(n) \le \eta \Ursot = \varepsilon.
\end{align*}
\paragraph{The complexity of Algorithm \ref{algorithm:rsot}.} By definition, $\Ursot = \bigO(\log(n))$. Applying part (c) of Lemma \ref{lemma:6_UOT} with $x = \frac{\tau}{\eta}$, we have
\begin{equation*}
    \log\left(\frac{\tau+\eta}{\tau}\right)\ge  \dfrac{1}{1+\frac{\tau}{\eta}}.
\end{equation*}
Then, $k_1$ can be bounded as follows
\begin{align*}
    k_1 = \dfrac{\log\left(\frac{8\Rrsot(2\tau+\eta)}{3\eta}\right)}{\log\left(\frac{\tau+\eta}{\tau}\right)}&\le \log\left(\frac{8\Rrsot(2\tau+\eta)}{3\eta}\right)\left(1+\frac{\tau}{\eta}\right)\\
    &= \left(1+\frac{\tau \Ursot}{\varepsilon}\right)\left[\log\left(\frac{\Ursot}{\varepsilon}\right)+\log\left(\frac{8}{3}\eta \Rrsot\right)+\log\left(2\frac{\tau \Ursot}{\varepsilon}+1\right)\right].
\end{align*}
Assume that $\Rrsot = \bigO\left(\frac{1}{\eta}\mnorm{C}\right)$, we obtain
\begin{align}
    k_1 &= \bigO\left(\frac{\tau\log(n)}{\varepsilon}\left[\log\left(\frac{\log(n)}{\varepsilon}\right)+\log(\mnorm{C})+\log\left(\frac{\tau\log(n)}{\varepsilon}\right)\right]\right)\nonumber\\
    \label{eq: k1}
    &= \bigO\left(\tau\frac{\log(n)}{\varepsilon}\left[\log(\mnorm{C})+\log(\log(n))+\log(\tau)+\log\left(\frac{1}{\varepsilon}\right)\right]\right).
\end{align}
Next, let us consider
\begin{align*}
    k_2 &= \left(1+\frac{\tau}{\eta}\right)\left[\log(3\tau \Rrsot)+\log(2(\tau+\eta)+3\Rrsot(2\tau+\eta))+2\log\left(\frac{1}{\eta}\right)-\log(\log(n))\right]\\
    &\le \left(1+\frac{\tau}{\eta}\right)\left[\log(3 \Rrsot)+\log(4+9\Rrsot)+2\log(\tau)+2\log\left(\frac{1}{\eta}\right)-\log(\log(n))\right]\\
    &\le \left(1+\frac{\tau \Ursot}{\varepsilon}\right)\left[\log(3\eta \Rrsot)+2\log(9\eta \Rrsot)+2\log(\tau)+5\log\left(\frac{\Ursot}{\varepsilon}\right)-\log(\log(n))\right].
\end{align*}
Thus,
\begin{align}
    k_2 &= \bigO\left(\tau\frac{\log(n)}{\varepsilon}\left[\log(\mnorm{C})+\log(\tau)+5\log\left(\frac{\log(n)}{\varepsilon}\right)-\log(\log(n))\right]\right)\nonumber\\
    \label{eq: k2}
    &= \bigO\left(\tau\frac{\log(n)}{\varepsilon}\left[\log(\mnorm{C})+\log(\tau)+\log(\log(n))+\log\left(\frac{1}{\varepsilon}\right)\right]\right).
\end{align}
Equations \eqref{eq: k1} and \eqref{eq: k2} imply that
\begin{equation*}
    k = \bigO\left(\tau \left[\frac{\log(n)}{\varepsilon} \right] \left[\log(\mnorm{C})+\log(\tau)+\log(\log(n))+\log\left(\frac{1}{\varepsilon}\right)\right]\right).
\end{equation*}
Multiplying the above quantity with $\bigO(n^2)$ arithmetic operations per iteration, we obtain the final complexity. As a consequence, we reach the conclusion of Theorem~\ref{theorem:rsot}.

\section{Robust Semi-Constrained Barycenter: Omitted Proofs}
In this appendix, we provide some useful lemmas and proofs for deriving the computational complexity of the robust semi-constrained barycenter problem.
\label{sec:proofs:robust_barycenter}
\subsection{Useful Lemmas}
\begin{lemma}
\label{lemma:rsbp:entropic:dual}
The dual form of entropic RSBP in \eqref{problem:entropic_RSBP} without constraints $\|X_i\|_1=1$ for all $i\in [m]$ is given by
\begin{align*}
    \min_{\substack{\uU, \vV: \sum_{i = 1}^m \omega_i v_i = \zeros_n}} \bar{h}_{\text{rsbp}}(\uU, \vV) : = \sum_{i = 1}^{m} \omega_{i} \Big( \eta \onorm{B(u_i, v_i; C_i)} + \tau \big \langle e^{-u_{i}/ \tau}, \pP_{i} \big \rangle \Big).
\end{align*}
\end{lemma}

\begin{proof}[Proof of Lemma \ref{lemma:rsbp:entropic:dual}]

First, we rewrite the objective function \eqref{problem:entropic_RSBP} as follows
\begin{align}
    \label{eq:rsbp:primal}
    \min_{\substack{X_i\in \Br_+^{n \times n}, X_i\one_n=\yY_i, \forall i\in [m]; \\X^{\top}_i\one_n = X^{\top}_{i+1}\one_n, \forall i\in [m-1] }} \quad\sum_{i=1}^m\omega_i\left[\left\langle C_i, X_i\right\rangle - \eta H(X_i) + \tau \KL(\yY_i || \pP_i)\right].
\end{align}
The Lagrangian function for the above problem is equal to
\begin{align*}
    & \sum_{i=1}^m\left(\omega_i[\langle C_i,X_i\rangle - \eta H(X_i) + \tau\KL(\yY_i||\pP_i)]-\lambda_i^{\top}(X_i\one_n-\yY_i)-\mu_i^{\top}(X_{i+1}^{\top}\one_n-X_i^{\top}\one_n)\right)\\
    & = \sum_{i=1}^m\left(\omega_i[\langle C_i,X_i\rangle - \eta H(X_i) + \tau\KL(\yY_i||\pP_i)]-\lambda_i^{\top}(X_i\one_n-\yY_i)-(\mu_{i-1}-\mu_i)^{\top}X_i^{\top}\one_n\right),
\end{align*}
where $\lambda_i,\mu_i\in \Br^n$ for all $i\in [m]$ with convention $\mu_0=\mu_m=\zeros_n$. Using the change of variables $u_i=\lambda_i/\omega_i$ and $v_i=(\mu_{i-1}-\mu_i)/\omega_i$, we have $\sum_{i=1}^m\omega_iv_i = \zeros_n$ which allows to uniquely reconstruct $\mu_1,\ldots,\mu_m$. Then, the problem \eqref{eq:rsbp:primal} is equivalent to 
\begin{align}
    \max_{\substack{\uU,\vV \\ \sum_{i=1}^m\omega_iv_i=\zeros_n}} \min_{\substack{X_i\in\Br^{n\times n},\forall i\in [m]\\\yY_i\in\Br^n,\forall i\in [m]}}\sum_{i=1}^m \omega_i[\langle C_i,X_i\rangle &- \eta H(X_i) + \tau \KL(\yY_i||\pP_i) \nonumber \\
    &- u_i^{\top}(X_i\one_n-\yY_i)-v_i^{\top}X_i^{\top}\one_n] \label{eq:rsbp:dual_1}
\end{align}
It can be verified that for all $i\in [m]$,
\begin{align*}
    \min_{\yY_i\in \Br^n}\tau \KL(\yY_i || \pP_i) + u_i^{\top}\yY_i = -\tau\Big\langle e^{-u_i/\tau}, \pP_i\Big\rangle + \pP_i^{\top}\one_n.
\end{align*}
Moreover, the objective function of the optimization problem
\begin{align*}
    \min_{X_i\in \Br^{n\times n}} \langle C_i, X_i\rangle - u_i^{\top}X_i\one_n - v_i^{\top}X_i^{\top}\one_n - \eta H(X_i)
\end{align*}
is strongly convex. Thus, it has an unique optimal solution which could be directly calculated as $\bar{X}_i=B(u_i,v_i;C_i)$. Therefore,
\begin{align*}
    \min_{X_i\in \Br^{n\times n}} \langle C_i, X_i\rangle - u_i^{\top}X_i\one_n - v_i^{\top}X_i^{\top}\one_n - \eta H(X_i) = -\eta\onorm{B(u_i,v_i;C_i)}.
\end{align*}
Collecting all of the above results, the optimization problem \eqref{eq:rsbp:dual_1} turns into
\begin{align*}
    &\max_{\substack{\uU,\vV \\ \sum_{i=1}^m\omega_iv_i=\zeros_n}} \sum_{i=1}^m\omega_i\left(-\eta\onorm{B(u_i,v_i;C_i)}-\tau\Big\langle e^{-u_i/\tau}, \pP_i\Big\rangle + \pP_i^{\top}\one_n\right)\\
    &= \min_{\substack{\uU,\vV \\ \sum_{i=1}^m\omega_iv_i=\zeros_n}}\sum_{i=1}^m\omega_i\left(\eta\onorm{B(u_i,v_i;C_i)}+\tau\Big\langle e^{-u_i/\tau}, \pP_i\Big\rangle\right).
\end{align*}
We have thus proved our claim.
\end{proof}
\noindent Next, we will derive formulas for the updates $(\uU^k,\vV^k)$ of Algorithm \ref{algorithm:barycenter_semiOT} in the following lemma. Assume that at iteration $k$ where $k$ is even, $\uU^{k+1}$ was found by minimizing the function $\bar{h}_\text{rsbp}$ given $\vV^{k}$ and simply keep $\vV^{k+1} = \vV^{k}$ while for odd $k$, we do vice versa. In particular,
\begin{align*}
    \uU^{k + 1}  &= \argmin_{\uU} \bar{h}_{\text{rsbp}}(\uU, \vV^k), \quad \qquad \qquad  \vV^{k+1}=\vV^k  \quad \text{if k is even};\\
    \vV^{k + 1}  &= \argmin_{\vV: \sum_{i=1}^m \omega_i v_i = \zeros_n} \bar{h}_{\text{rsbp}}(\uU^k, \vV), \qquad \uU^{k+1}=\uU^k \quad \text{if k is odd}.
\end{align*}
Let $\Xbpk = (\bar{X}^{k}_1,\ldots,\bar{X}^{k}_m)$ be the non-normalized output at $k$-th iteration of Algorithm \ref{algorithm:barycenter_semiOT}. For the ease of presentation, let us denote $a^k_i = \Xbpki\one_n$ and $b^k_i=(\Xbpki)^{\top}\one_n$ for all $i\in [m]$.
\begin{lemma}
\label{lemma:rsbp:ibp_algorithm}
In Algorithm \ref{algorithm:barycenter_semiOT}, the updates $(\uU^k,\vV^k)$ admit the following form
\begin{align}
    \label{eq:rsbp:update_u}
    u^{k + 1}_i  &= \frac{\eta \tau}{\eta + \tau} \left[\frac{u^k_i}{\eta} + \log(\bold{p}_{i}) - \log(a_{i}^k)\right] \quad \qquad \qquad\text{if k is even};\\
    \label{eq:rsbp:update_v}
    v^{k + 1}_i  &= \eta  \left[ \frac{v^k_i}{\eta} - \log(b_{i}^{k}) - \sum_{t = 1}^{m} \omega_{t} \big(\frac{v^k_t}{\eta} - \log(b_{t}^{k}) \big) \right] \quad \text{if k is odd},
\end{align}
for all $i\in [m]$.
\end{lemma}
\begin{proof}[Proof of Lemma \ref{lemma:rsbp:ibp_algorithm}]
For $k$ even, by setting the gradients of $\bar{h}_\text{rsbp}$ with respect to $u_i$ to 0 given fixed $\vV^k$, the update $u^k_i$ satisfies
\begin{align*}
    \exp\left(\dfrac{(u^{k+1}_i)_j}{\eta}\right)\sum_{l=1}^n\exp\left(\dfrac{\vbkil-(C_i)_{jl}}{\eta}\right)=\exp\left(-\dfrac{(u_i^{k+1})_j}{\tau}\right)\pP_i \quad \text{for all } j\in [n].
\end{align*}
Multiplying both sides by $\exp\left(\frac{\ubkij}{\eta}\right)$, we get
\begin{align*}
    \exp\left(\dfrac{(u^{k+1}_i)_j}{\eta}\right)(a^k_i)_j= \exp\left(\frac{\ubkij}{\eta}\right)\exp\left(-\dfrac{(u_i^{k+1})_j}{\tau}\right)\pP_i\quad \text{for all } j\in [n].
\end{align*}
Taking logarithm of the above equation and simplifying the result lead to the equality \eqref{eq:rsbp:update_u}. \\
For $k$ odd, recall that $\vV^{k + 1}  = \argmin_{\vV: \sum_{i=1}^m \omega_i v_i = \zeros_n} \bar{h}_{\text{rsbp}}(\uU^k, \vV)$, which also means that
\begin{align*}
    \vV^{k + 1}  = \argmin_{\vV}\sum_{i=1}^m\omega_i\left(\eta\onorm{B(u^k_i,v_i;C_i}+\tau\big\langle e^{-u^k_i/\tau},\pP_i\big\rangle\right) +\gamma^{\top} \Big(\sum_{i=1}^m \omega_i v_i\Big),
\end{align*}
where $\gamma\in\Br^n$ is a vector of Lagrange multipliers. Taking the derivatives of the above objective function with respect to $v_i$,
\begin{align}
\exp\Big(\dfrac{v^{k+1}_i}{\eta}\Big)\odot A^k_{i} + \gamma &= \zeros_n\nonumber\\
\label{eq:rsbp:log_gamma}
\Leftrightarrow\frac{v^{k+1}_i}{\eta} + \log(A^k_i) &= \log(-\gamma),
\end{align}
where $A^k_i=\left(\sum_{j=1}^n\exp\Big\{\frac{\ubkij-(C_i)_{jl}}{\eta}\Big\}\right)^n_{l=1}$. Subsequently, taking sum over $i$ and using the fact that $\sum_{i=1}^m\omega_iv^{k+1}_i=0$, we obtain $\log(-\gamma) =\sum_{i=1}^m\omega_i\log(A^k_i)$. Plugging this result in equation \eqref{eq:rsbp:log_gamma}, we obtain
\begin{align*}
    \frac{v^{k+1}_i}{\eta} &= \sum_{t=1}^m\omega_t\log(A^k_t) - \log(A^k_i)\\
    &= \frac{\vbki}{\eta} - \log\left(A^k_i\odot\exp\Big(\frac{\vbki}{\eta}\Big)\right)+\sum_{t=1}^m\omega_t\left[\log\left(A^k_t\odot\exp\Big(\frac{v^k_t}{\eta}\Big)\right)-\frac{v^k_t}{\eta}\right]\\
    &= \frac{\vbki}{\eta}- \log(b^k_i)-\sum_{t=1}^m\omega_t\left(\frac{v^k_t}{\eta}-\log(b^k_t)\right).
\end{align*}
Hence, the proof is completed.
\end{proof}
\begin{lemma}\label{lemma:grsot_tx}
Reusing the definition of the function $\grsot$ in equation \eqref{eq:semi_robust_OT_entropic}, we have the following property which is useful for the proofs of subsequent lemmas
\begin{align*}
    \grsot(tX) = t\grsot(X) + \tau(1-t)\alpha+(\tau+\eta)xt\log(t),
\end{align*}
for any $X\in\Br^{n\times n}_+$ and $t\in\Br^+$ where $x=\onorm{X}$.
\end{lemma}
\begin{proof}[Proof of Lemma \ref{lemma:grsot_tx}]
By the definition of $\grsot$, one has
\begin{align*}
    \grsot(tX) = \langle C,tX\rangle + \tau\KL(tX\one_n||\aA) -\eta H(tX).
\end{align*}
For the KL term of $\grsot(tX)$, by denoting $a:=X\one_n$, we get
\begin{align*}
    \KL(tX\one_n||\aA) &= \sum_{i=1}^nta_i\log\left(\dfrac{a_i}{\aA_i}\right)-\sum_{i=1}^nta_i + \sum_{i=1}^n\aA_i\\
    &= \sum_{i=1}^nta_i\left[\log\left(\dfrac{a_i}{\aA_i}\right)+\log(t)\right]-tx+\alpha\\
    &= t\sum_{i=1}^n\left[a_i\log\left(\dfrac{a_i}{\aA_i}\right)-a_i+\aA_i\right]+(1-t)\alpha+xt\log(t)\\
    &=t\grsot(X) + \tau(1-t)\alpha+(\tau+\eta)xt\log(t).
\end{align*}
For the entropic term, it can be verified that
\begin{align*}
    -H(tX) = \sum_{i,j=1}^ntX_{ij}(\log(tX_{ij})-1)=\sum_{i,j=1}^ntX_{ij}(\log(X_{ij})-1) + xt\log(t)= -tH(X) + xt\log(t).
\end{align*}
Collecting all of the above results, we obtain the conclusion.
\end{proof}

\begin{remark}
\label{remark:rsbp:xk}
Notice that when $k$ is even, at step $(k-1)$-th of Algorithm \ref{algorithm:barycenter_semiOT}, $\{\vbki\}_{i=1}^m$ is found by minimizing the dual function \eqref{problem:dual_entropic_rsbp_no_norm} given $\{\pP_i\}_{i=1}^m$ and fixed $\{u^{k-1}_i\}_{i=1}^m$, and remain $\{\ubki\}_{i=1}^m=\{u^{k-1}_i\}_{i=1}^m$. Thus, $\Xbpk$ is the optimal solution of
\begin{align*}
    & \min_{X_{1}, \ldots, X_{m} \in \mathbb{R}_{+}^{n \times n}} \grsbp^k(X_1,\ldots,X_m):=\sum_{i=1}^m\omega_i\left[\langle C_i,X_i\rangle +\tau\KL(X_i\one_n||\pP^k_i)-\eta H(X_i)\right] \\
    & \text{s.t.} \ X_{i}^{\top} \one_{n} = X_{i + 1}^{\top} \one_{n} \ \text{for all} \ i \in [m - 1], 
\end{align*}
where $\pP^k_i = \exp\left(\frac{\ubki}{\tau}\right)\odot (\Xbpki\one_n)$ with $\odot$ denoting element-wise multiplication. The constraints $X_{i}^{\top} \one_{n} = X_{i + 1}^{\top} \one_{n}$ for all $i\in [m-1]$ imply that $\|\bar{X}^{k}_i\|_1=\|\bar{X}^{k}_{i+1}\|_1$ for any $i\in [m-1]$. Recall that $\Xbps$ is the optimizer of $\grsbp$ with the feasible set $\mathcal{D}(\mathbf{X})$. By using similar arguments, we also have $\|\bar{X}^{*}_i\|_1=\|\bar{X}^{*}_{i+1}\|_1$ for all $i\in [m-1]$. Denote $\xbpk=\|\bar{X}^{k}_1\|_1$ for $k$ even and $\xbps=\onorm{\bar{X}^*_1}$, we will derive the upper bound of these quantities in the following lemma.
\end{remark}
\begin{lemma}
\label{lemma:rsbp:x_bound}
The upper bounds of $\xbpk$ and $\xbps$ are derived as follows
\begin{itemize}
    \item[(i)] $\xbps \le 3+\dfrac{1}{\log(n)}$;
    \item[(ii)] $\xbpk \le \dfrac{3}{2}\left(3+\dfrac{1}{\log(n)}\right)$,\quad for all even $k\ge 2 +2\left(\frac{\tau}{\eta}+1\right)\log\left(\frac{4\Rrsbp\tau^2}{\eta^2}\right)$ .
\end{itemize}
\end{lemma}
\begin{proof}[Proof of Lemma \ref{lemma:rsbp:x_bound}]
\textbf{}\\
\textbf{(i)} Consider the function $\grsbp(t\Xbps)$ where $t\in\mathbb{R}^+$,
\begin{align}
    \grsbp(t\Xbps)&=\sum_{i=1}^m \omega_i\grsot(t\Xbpki;\pP_i,C_i)\nonumber\\
    &=\sum_{i=1}^m\omega_i\left[t\grsot(\Xbpki;\pP_i,C_i)+\tau(1-t)+(\tau+\eta)\xbpk t\log(t)\right]\nonumber\\
    &= t\grsbp(\Xbps) +\tau(1-t) +(\tau+\eta)t\log(t)\xbps. 
    \label{eq:rsbp:g(tX)}
\end{align}
The second equality is due to Lemma \ref{lemma:grsot_tx}. Taking the derivative of $\grsbp(t\Xbps)$ with respect to $t$,
\begin{equation*}
    \partial_t \grsbp(t\Xbpk)=\grsbp(\Xbps) - \tau + (\tau+\eta)(1+\log(t))\xbps.
\end{equation*}
Since $\grsbp(t\Xbps)$ attains its minimum at $t=1$, we obtain
\begin{equation}
    \grsbp(\Xbps) + (\tau+\eta)\xbps = \tau. 
    \label{eq:rsbp:g_equality}
\end{equation}
By using the facts $\grsbp(\Xbps)\ge -\eta\sum_{i=1}^m\omega_iH(\Xbpsi)$ and $H(\Xbpsi)\le 2\xbps\log(n)+\xbps-\xbps\log(\xbps)$, we have 
\begin{align*}
    \tau-(\tau+\eta)\xbps &\ge -\eta\sum_{i=1}^m\omega_iH(\Xbpsi)\\
    &\ge \eta\sum_{i=1}^m\omega_i\left[-2\xbps\log(n)-\xbps+\xbps\log(\xbps)\right]\\
    &= \eta\left[-2\xbps\log(n)-\xbps+\xbps\log(\xbps)\right].
\end{align*}
It follows from the inequalities $z\log(z)\ge z-1$ that
\begin{align*}
    \tau \ge\eta\xbps\log(\xbps)+(\tau-2\eta\log(n))\xbps\ge\eta\xbps-\eta+(\tau-2\eta\log(n))\xbps.
\end{align*}
Then, combining the above result and the inequality $3\eta\log(n)\le\tau$, we get 
\begin{equation*}
    \xbps\le \frac{\tau+\eta}{\eta+\tau-2\eta\log(n)}\le 3+\frac{1}{\log(n)}.
\end{equation*}
\textbf{(ii)} First, let us denote
\begin{align*}
    \Delta^k_i=\max\left\{\mnorm{\ubki-\ubsi},\mnorm{\vbki-\vbsi}\right\}.
\end{align*}
From Lemma \ref{lemma:geometric_converge_IBP}, we have
\begin{align*}
    \Delta^{k+1}_i\le \tau \left(\frac{\tau}{\tau+\eta}\right)^{k/2} \Rrsbp.
\end{align*}
Next, we will prove that $\Delta^{k+1}_i \le \frac{\eta^2}{4\tau}$ for all even $k\ge 2\left(\frac{\tau}{\eta}+1\right)\log\left(\frac{4\Rrsbp\tau^2}{\eta^2}\right)$ $i\in [m]$, which is equivalent to 
\begin{align*}
    \tau \left(\frac{\tau}{\tau+\eta}\right)^{k/2} \Rrsbp&\le\frac{\eta^2}{4\tau}\\
    \Leftrightarrow \left(\frac{\tau+\eta}{\tau}\right)^{k/2}\frac{\eta^2}{\tau^2}&\ge 4\Rrsbp\\
    \Leftrightarrow (1+s)^{k/2}s &\ge 4\Rrsbp,
\end{align*}
where $s=\frac{\eta}{\tau}$. Let $t=1+\frac{\log(4\Rrsbp)}{2\log(\frac{1}{s})}$. Since $4\Rrsbp\ge 8\log(n)\ge \frac{\eta^2}{\tau^2}=s^2$, therefore, $t>1+\frac{2\log(s)}{2\log(\frac{1}{s})}=0$. Due to the fact that $\frac{k}{2}\ge\left(\frac{\tau}{\eta}+1\right)\log\left(\frac{4\Rrsbp\tau^2}{\eta^2}\right)=\left(1+\frac{1}{s}\right) (2t) \log\left(\frac{1}{s}\right)>0$, we obtain
\begin{align*}
    s^2(1+s)^{k/2}&\ge s^2(1+s)^{(\frac{1}{s}+1)2\log(\frac{1}{s})t}\\
    &\ge s^2\exp\left\{2\log(1/s)t\right\}\\
    &=\frac{1}{s^{2t-2}}=\frac{1}{s^{\log(4\Rrsbp)/\log(1/s)}}=\frac{1}{s^{-\log_s(4\Rrsbp)}}=4\Rrsbp.
\end{align*}
Therefore, $\max_{1\le i\le m}\Delta^{k+1}_i\le\frac{\eta^2}{4\tau}\le\frac{1}{8}$. Then, by using the same arguments as part (b) of Lemma 5 in \cite{pham2020unbalanced}, we get
\begin{equation}
    \label{eq:rsbp:x-x}
    |\xbpk-\xbps|\le \frac{3}{\eta}\Delta^k_1\min\left\{\xbpk,\xbps\right\}.
\end{equation}
Note that $u^{k}_1 = u^{k-1}_1$ and $v^{k+1}_1=v^{k}_1$ for even $k$, hence, $\Delta^k_1\le\max\{\Delta^{k-1}_1,\Delta^{k+1}_1\}\le\frac{\eta^2}{4\tau}$. As a result,
\begin{equation*}
    \xbpk\le\left(1+\frac{3}{\eta}\Delta^k_1\right)\xbps\le\frac{3}{2}\xbps\le\frac{3}{2}\left(3+\frac{1}{\log(n)}\right).
\end{equation*}
We have thus proved our claim.
\end{proof}

\subsection{Proof of Lemma~\ref{lemma:barycenter_normalize}}
From the constraints $X_{i}^{\top} \one_{n} = X_{i + 1}^{\top} \one_{n}$ for all $i\in [m-1]$ in $\Dd(\bold{X})$, we have that $\|\bar{X}^{*}_i\|_1$ is equal to each other for all $i \in [m]$ and denote $\bar{x}^* = \|\bar{X}^{*}_1\|_1$. Applying Lemma \ref{lemma:grsot_tx}, we get
\begin{align*}
    \grsbp(\Xbps)
    &= \sum_{i=1}^m \omega_i\grsot\left(\Xbpsi;\pP_i,C_i\right) \\
    &= \sum_{i=1}^m \omega_i\grsot\left(\bar{x}^*\frac{\Xbpsi}{\bar{x}^*};\pP_i,C_i\right)\\
    &= \sum_{i=1}^m \omega_i \left[\bar{x}^*\grsot\left(\frac{\Xbpsi}{\bar{x}^*};\pP_i,C_i\right)+\tau(1-\bar{x}^*) \rho_i+(\tau+\eta)\bar{x}^*\log(\bar{x}^*)\right]\\
    &=\bar{x}^*\grsbp\left(\frac{\Xbps}{\bar{x}^*}\right)+\tau(1-\bar{x}^*)\sum_{i=1}^m\omega_i\rho_i+(\tau+\eta)\bar{x}^*\log(\bar{x}^*).
\end{align*}
Similarly, applying Lemma \ref{lemma:grsot_tx}, we obtain
\begin{align*}
    \grsbp(x^*\Xbs)
    &= \sum_{i=1}^m \omega_i\grsot\left(\bar{x}^*\Xbsi;\pP_i,C_i\right) \\
    &=\sum_{i=1}^m\omega_i\left[\bar{x}^*\grsot(\Xbsi;\pP_i,C_i)+\tau(1-\bar{x}^*)\rho_i+(\tau+\eta)\bar{x}^*\log(\bar{x}^*)\right]\\
    &= \bar{x}^*\grsbp(\Xbs)+\tau(1-\bar{x}^*)\sum_{i=1}^m\omega_i\rho_i+(\tau+\eta)\bar{x}^*\log(\bar{x}^*).
\end{align*}
It follows from $\bar{x}^*\Xbs \in \Dd(\bold{X})$ and the definition of $\Xbps$ that $\grsbp(\Xbps)\le\grsbp(\bar{x}^*\Xbs)$. Therefore, we have $\grsbp\left(\dfrac{\Xbps}{\bar{x}^*}\right)\le \grsbp(\Xbs)$. Since $\dfrac{\Xbps}{\bar{x}^*} \in \Dd_1(\bold{X})$ and the minimizer $\Xbs$ of function $\grsbp$ is unique, we obtain $\Xbsi = \dfrac{\Xbpsi}{\bar{x}^*} = \dfrac{\Xbpsi}{\|\Xbpsi\|_1}$ for all $i \in [m]$.
\subsection{Proof of Lemma \ref{lemma:geometric_converge_IBP}}
\begin{lemma}
\label{lemma:geometric_converge_IBP}
Let $(\uU^{k}, \vV^{k})$ be the updates of \textsc{RobustIBP} algorithm at the $k$-th step and $\uU^{*} = (u_1^*, \dots, u_m^*)$ and $\vV^{*} = (v_1^*, \dots, v_m^*)$ be the optimal solution of the dual problem \eqref{problem:dual_entropic_rsbp_no_norm}. Let $\Delta u^k_i := u^k_i - u^*_i$ and $\Delta v^k_i := v^k_i-v^*_i$ for $i \in [m]$. When $m = 2$ and $k$ is even, we obtain that
\begin{align*}
     \max\Big\{\sum_{i=1}^m\mnorm{\Delta u^{k+1}_i}, \sum_{i=1}^m\mnorm{\Delta v^{k+1}_i} \Big\} \leq \tau \Big(\frac{\tau}{\tau + \eta}\Big)^{k/2} \Rrsbp,
\end{align*}
where $\Rrsbp := \sum_{i=1}^m\Big(\max\Big\{\log(n),\frac{\mnorm{C_i}}{\eta}-\log(n)\Big\} +\mnorm{\log(\pP_i)}+\frac{\eta+\tau}{\eta\tau}\mnorm{C_i}\Big).$
\end{lemma}

\begin{proof}
Firstly, we will show that when $k$ is even, $k\ge 1$ and $m=2$,
\begin{align}
    \label{eq:rsbp:sum_delta}
    \max\Big\{\sum_{i=1}^m \mnorm{\Delta u^{k+1}_i},\sum_{i=1}^m\mnorm{\Delta v^{k+1}_i}\Big\}\le\left(\dfrac{\tau}{\tau+\eta}\right)^{k/2}\sum_{i=1}^m\mnorm{v^*_i}.
\end{align}
 Using the same arguments as deriving inequality \eqref{rsot5}, we have $\mnorm{\Delta u_i^{k + 1}} \le \frac{\tau}{\tau + \eta} \mnorm{\Delta v_i^k}$. Since $\{v^*_i\}^m_{i=1}$ are the fixed points of the update in Algorithm \ref{algorithm:barycenter_semiOT}, 
\begin{equation*}
    \frac{v^*_i}{\eta}=\left[\frac{v^*_i}{\eta}-\log(b^*_i)\right]-\sum_{t=1}^m \omega_t\left[\frac{v^*_t}{\eta}-\log(b^*_t)\right].
\end{equation*}
Combining the above equality with the update of $v_{i}^{k}$ in Algorithm~\ref{algorithm:barycenter_semiOT} and the fact $\sum_{t = 1}^m \omega_t = 1$, we find that
\begin{align*}
    \frac{\Delta v_i^k }{\eta}
    = \Delta V_i^{k - 1}-\sum_{t=1}^m \omega_{t} \Delta V_t^{k - 1}= \sum_{t \neq i} \omega_{t} (\Delta V_i^{k - 1}-\Delta V_t^{k - 1}).
\end{align*}
where
\begin{align*}
    \Delta V_i^k := \Big( \frac{v_i^k}{\eta} - \log(b_i^k) \Big) - \Big( \frac{v_i^*}{\eta} - \log(b_i^*) \Big) \quad \text{for all } i\in [m].
\end{align*}
Notice that Lemma \ref{lemma:2_UOT} can also be applied for this section, therefore, $\mnorm{\Delta V_i^k} \le \frac{\mnorm{\Delta u_i^k}}{\eta}$ for all $i \in [m]$. Collecting these results, we have
\begin{align*}
    \mnorm{\Delta v_i^k} \leq \sum_{t \neq i} \omega_{t}
    (\mnorm{\Delta u_t^{k - 1}} + \mnorm{\Delta u_i^{k - 1}}).
\end{align*}
When $m = 2$, these bounds show that
\begin{align*}
    \sum_{i = 1}^{m} \mnorm{\Delta v_i^k} \le \sum_{i = 1}^{m} \mnorm{\Delta u_i^{k - 1}}.
\end{align*}
Thus,
\begin{align*}
    \sum_{i=1}^m\mnorm{\Delta u^{k+1}_i}&\le \dfrac{\tau}{\tau+\eta}\sum_{i=1}^m\mnorm{\Delta u^{k-1}_i}\le\ldots\le\left(\dfrac{\tau}{\tau+\eta}\right)^{k/2}\sum_{i=1}^m\mnorm{\Delta u^{1}_i}\\
    &\le\left(\dfrac{\tau}{\tau+\eta}\right)^{(k+2)/2}\sum_{i=1}^m\mnorm{\Delta v^0_i}=\left(\dfrac{\tau}{\tau+\eta}\right)^{(k+2)/2}\sum_{i=1}^m\mnorm{v^*_i},
\end{align*}
which leads to
\begin{align*}
    \sum_{i = 1}^{m} \mnorm{\Delta v_i^k} \le \sum_{i = 1}^{m} \mnorm{\Delta u_i^{k - 1}}\le \left(\dfrac{\tau}{\tau+\eta}\right)^{k/2}\sum_{i=1}^m\mnorm{v^*_i}.
\end{align*}
Recall that $v^{k+1}_i=v^k_i$ for all $i\in [m]$ when $k$ is even. Then, putting all of the above results, we obtain equation \eqref{eq:rsbp:sum_delta}.\\
Next, we will prove that 
\begin{align}
    \label{eq:rsbp:v_R}
    \sum_{i=1}^m\mnorm{\vbsi}\le\tau\Rrsbp.
\end{align}
Since $\bold{u}^*$ is the fixed point of the update in Algorithm~\ref{algorithm:barycenter_semiOT} , we have
\begin{equation*}
    \frac{\ubsij}{\tau}=\log((\pP_i)_j)-\log\left(\sum_{l=1}^n\exp\left\{\frac{\ubsij+\vbsil-(C_i)_{jl}}{\eta}\right\}\right),
\end{equation*}
which is equivalent to,
\begin{equation*}
    \left(\frac{1}{\tau}+\frac{1}{\eta}\right)\ubsij=\log((\pP_i)_j)-\log\left(\sum_{l=1}^n\exp\left\{\frac{\vbsil-(C_i)_{jl}}{\eta}\right\}\right).
\end{equation*}
Therefore,
\begin{equation}
    \label{eq:rsbp:sum_u}
    \left(\frac{1}{\tau}+\frac{1}{\eta}\right)\sum_{i=1}^m\mnorm{\ubsi}\le\sum_{i=1}^m\left[\mnorm{\log(\pP_i)}+\frac{\mnorm{\vbsi}}{\eta}+\max\left\{\log(n),\frac{\mnorm{C_i}}{\eta}-\log(n)\right\}\right].
\end{equation}
For fixed $\uU^*$, we have that
\begin{align*}
\vV^* = \argmin_{\vV: \sum_{i = 1}^m \omega_i v_i = \zeros_n}\bar{h}_{\text{rsbp}}(\uU^*,\vV),
\end{align*}
or equivalently,
\begin{equation*}
    \vV^* = \argmin \sum_{i = 1}^{m} \omega_{i} \Big[ \eta \sum_{j, l = 1}^{n} \exp\Big\{\frac{\ubsij + (v_i)_l - (C_{i})_{jl}}{\eta}\Big\}+ \tau \big \langle e^{-\ubsi/ \tau}, \pP_{i} \big \rangle \Big] + \lambda^{\top}\sum_{i=1}^m\omega_iv_i,
\end{equation*}
where $\lambda\in\mathbb{R}^n$ is a vector of Lagrange multipliers. For each $i\in [m]$, taking derivatives of the RHS with respect to $v_i$,
\begin{align}
\exp\Big(\dfrac{\vbsi}{\eta}\Big)\odot A_{i} + \lambda &= \zeros_n\nonumber\\
\Leftrightarrow\frac{\vbsi}{\eta} + \log(A_i) &= \log(-\lambda).
\label{eq:rsbp:derivative}
\end{align}
where $A_i=\left(\sum_{j=1}^n\exp\Big\{\frac{\ubsij-(C_i)_{jl}}{\eta}\Big\}\right)^n_{l=1}$. \\
Next, taking sum over $i$ and utilizing the fact that $\sum_{i=1}^m\omega_i\vbsi=0$, we obtain $\sum_{i=1}^m\omega_i\log(A_i) = \log(-\lambda)$. Putting this result together with equation \eqref{eq:rsbp:derivative} leads to
\begin{equation*}
    \frac{\vbsi}{\eta}=\sum_{t=1}^m\omega_t\log(A_t)-\log(A_i) = \sum_{t=1}^m\omega_t\left[\log(A_t)-\log(A_i)\right].
\end{equation*}
Since $m=2$, the above equality indicates that $\frac{1}{\eta}\sum_{i=1}^m\mnorm{\vbsi}\le\mnorm{\log(A_2)-\log(A_1)}$. Furthermore, for all $l\in [n]$, applying part (a) of Lemma \ref{lemma:6_UOT},
\begin{align*}
    |\log(A_2)_l-\log(A_1)_l|&=\left|\log\left(\dfrac{\sum_{j=1}^n\exp\Big\{\frac{(u^*_2)_j-(C_2)_{jl}}{\eta}\Big\}}{\sum_{j=1}^n\exp\Big\{\frac{(u^*_1)_j-(C_1)_{jl}}{\eta}\Big\}}\right)\right|\\
    &\le\dfrac{1}{\eta}\max_{1\le j\le n}|(u^*_2)_j-(C_2)_{jl}-(u^*_1)_j+(C_1)_{jl}|\\
    &\le\dfrac{1}{\eta}\sum_{i=1}^m(\mnorm{\ubsi}+\mnorm{C_i}),
\end{align*}
which implies that 
\begin{align}
    \label{eq:rsbp:sum_v}
    \sum_{i=1}^m\mnorm{\vbsi}\le\eta\mnorm{\log(A_2)-\log(A_1)}\le\sum_{i=1}^m(\mnorm{u^*_i}+\mnorm{C_i}).
\end{align}
Combining equation \eqref{eq:rsbp:sum_u} with equation \eqref{eq:rsbp:sum_v}, we obtain
\begin{equation*}
     \sum_{i=1}^m\mnorm{\ubsi}\le\tau\sum_{i=1}^m\left[\mnorm{\log(\pP_i)}+\frac{\mnorm{C_i}}{\eta}+\max\left\{\log(n),\frac{\mnorm{C_i}}{\eta}-\log(n)\right\}\right].
\end{equation*}
Hence,
\begin{equation*}
    \sum_{i=1}^m\mnorm{\vbsi}\le\sum_{i=1}^m\left[\tau\mnorm{\log(\pP_i)}+\Big(1+\frac{\tau}{\eta}\Big)\mnorm{C_i}+\tau\max\left\{\log(n),\frac{\mnorm{C_i}}{\eta}-\log(n)\right\}\right] = \tau\Rrsbp.
\end{equation*}
From equations \eqref{eq:rsbp:sum_delta} and \eqref{eq:rsbp:v_R}, we get the conclusion of this lemma.
\end{proof}

\subsection{Proof of Theorem \ref{theorem:barycenter}}
Let $\Xbk=(X^k_1,\ldots,X^k_m)$ be the normalized output at $k$-th iteration of Algorithm \ref{algorithm:barycenter_semiOT}. We will firstly show that $\Xbk$ is an $\varepsilon$-approximation of $\widehat{\bold{X}}$ for all even $k\ge 2 +2\left(\frac{\tau}{\eta}+1\right)\log\left(\frac{4\Rrsbp\tau^2}{\eta^2}\right)$. By definition of $\frsbp$ and $\grsbp$, 
\begin{align*}
    \frsbp(\Xbk)-\frsbp(\Xbh)&=\grsbp(\Xbk)-\grsbp(\Xbh)+\eta\sum_{i=1}^m\omega_i\left[H(\Xbki)-H(\Xbhi)\right]\\
    &\le\grsbp(\Xbk)-\grsbp(\Xbs) +\eta\sum_{i=1}^m\omega_i\left[H(\Xbki)-H(\Xbhi)\right]
\end{align*}
The above two terms can be bounded as follows.\\
\textbf{Upper bound of} $\sum_{i=1}^m\omega_i\left[H(\Xbki)-H(\Xbhi)\right]$.\\
Applying the inequalities \eqref{eq:entropy} for the entropy function, we have
\begin{align}
    \label{eq:rsbp:h-h}
    \sum_{i=1}^m\omega_i\left[H(\Xbki)-H(\Xbhi)\right]\le\sum_{i=1}^m\omega_i[2\log(n)+1-1] = 2\log(n).
\end{align}
\textbf{Upper bound of} $\grsbp(\Xbk)-\grsbp(\Xbs)$.\\
Firstly, we consider the quantity $\grsbp(\Xbs)$.
\begin{align*}
    \grsbp(\Xbs)&=\grsbp\left(\frac{1}{\xbps}\Xbps\right)=\frac{1}{\xbps}\grsbp(\Xbps)+\tau \Big(1-\frac{1}{\xbps}\Big)\sum_{i=1}^m\omega_i\rho_i +(\tau+\eta)\log\Big(\frac{1}{\xbps}\Big)\\
    &=\frac{1}{\xbps}\left[\tau\sum_{i=1}^m\omega_i\rho_i - (\tau+\eta)\xbps\right]+\tau \Big(1-\frac{1}{\xbps}\Big)\sum_{i=1}^m\omega_i\rho_i -(\tau+\eta)\log(\xbps)\\
    &= -(\eta+\tau)-(\eta+\tau)\log(\xbps)+\tau\sum_{i=1}^m\omega_i\rho_i.
\end{align*}
The second equality is due to equation \eqref{eq:rsbp:g(tX)} and the third one results from equation \eqref{eq:rsbp:g_equality}.\\
Based on Remark \ref{remark:rsbp:xk} and the fact that $\Xbk=\dfrac{\Xbpk}{\xbpk}$, it is clear that $\Xbk$ is the optimal solution of
\begin{align*}
    & \min_{X_{1}, \ldots, X_{m} \in \mathbb{R}_{+}^{n \times n}} \grsbp^k(X_1,\ldots,X_m):=\sum_{i=1}^m\omega_i\left[\langle C_i,X_i\rangle +\tau\KL(X_i\one_n||\pP^k_i)-\eta H(X_i)\right] \nonumber \\
    & \text{s.t.} \ X_{i}^{\top} \one_{n} = X_{i + 1}^{\top} \one_{n} \ \text{for all} \ i \in [m - 1], \nonumber \\
    & \hspace{1 em} \|X_{i}\|_{1} = 1 \ \text{for all} \ i \in [m].
\end{align*}
Therefore, using the same arguments as for deriving for the quantity $\grsbp(\Xbs)$, we have
\begin{equation*}
    \grsbp^k(\Xbk) = -(\eta+\tau)-(\eta+\tau)\log(\xbpk)+\tau\sum_{i=1}^m\omega_i\rho^k_i.
\end{equation*}
where $\rho^k_i:=\onorm{\pP^k_i}$. Denote $a^k_i=\Xbpki\one_n$ for all $i\in [m]$. Writing $\grsbp(\Xbk)-\grsbp(\Xbs)=\left[\grsbp(\Xbk)-\grsbp^k(\Xbk)\right]+\left[\grsbp^k(\Xbk)-\grsbp(\Xbs)\right]$, using the above equations of $\grsbp^k(\Xbk)$ and $\grsbp(\Xbs)$, and the definitions of $\grsbp(\Xbk)$ and $\grsbp^k(\Xbk)$, we get
\begin{align*}
    \grsbp(\Xbk)-\grsbp(\Xbs) = (\eta+\tau)\log\Big(\frac{\xbps}{\xbpk}\Big)+\frac{\tau}{\xbpk}\sum_{i=1}^m\omega_i\sum_{j=1}^n\abkij\log\left(\frac{(\pP^k_i)_j}{(\pP_i)_j}\right).
\end{align*}
It follows from equation \eqref{eq:rsbp:x-x} that
\begin{equation*}
    \frac{1}{1+\frac{3}{\eta}\Delta^k_1}\le\frac{\xbps}{\xbpk}\le 1+\frac{3}{\eta}\Delta^k_1,
\end{equation*}
or equivalently,
\begin{equation*}
    \left|\log\left(\dfrac{\xbps}{\xbpk}\right)\right|\le\log\left(1+\frac{3}{\eta}\Delta^k_1\right)\le\frac{3}{\eta}\Delta^k_1\le\frac{3}{4}\frac{\eta}{\tau}.
\end{equation*}
Note that $(\pP^k_i)_j=\exp\Big(\frac{\ubkij}{\tau}\Big)\abkij$ and $(\pP_i)_j=\exp\Big(\frac{\ubsij}{\tau}\Big)\absij$, the second term can be bounded as follows 
\begin{align*}
    \tau\left|\log\left(\frac{(\pP^k_i)_j}{(\pP_i)_j}\right)\right|&=\tau\left|\frac{1}{\tau}(\ubkij-\ubsij)-\log\left(\frac{\absij}{\abkij}\right)\right| \\
    &\le|\ubkij-\ubsij|+\tau\left|\log\left(\frac{\absij}{\abkij}\right)\right|\\
    &\le \mnorm{\ubki-\ubsi}+\frac{\tau}{\eta}\left(\mnorm{\ubki-\ubsi}+\mnorm{\vbki-\vbsi}\right)\\
    &\le \left(\frac{2\tau+\eta}{\eta}\right)\Delta^k_i\\
    &\le \left(\frac{2\tau+\eta}{\eta}\right)\left(\frac{\eta^2}{4\tau}\right)\le\eta\left(\frac{1}{2}+\frac{1}{12\log(n)}\right).
\end{align*}
Therefore, 
\begin{align*}
    \left|\frac{\tau}{\xbpk}\sum_{i=1}^m\omega_i\sum_{j=1}^n\abkij\log\left(\frac{(\pP^k_i)_j}{(\pP_i)_j}\right)\right| &\le\eta\left(\frac{1}{2}+\frac{1}{12\log(n)}\right)\left[\frac{1}{\xbpk}\sum_{i=1}^m\omega_i\sum_{j=1}^n\abkij\right] \\
    &=\eta\left(\frac{1}{2}+\frac{1}{12\log(n)}\right).
\end{align*}
Combining the above bounds of the two terms leads to
\begin{equation}
    \grsbp(\Xbk)-\grsbp(\Xbs)\le\eta\left(\frac{5}{4}+\frac{1}{3\log(n)}\right)\le 2\eta.
    \label{eq:rsbp:g-g}
\end{equation}
Finally, from equations \eqref{eq:rsbp:h-h} and \eqref{eq:rsbp:g-g}, we obtain
\begin{align*}
    \frsbp(\Xbk)-\frsbp(\Xbh)\le \eta\left(2+2\log(n)\right)\le \eta \Ursbp = \varepsilon.
\end{align*}
\paragraph{The complexity of Algorithm \ref{algorithm:barycenter_semiOT}.} Next, we will derive the computational complexity of Algorithm \ref{algorithm:barycenter_semiOT}. By definition of $\Ursbp$, the order of this quantity is $\bigO(\log(n))$. Rewriting the sufficient number of iterations for obtaining an $\varepsilon$-approximation as below 
\begin{align*}
    2+2\left(\frac{\tau\Ursbp}{\varepsilon}\left[\log(4)+2\log(\tau)+\log(\eta\Rrsbp)+\log\Big(\frac{\Ursbp}{\varepsilon}\Big)\right]\right),
\end{align*}
which leads to
\begin{equation*}
    k=\bigO\left(\frac{\tau\log(n)}{\varepsilon}\left[\log(\tau)+\log(\mnorm{C_1}+\mnorm{C_2})+\log\Big(\frac{\log(n)}{\varepsilon}\Big)\right]\right).
\end{equation*}
Multiplying with $\bigO(n^2)$ arithmetic operations per iteration, we get the final complexity.

\section{Robust Unconstrained Optimal Transport: Useful Lemmas and Omitted Proofs}
\label{appendix:rot:proofs}
In this appendix, we continue to discuss in-depth the ROT problem, which is briefly introduced in Section \ref{sec:rot}. Similar to RSOT, solving directly the optimization problem \eqref{eq:robust_OT} would be computationally expensive, particularly when $n$ is large. This encourages us to work on the entropic version of the problem~\eqref{eq:robust_OT}, which admits the following form:
\begin{align}
    \label{eq:entropic_rot}
     \min_{X\in \Br^{n \times n}_+; \|X\|_{1} = 1} \grot(X) : = \frot(X) - \eta H(X),
\end{align}
for some $\eta>0$. We name this objective \textit{entropic ROT}. A general approach to solve this optimization problem is to derive its Fenchel duality, then performing alternating minimization on dual variables.
\begin{lemma}
\label{lemma:dual_entropic_rot}
The dual form of the entropic ROT problem in equation~\eqref{eq:entropic_rot} admits the following form
\begin{align}
    \label{eq:robust_ot_entropic_dual}
    \min_{u, v \in \Br^{n}} h(u,v) : = \eta \log \onorm{B(u, v)} + \tau \big \langle e^{-u/ \tau}, \aA \big \rangle + \tau \big \langle e^{-v/ \tau}, \bB \big \rangle.
\end{align}
\end{lemma}
\begin{proof}[Proof of Lemma \ref{lemma:dual_entropic_rot}]
The objective function~\eqref{eq:entropic_rot} can be rewritten as follows
\begin{align*}
    \min_{\substack{X\in \Br^{n \times n}, \|X\|_{1} = 1; \\X \one_{n} = y, X^{\top} \one_{n} = z}} &\left\langle C, X\right\rangle - \eta H(X) + \tau \KL(y || \aA) + \tau \KL(z || \bB).
\end{align*}
By introducing the dual variables $u \in \mathbb{R}^{n}$ and $v \in \mathbb{R}^{n}$, the Lagrangian duality of the above objective function takes the following form
\begin{align*}
    \max_{u, v \in \mathbb{R}^{n}} \min_{\substack{X\in \Br^{n \times n}, \|X\|_{1} = 1; \\ y, z \in \Br^{n}}} \left\langle C, X\right\rangle -\eta H(X) &+ \tau \KL(y || \aA) + \tau \KL(z || \bB) \\
    &- u^{\top} (X \one_{n} - y) - v^{\top} (X^{\top} \one_{n} - z). 
\end{align*}
We can check that
\begin{align*}
    \min_{y \in \Br^n} \tau \KL(y || \aA) + u^{\top} y &= - \tau \left \langle e^{-u/ \tau}, \aA \right \rangle + \aA^{\top} \one_{n}, \\
    \min_{z \in \Br^n} \tau \KL(z || \bB) + v^{\top} z &= - \tau \left \langle e^{-v/ \tau}, \bB \right \rangle + \bB^{\top} \one_{n}.
\end{align*}
Furthermore, for the minimization problem
\begin{align*}
    \min_{X\in \Br^{n \times n}, \|X\|_{1} = 1} \left\langle C, X\right\rangle - u^{\top} X \one_{n} - v^{\top} X^{\top} \one_{n} - \eta H(X),
\end{align*}
the objective function is strongly convex. Therefore, it has an unique global minima. Direct calculations demonstrate that the optimal solution of that objective function takes the following form
\begin{align*}
    \bar{X} = \frac{B(u, v)}{\onorm{B(u, v)}}, \quad \text{where } B(u, v)_{ij} := \exp \Big( \frac{u_i + v_j - C_{ij}}{\eta} \Big).
\end{align*}
Based on the above argument, we can check that
\begin{align*}
    \min_{X\in \Br^{n \times n}, \|X\|_{1} = 1} \left\langle C, X\right\rangle - u^{\top} X \one_{n} - v^{\top} X^{\top} \one_{n} - \eta H(X) = - \eta \log \onorm{B(u, v)}. 
\end{align*}
Combining all the above results, we obtain the conclusion.
\end{proof}
Strong duality holds for the problem~\eqref{eq:entropic_rot}, and its optimal solution can be obtained via the optimal solution of the problem~\eqref{eq:robust_ot_entropic_dual}, i.e., $X^* = B(u^*, v^*)$. To solve the latter,  we can set the partial derivatives of its objective with respect to $u$ and $v$ to zero, resulting in
\begin{align*}
\frac{B(u, v) \ones_n}{\onorm{B(u, v)}} &= e^{-u/\tau} \odot \aA, \quad \frac{B(u, v)^T \ones_n}{\onorm{B(u, v)}} = e^{-v/\tau} \odot \bB,
\end{align*}
where $\odot$ denoting element-wise multiplication. It is challenging to derive closed-form solutions for each coordinate $u_i$ and $v_j$ for $i,j \in [n]$ from this system of equations. Consequently, we do not get a direct update for $u_i$ and $v_j$ in  the coordinate descent algorithm. Therefore, developing directly Sinkhorn algorithm for solving entropic ROT like the RSOT case could be non-trivial.
\begin{algorithm}[H]
        \caption{\textsc{Robust-Sinkhorn}} \label{algorithm:rot:sinkhorn}
        \begin{algorithmic}
        \STATE \textbf{Input:} $C, \aA, \bB, \tau, \eta, k_{iter}$
        \STATE \textbf{Output:} $X$
        \STATE \textbf{Initialization:} $u^0 = v^0 = \zeros_n, k = 0$
        \WHILE {$k < k_{iter}$}
        \STATE $a^k = B(u^k,v^k) \ones_n$
        \STATE $b^k = (B(u^k,v^k))^{\top} \ones_n$
        \IF {$k$ is even}
        \STATE $u^{k+1} \leftarrow \frac{\eta \tau}{\eta + \tau} \big[\frac{u^k}{\eta} + \log(\bold{a}) - \log(a^k)\big]$
        \STATE $v^{k+1} \leftarrow v^k$
        \ELSE
        \STATE $u^{k+1} \leftarrow u^k$
        \STATE $v^{k+1} \leftarrow \frac{\eta \tau}{\eta + \tau}
        \big[\frac{v^k}{\eta} + \log(\bold{b}) - \log(b^k)\big]$
        \ENDIF
        \STATE $k = k + 1$
        \ENDWHILE
        \STATE \textbf{return} $X^k = B(u^k,v^k)/\onorm{B(u^k,v^k)}$
        \end{algorithmic}
    \end{algorithm}

It is worth noting that the required iteration to reach an $\varepsilon$-approximation of UOT is \textit{not} identical to that of ROT, or in a broader sense, it is not trivial to derive one from the other. 
Hence, in the following theorem, we present one of our main results regarding the complexity of \textsc{Robust-Sinkhorn} algorithm in reaching an $\varepsilon$-approximation of ROT.

\begin{theorem}
\label{theorem:rot:complexity}
For $\eta=\varepsilon\Urot^{-1}$ where
\begin{align*}
    \Urot = \max\Big\{\frac{3(\tau + 2)}{4(\tau+1)}+2\log(n),2\varepsilon,\frac{5\varepsilon\log(n)}{\tau}\Big\},
\end{align*}
Algorithm~\ref{algorithm:rot:sinkhorn} returns an $\varepsilon$-approximation of the optimal solution $\Xroth$ for the problem \eqref{eq:robust_OT} in time
\begin{align*}
    \bigO\left(\frac{\tau n^2}{\varepsilon}\log(n)\left[\log\left(\frac{\tau\mnorm{C}}{\varepsilon}\right)+\log(\log(n))\right]\right).
\end{align*}
\end{theorem}
The result of Theorem~\ref{theorem:rot:complexity} shows that the complexity of \textsc{Robust-Sinkhorn} algorithm for computing ROT is at the order of $\bigOtil(\frac{n^2}{\varepsilon})$, which is near-optimal and at the same order as that of the Sinkhorn algorithm for solving UOT~\cite{pham2020unbalanced}. Furthermore, similar to the RSOT case, the complexity of \textsc{Robust-Sinkhorn} algorithm is also better than that of the Sinkhorn algorithm for computing the standard optimal transport problem.

\subsection{Useful Lemmas}
Prior to presenting the proof of Theorem \ref{theorem:rot:complexity}, in this section, we provide the proof of Lemma~\ref{lemma:rot:Xrots} as well as several useful properties of ROT and UOT that will be used later on.
\begin{proof}[Proof of Lemma~\ref{lemma:rot:Xrots}]
Using the equation for $\grot(tX)$ in \eqref{equation:uot:gtx}, we have that
\begin{align*}
  \grot(\Xuots) &= \grot\left( (\xuots) \left( \frac{\Xuots}{ \xuots} \right)\right) \\
  &= \xuots \grot\Big(\frac{\Xuots}{ \xuots}\Big) + \tau\big(1 - \xuots\big)(\alpha + \beta) + (2\tau + \eta) \xuots \log(\xuots)\\
  \grot(\xuots \Xrots) &= \xuots \grot(\Xrots) + \tau\big(1-\xuots\big) (\alpha + \beta) + (2\tau + \eta) \xuots \log(\xuots).
\end{align*}
In terms of the left-handed sides, $\grot(\Xuots) \leq \grot(\xuots \Xrots)$ by definition of $\Xuots$. On the right-handed sides, the second and third are the same. Thus, from the above two equations we obtain
\begin{align*}
  \grot\Big(\frac{\Xuots}{\xuots}\Big)   \leq \grot(\Xrots).
\end{align*}
As the optimization problem of ROT has an unique solution, $\Xrots  = \frac{\Xuots}{\xuots}$.
\end{proof}
\begin{lemma}[\textbf{Convergence rate for $\uuotk$ and $\vuotk$}]
\label{lemma:uot:bound_optimal_dual}
For any $k\ge 1+\left(\frac{\tau}{\eta}+1\right)\log\left(\frac{8 \Rrot\tau(\tau+1)}{\eta^2}\right)$, the updates $(u^{k}_{\mathrm{uot}},v^{k}_{\mathrm{uot}})$ from Algorithm \ref{algorithm:rot:sinkhorn} can be bounded as follows,
\begin{align*}
    \Delta^{k}_{\mathrm{uot}}:= \max\{ \mnorm{\uuotk - \uuots}, \mnorm{\vuotk - \vuots} \} \le \frac{\eta^2}{8(\tau+1)}.
\end{align*}
\begin{proof}[Proof of Lemma \ref{lemma:uot:bound_optimal_dual}]
This lemma is the combination of Theorem 1 and Lemma 5 part (a) in \cite{pham2020unbalanced}.
\end{proof}
\end{lemma}

\begin{lemma}
\label{lemma:4_UOT}
Let $\xuots := \onorm{\Xuots}$, then the quantity $\grot(\Xuots)$ is presented as
\begin{align*}
    \grot(\Xuots)+2(\tau+\eta)\xuots = \tau(\alpha+\beta).
\end{align*}
\end{lemma}
\begin{proof}[Proof of Lemma \ref{lemma:4_UOT}]
The proof of this lemma can be found in Lemma 4 of \cite{pham2020unbalanced}.
\end{proof}
\begin{lemma}
\label{lemma:rot:equality:g_optimal} We have the following relation between the optimal value of entropic ROT and other parameters
\begin{align*}
    \grot(\Xrots) = \tau (\alpha + \beta - 2) - \eta - (2\tau + \eta) \log(\xuots).
\end{align*}
Furthermore, let $\grot^{k}(X) :=\langle C, X\rangle-\eta H(X) +\tau \mathbf{K L}\left(X \mathbf{1}_{n} \| \mathbf{a}\right)+\tau \mathbf{K L}\left(X^{\top} \mathbf{1}_{n} \| \mathbf{b}_{\mathrm{uot}}^{k}\right)$, with $\mathbf{b}_{\mathrm{uot}}^{k} := \exp \left(\frac{\vuotk}{\tau}\right) \odot\left[\left(\Xuotk\right)^{T} \mathbf{1}_{n}\right]$ and $\beta_{\mathrm{uot}}^k := \onorm{\mathbf{b}_{\mathrm{uot}}^{k}}$. If $k$ is odd, we have that
\begin{align*}
    \grot^k(\Xrotk) = \tau (\alpha + \beta_{\mathrm{uot}}^k - 2) - \eta - (2\tau + \eta) \log(\xuotk).
\end{align*}
\end{lemma}

\begin{proof}[Proof of Lemma \ref{lemma:rot:equality:g_optimal}]
First, we recall from Lemma 4 \cite{pham2020unbalanced} that, for $t \in \RR_+$ and $X \in \RR_+^{n \times n}$,
\begin{align}
\label{equation:uot:gtx}
    \grot(tX) = t\grot(X) + \tau(1-t)(\alpha + \beta) + (2\tau + \eta) x t \log(t).
\end{align}
Applying this equation with $X = \Xrots$ and $t = \xuots$, we obtain
\begin{align*}
    \grot(\Xuots)= \xuots \grot(\Xrots) +  \tau(1-\xuots)(\alpha+\beta) + (2\tau+\eta) \xuots \log(\xuots).
\end{align*}
Combining with the fact that $\grot \left(\Xuots\right)+(2 \tau+\eta) \xuots = \tau(\alpha+\beta)$ stated in Lemma \ref{lemma:4_UOT}, we get the final equality for $\grot(\Xrots)$. Finally, note that $\Xuotk = \argmin \grot^k(X)$, the same argument thus can be applied, and we obtain the equality for $\grot^k(\Xrotk)$.
\end{proof}

\subsection{Proof of Theorem \ref{theorem:rot:complexity}}
First, we will show that $\Xrotk$ is an $\varepsilon$-approximation of $\Xroth$ for all $k\ge 1+\left(\frac{\tau}{\eta}+1\right)\log\left(\frac{8 \Rrot\tau(\tau+1)}{\eta^2}\right)$. By  definitions of $\frot$ and $\grot$, we have
\begin{align}
    \frot(\Xrotk) - \frot(\Xroth) \nonumber
    &= \grot(\Xrotk)+\eta H(\Xrotk) - \grot(\Xroth) - \eta H(\Xroth) \nonumber \\
    &\leq \Big[\grot(\Xrotk)-\grot(\Xrots)\Big] + \eta \Big[H(\Xrotk) - H(\Xroth)\Big],
\end{align}
\paragraph{Upper bound of $H(\Xrotk) - H(\Xroth)$.} Since $\|\Xrotk\|_1 = \|\Xroth\|_1=1$, applying the lower and upper bounds for the entropy in \eqref{eq:entropy}, we have
\begin{equation}
    \label{eq:rot:diff_entropy}
    H(\Xrotk) - H(\Xroth) \le 2 \log(n).
\end{equation}



\paragraph{Upper bound of $\grot(\Xrotk) - \grot(\Xrots)$.} WLOG, we consider the case where $k$ is odd. By Lemma \ref{lemma:rot:equality:g_optimal},
\begin{align}
    \grot(\Xrots) &= \tau (\alpha + \beta - 2) - \eta - (2\tau + \eta) \log(\xuots) \label{eq:rot_objective_function_equation} \\
    \grot^k(\Xrotk) &= \tau (\alpha + \beta_{\mathrm{uot}}^k - 2) - \eta - (2\tau + \eta) \log(\xuotk). \label{eq:rot_objective_function_equation_k}
\end{align}
Writing $\grot(\Xrotk) - \grot(\Xrots) = \left[ \grot(\Xrotk) - \grot^k(\Xrotk) \right] + \left[ \grot^k(\Xrotk) - \grot(\Xrots) \right]$. For the first term, we have
\begin{align*}
    \grot(\Xrotk) &= \langle C, \Xrotk \rangle + \tau \KL(\Xrotk \ones_n \| \aA) + \tau \KL((\Xrotk)^T \ones_n \| \bB) - \eta H(\Xrotk) \\
    \grot^k(\Xrotk) &= \langle C, \Xrotk \rangle + \tau \KL(\Xrotk \ones_n \| \aA) + \tau \KL((\Xrotk)^T \ones_n \| \bB^k_{\textnormal{uot}}) - \eta H(\Xrotk).
\end{align*}
Then, we find that
\begin{align}
    \grot(\Xrotk)-\grot^k(\Xrotk)  
    &= \tau \Big[ \KL(\underbrace{(\Xrotk)^T \ones_n}_{:= b^k_{rot}}\| \bB)-\KL((\Xrotk)^T \ones_n  \| \bB^k_{\textnormal{uot}})\Big] \nonumber \\
    &= \tau \left[ \sum_{j=1}^n (b_{\text{rot}}^k)_j \log \left( \frac{(\bB_{\mathrm{uot}}^k)_j}{\bB_j} \right) + (\beta-\beta^k_{\textnormal{uot}}) \right]. \label{equation:rot:diff_g_k}
\end{align}
Combining equations \eqref{eq:rot_objective_function_equation}, \eqref{eq:rot_objective_function_equation_k} and \eqref{equation:rot:diff_g_k}, we obtain
\begin{align}
\label{eq:rot:diff_g_k_2}
\grot(\Xrotk) - \grot(\Xrots) = (2\tau +\eta) \log \left( \frac{\xuots}{\xuotk} \right) +\tau \left[ \sum_{j=1}^n (b_{\text{rot}}^k)_j \log \left( \frac{(\bB_{\mathrm{uot}}^k)_j}{\bB_j} \right) \right].
\end{align}
Using the following result 
\begin{align*}
    \max\Big\{\frac{\xuots}{\xuotk},\frac{\xuotk}{\xuots}\Big\}\leq \left(\frac{\mnorm{\uuotk-\uuots}}{\eta}\right)\left(\frac{\mnorm{\vuotk-\vuots}}{\eta}\right)
\end{align*}
in the proof of Lemma 5 part (b) in \cite{pham2020unbalanced}, the first term is bounded by $\frac{2(2 \tau + \eta)}{\eta} \Delta_\mathrm{uot}^k$.

\noindent Let $\buotk:=(\Xuotk)^{\top}\one_n$ and $\buots:=(\Xuots)^{\top}\one_n$. Note that $(\bB^k_{\textnormal{uot}})_j=\exp\left(\frac{\vuotkj}{\eta}\right)\buotkj$ and $\bB_j=\exp\left(\frac{\vuotsj}{\eta}\right)\buotsj$. Applying part (b) of Lemma \ref{lemma:2_UOT}, we find that
\begin{align*}
    \left|\log \left(\frac{(\mathbf{b}_{\mathrm{uot}}^{k})_j}{\mathbf{b}_{j}}\right)\right|& = \left|-\log\left(\frac{\buotsj}{\buotkj}\right)+\frac{1}{\tau}[\vuotkj-\vuotsj]\right|\\ 
    &\le \frac{2}{\eta}\Delta^k_{\mathrm{uot}}+\frac{1}{\tau}\Delta^k_{\mathrm{uot}} =\left(\frac{2}{\eta}+\frac{1}{\tau}\right) \Delta^{k}_{\mathrm{uot}},
\end{align*}
which leads to
\begin{align*}
\left| \sum_{j=1}^n (b_\text{rot}^k)_j \log \left( \frac{(\bB_{\mathrm{uot}}^k)_j}{\bB_j} \right) \right| \leq \underbrace{\left( \sum_{j=1}^n (b^k_\text{rot})_j \right)}_{= \onorm{\Xrotk} = 1} \max_{1\leq j\leq n} \left| \log \left( \frac{(\bB_{\mathrm{uot}}^k)_j}{\bB_j} \right) \right| \leq \left(\frac{2}{\eta}+\frac{1}{\tau}\right) \Delta^{k}_{\mathrm{uot}}.
\end{align*}
Collecting all the inequalities for each term in \eqref{eq:rot:diff_g_k_2}, we obtain
\begin{align*}
    \grot(\Xrotk) - \grot(\Xrots) \le \frac{3}{\eta}(2\tau + \eta) \Delta_\mathrm{uot}^k.
\end{align*}
Furthermore, from Lemma \ref{lemma:uot:bound_optimal_dual}, we get $\Delta_\mathrm{uot}^k\le\frac{\eta^2}{8(\tau+1)}$. Then,
\begin{equation}
    \label{eq:rot:diff_g_k_3}
    \grot(\Xrotk) - \grot(\Xrots) \le \frac{3\eta(2\tau + 4) }{8(\tau +1)}= \eta \Big[ \frac{3(\tau +2)}{4(\tau +1)} \Big]. 
\end{equation}
Putting the results from equations \eqref{eq:rot:diff_entropy} and \eqref{eq:rot:diff_g_k_3} leads to 
\begin{align*}
    \frot(\Xrotk) - \frot(\Xroth) \le \eta\left[\frac{3(\tau + 2)}{4(\tau +1)} + 2\log(n)\right] \le \eta \Urot = \varepsilon.
\end{align*}
\paragraph{The complexity of Algorithm \ref{algorithm:rot:sinkhorn}.} Next, we will compute the complexity of Algorithm \ref{algorithm:rot:sinkhorn} under the assumption that $\Rrot=\bigO\left(\frac{1}{\eta}\mnorm{C}\right)$. The sufficient number of iterates to obtain an $\varepsilon$-approximation of $\Xroth$ can be rewritten as
\begin{align*}
    \left(\frac{\tau\Urot}{\varepsilon}+1\right)\left[\log(\eta\Rrot)+\log(\tau(\tau+1))+\log\Big(\frac{\Urot}{\varepsilon}\Big)\right].
\end{align*}
By the definition of $\Urot$, we find that $\Urot=\bigO(\log(n))$. Overall,
\begin{equation*}
    k = \bigO\left(\frac{\tau\log(n)}{\varepsilon}\left[\log(\mnorm{C})+\log(\tau)+\log(\log(n))+\log\Big(\frac{1}{\varepsilon}\Big)\right]\right).
\end{equation*}
By multiplying the above bound of $k$ with $\bigO(n^2)$ arithmetic operations per iteration, we get the desired complexity.

\section{Details on Low-Rank Approximation}
\label{section:nystrom:proofs}
Though previous complexity analyses of standard Sinkhorn algorithms are favorable in terms of $\varepsilon$, they exhibit quadratic growth with regards to $n$ in both time and space complexity. Therefore, they are unscalable when $n$ is huge in practice. As the robust Sinkhorn algorithms mainly involve matrix-vector multiplications, the computational cost can be reduced by utilizing special structures of some factors, such as the Gaussian kernel matrix $K : = \exp \big( \frac{-C}{\eta} \big)$. By approximating $K$ with a low-rank matrix, we show that the proposed robust Sinkhorn algorithms can be sped up considerably with a high probability while still reaching a nearly-optimal solution. A similar approach based on \Nystrom method had been studied in the optimal transport problem \cite{Altschuler-2018-Massively}. In this section, building on these analyses, we provide some novel results for scaling up the robust algorithms developed in previous sections. The idea of \Nystrom approximation is that given a kernel matrix $K$ where $K_{ij} = k(x_i, x_j)$ are constructed from $n$ data points $\Xx = \{ x_1, \dots, x_n \} \subset \RR^d$, with $k: \Xx \times \Xx \to \RR$ being a kernel function, we select $r$ points $\{x_{p_1}, \dots, x_{p_r}\} \subset \Xx$ to construct two matrices: $V \in \RR^{n \times r}$ where $V_{ij} = k(x_i, x_{p_j})$ and $A \in \RR^{r \times r}$ where $A_{ij} = k(x_{p_i}, x_{p_j})$. An approximation of $K$ is given by $\widetilde{K} = V A^{-1} V^{\top}$, which is the kernel matrix of the dataset after being projected onto the space of the chosen subset. Whether $\widetilde{K}$ is a good approximation of  $K$ depends on $r$ and the art of selecting $r$ data points. In Algorithm \ref{algorithm:rot:nystrom}, we make use of the adaptive procedure namely $\textsc{AdaptiveNystr\"{o}m}$ from \cite{Altschuler-2018-Massively} to obtain $\widetilde{K}$, which subsequently is used in the \textsc{Robust-SemiSinkhorn} (or \textsc{Robust-Sinkhorn}) algorithm. We show in Theorem \ref{theorem:rot:nystrom} that,  with some specific choices of parameters, we could obtain matrix $\widetilde{K}$ such that an $\varepsilon$-approximation is achievable in almost linear time.

\begin{algorithm}[H]
    \caption{\textsc{Adaptive\Nystrom}} \label{algorithm:adaptive_nystrom}
    \begin{algorithmic}
    \STATE \textbf{Input:} $\Xx=\{x_1, x_2, ..., x_n\}, \eta > 0, \tau > 0$ 
    \STATE \textbf{Output:} $\widetilde{K} \in \RR^{n \times n}, r \in \mathbb{N}$ 
    \STATE $\mathrm{err} \gets +\infty, r \gets 1$
    \WHILE{$\mathrm{err} > \tau$}
    \STATE $r \to 2r$
    \STATE $\widetilde{K} \leftarrow \textsc{\Nystrom}(\Xx, \eta, r)$
    \STATE $\mathrm{err} \gets 1 - \min_{i \in [n]} \widetilde{K}_{ii}$
    \ENDWHILE
    \STATE \textbf{return } $(\widetilde{K}, \mathrm{rank}(\widetilde{K}))$
    \end{algorithmic}
\end{algorithm}

\begin{algorithm}[h]
    \caption{\textsc{Robust-NysSink}}
    \begin{algorithmic}
    \label{algorithm:rot:nystrom}
    \STATE \textbf{Input:} $\Xx = \{x_1, \dots, x_n: \|x_i\|_2 \le R\}, \aA, \bB, \eta, \tau, \varepsilon, k$ 
    \STATE $Z \leftarrow 1 + 2(\tau + \eta)$ or $2 + \eta + \frac{2 \tau}{\eta}$ \quad (RSOT or ROT)
    \STATE $\varepsilon' \leftarrow \min(1,\frac{\varepsilon}{Z})$
    \STATE $(\widetilde{K}, r) \leftarrow \textsc{AdaptiveNystr\"{o}m}(\Xx, \eta, \frac{\varepsilon'}{2}e^{-4\eta^{-1} R^2})$
    \STATE $\widetilde{C} \leftarrow -\eta \log \widetilde{K}$
    \STATE $\widehat{X} \leftarrow \textsc{Robust-(Semi)Sinkhorn}(\widetilde{C}, \aA, \bB, \eta, \tau, k)$
    \STATE \textbf{Output:} $\widehat{X}$ 
    \end{algorithmic}
\end{algorithm}

\begin{theorem}
\label{theorem:rot:nystrom}
We denote by $f_C$ the objective function of RSOT \eqref{eq:semi_robust_OT} and ROT \eqref{eq:robust_OT} problems regarding some cost matrix $C$. Furthermore, let $\widehat{X}_C$ be the corresponding optimal solution, and $X^k_{\widetilde{C}}$ be the output of Algorithm \ref{algorithm:rot:nystrom} for $k$ Sinkhorn iterations. Then, for $0 < \varepsilon < 1$, Algorithm \ref{algorithm:rot:nystrom} achieves an $\varepsilon$-approximation $X^k_{\widetilde{C}}$ of $\widehat{X}_C$, i.e., $f_C(X^k_{\widetilde{C}}) - f_C(\widehat{X}_C) \leq \varepsilon$, in $\widetilde{O}(nr^2 + \frac{nr}{\varepsilon})$ calculations.
\end{theorem}
Theorem~\ref{theorem:rot:nystrom} indicates that using $\Nystrom$ approximation reduces the original complexity of the robust algorithms by a factor $n/r^2$. As a side note, \cite{Altschuler-2018-Massively} provides a probabilistic bound on $r$ (for more detail see Appendix \ref{section:nystrom:proofs}). Furthermore, in terms of space complexity, Algorithm \ref{algorithm:rot:nystrom} uses $O(n(r+d))$ space, where $d$ is the dimension of data constructing the cost matrix $C$.

Subsequently, we derive the complexity of Sinkhorn-based algorithms using $\Nystrom$ approximation in both RSOT and ROT problems. As the proof for both problems share many similarities, we abuse the notation by using the same notations for both cases.  In particular, we denote $f_C$ to be the objective functions of RSOT and ROT  as in \eqref{eq:semi_robust_OT}  and  \eqref{eq:robust_OT} , respectively, with $C$ is the cost matrix. Similarly we denote $g_C$ to be the objective functions with entropic regularization of  RSOT and ROT as in \eqref{eq:semi_robust_OT_entropic} and   \eqref{eq:entropic_rot}, respectively.  We recall and define some other quantities as follow:
\begin{align*}
   \widehat{X}_C &= \argmin f_C(X),  \\
   X_C^* &=  \argmin g_C(X), \\
   X_{\widetilde{C}}^* &= \argmin g_{\widetilde{C}}(X); 
\end{align*}
where $\widetilde{C}$ is the matrix produced by the $\Nystrom$ method. For other notations, we remove the index rsot and rot in quantities i.e. $\ursotk$ in order to keep them simple. 
\begin{proof}[Proof of Theorem \ref{theorem:rot:nystrom}] 
Assume that we have following bounds
\begin{align}
    \| X^k_{\widetilde{C}} \|_1 &\leq S_x, \label{result:bound_Xk} \\
    g_{\widetilde{C}}(X_{\widetilde{C}}^k) - g_{\widetilde{C}}(X_{\widetilde{C}}^*) &\le \eta S_g, \label{result:bound_g_diff} \\
    H(X^k_{\widetilde{C}})- H(\widehat{X}_C) &\le S_H, H(X^k_{\widetilde{C}})- H(\widehat{X}_{\widetilde{C}}) \le S_H, \label{result:bound_h_diff} \\
    \big|g_{\widetilde{C}}(X^*_{\widetilde{C}}) - g_C(X^*_C) \big| &\le S_C \big\| C- \widetilde{C} \big\|_{\infty}, \label{result:bound_g_diff_2}
\end{align}
where $S_x, S_g, S_H, S_C$ are constants that may contain $\alpha, \beta, \eta, \tau$ or $C$, varying between cases.

\noindent By definitions of $\widehat{X}_C$ and $X^*_{\widetilde{C}}$, we have
\begin{align*}
    f_C(\widehat{X}_C) 
    &= g_C(\widehat{X}_C) + \eta H(\widehat{X}_C) \geq g_C(X^*_C) + \eta H(\widehat{X}_C),
\end{align*}
and
\begin{align*}
    f_C(X_{\widetilde{C}}^k) &\leq \big|f_C(X^k_{\widetilde{C}}) - f_{\widetilde{C}}(X^k_{\widetilde{C}})\big| + f_{\widetilde{C}}(X^k_{\widetilde{C}}) \\
    &= \big|\langle C-\widetilde{C}, X_{\widetilde{C}}^k \rangle\big| + \eta H(X_{\widetilde{C}}^k)+ g_{\widetilde{C}}(X_{\widetilde{C}}^k).
\end{align*}
For the first term, using Holder's inequality and \eqref{result:bound_Xk} we get $\big|\langle C-\widetilde{C}, X_{\widetilde{C}}^k \rangle\big| \leq \big\| C- \widetilde{C} \big\|_{\infty} \big\| X_{\widetilde{C}}^k \big\|_1 \le \big\| C- \widetilde{C} \big\|_{\infty} S_x$. Combining with (\ref{result:bound_g_diff}), we have $f_C(X_{\widetilde{C}}^k)$ is bounded by
\begin{align*}
    \big \| C- \widetilde{C} \big\|_{\infty} S_x + \eta H(X^k_{\widetilde{C}})+ \eta S_g + g_{\widetilde{C}}(X^*_{\widetilde{C}}).  
\end{align*}
We thus obtain
\begin{align*}
f_C(X_{\widetilde{C}}^k) - f_C(\widehat{X}_C)
&\leq \big\| C- \widetilde{C} \big\|_{\infty} S_x + \eta \underbrace{\big(H(X^k_{\widetilde{C}})- H(\widehat{X}_C) \big)}_{\le S_H} + \eta S_g + \underbrace{(g_{\widetilde{C}}(X_{\widetilde{C}}^*) - g_C(X_C^*))}_{\le S_C \mnorm{C- \widetilde{C}}} \\
&\leq \| C- \widetilde{C} \|_{\infty} S_x + \eta S_H + \eta S_g + S_C \big\| C- \widetilde{C} \big\|_{\infty} \\
&= (\underbrace{\eta S_H + \eta S_g}_{\le \varepsilon'}) + (S_x + S_C) \underbrace{\mnorm{C- \widetilde{C}}}_{=\eta \mnorm{\log(K)- \log(\widetilde{K})}}\\
&\le \varepsilon' + (S_x + S_C) \eta \mnorm{\log(K)- \log(\widetilde{K})} \\
&\le \varepsilon' + (S_x + S_C) \eta \varepsilon' \\
&= \varepsilon' (1 + \eta S_x + \eta S_C) \\
&= \varepsilon,
\end{align*}
where the third inequality $\eta S_H + \eta S_g \le \varepsilon'$ comes from using \textsc{Robust-(Semi)Sinkhorn} algorithm on the approximated cost $\widetilde{C}$ with the error $\varepsilon'$, and the fourth inequality $\mnorm{\log(K) - \log(\widetilde{K})} \le \varepsilon' $ is a result of the $\textsc{Adaptive\Nystrom}$procedure (see Lemma L, \cite{Altschuler-2018-Massively}).

\paragraph{Time complexity.} Since $S_x=\widetilde{O}(1)$ and $S_C = \widetilde{O}(1)$, we get $\widetilde{O}(\frac{1}{\varepsilon'})=\widetilde{O}(\frac{1+\eta S_X + \eta S_C}{\varepsilon})=\widetilde{O}(\frac{1}{\varepsilon})$ . The \textsc{AdaptiveNystr\"{o}m} routine takes $O(nr^2)$ time, while the \textsc{Robust-(Semi)Sinkhorn} routine runs through $\widetilde{O}(\frac{1}{\varepsilon'}$) iterations. Each iteration then takes $O(n+nr)=O(nr)$ time, in which $O(n)$ for vector additions, and $O(nr)$ for low-rank matrix vector multiplications. In total, the time complexity is $\widetilde{O}(nr^2 + \frac{nr}{\varepsilon'})$. 

\paragraph{Space complexity.} As we only need to save the implicit form of $\widetilde{K}$ via two matrices $KS \in \RR^{n \times r}$ and $(S^T K S)^+ \in \RR^{r \times r}$ (where $S$ is the column selection matrix, i.e. $KS$ comprises $r$ columns of $K$), $n$ data points of dimension $d$ as well as other $n$-dimensional vectors, the total space required is $O(nr + r^2 + nd) = O(nr + nd)$.
\end{proof}
\begin{figure}[t!]
    \centering
    \includegraphics[width=1\linewidth]{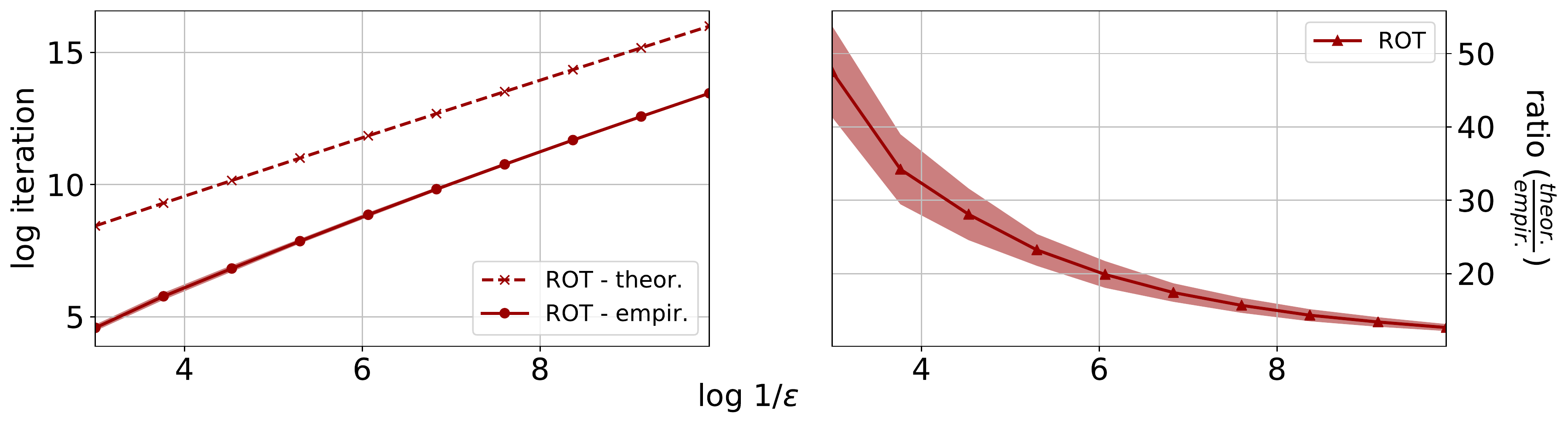}  
    \vspace*{-5mm}
    \caption{Complexity demonstration for \textsc{RobustSinkhorn} on synthetic data. All the plots presented in this figure are set up similarly to those in Figure \ref{figure:rot:runtime}.}
    \label{figure:rot:synthetic:runtime}
\end{figure}

\noindent Now we take a look at the cases of RSOT and ROT. In particular, we derive the upper bounds for $S_x$, $S_g$, $S_H$ and $S_C$.
\subsection{Robust Unbalanced Optimal Transport} 
In this case, the constants are
\begin{align*}
    S_x = 1, S_g = \frac{3\left( \tau+2\right)}{4(\tau+1)}, S_H = 2 \log (n), S_C = \frac{2\tau + \eta}{\eta^2}.
\end{align*}

\begin{proof}[Proofs of Inequalities]
The inequalities for $S_x, S_g$ and $S_H$ comes from the fact that the $X^k_{\widetilde{C}}$  was  normalized, inequality \eqref{eq:rot:diff_g_k_3} and inequality \eqref{eq:rot:diff_entropy}  respectively in the section \ref{appendix:rot:proofs} of ROT's proofs. Regarding to $S_C$, we have
\begin{align*}
    g_C(X_C^*) &= \tau (\alpha + \beta - 2) - \eta - (2\tau + \eta) \log (x_C^*), \\
    g_{\widetilde{C}}(X_{\widetilde{C}}^*) &= \tau (\alpha + \beta - 2) - \eta - (2\tau + \eta) \log (x_{\widetilde{C}}^*).
\end{align*}
Consequently, $\big|g_{\widetilde{C}}(X^*_{\widetilde{C}}) - g_C(X^*_C) \big| = (2\tau + \eta) \left|\log \left( \frac{x_{\widetilde{C}}^*}{x_C^*} \right) \right|$. 
\paragraph{Upper bound for $\left|\log \left( \frac{x_{\widetilde{C}}^*}{x_C^*} \right) \right|$.}

For any $u, v \in \RR^n$ and $C \in \RR^{n \times n}$, defining $B(u, v; C)$ is a matrix with entries $B(u, v; C)_{ij} = \exp \left( \frac{u_i + v_j - C_{ij}}{\eta} \right)$, we have the following lemma
\begin{lemma}
\label{lemma:nystrom:bound_delta_u}
For $\tau >0$ and $a \in \RR^n$, if $\frac{u}{\tau} = \log a - B(u, v; C) \ones_n$ and $\frac{u'}{\tau} = \log a - B(u', v'; C') \ones_n$, then
\begin{align*}
    \Big( \frac{1}{\tau} + \frac{1}{\eta} \Big) \mnorm{u' - u} \le \frac{1}{\eta} \mnorm{v' - v} + \frac{1}{\eta} \mnorm{C' - C}.    
\end{align*}
\end{lemma}
\begin{proof}[Proof of Lemma \ref{lemma:nystrom:bound_delta_u}]
Taking the difference between $u/\tau$ and $u^{\prime}/\tau$, for $i \in [n]$,
\begin{align*}
    \frac{u'_i - u_i}{\tau} = \log \left( \frac{B(u, v; C)_i}{B(u', v'; C')_i} \right) &= - \frac{u'_i - u_i}{\eta} + \log \left( \frac{\sum_j \exp \Big( \frac{v'_j - C'_{ij}}{\eta} \Big)}{\sum_j \exp \Big( \frac{v_j - C_{ij}}{\eta} \Big)} \right) \\
    &\le -\frac{u'_i - u_i}{\eta} +  \frac{\mnorm{v' - v}}{\eta} + \frac{\mnorm{C' - C}}{\eta},
\end{align*}
which results in the final statement.
\end{proof}

\noindent From the fixed-point equations for $(u^*_{C},v^*_{C})$ and $(u^*_{\widetilde{C}},v^*_{\widetilde{C}})$ and Lemma \ref{lemma:nystrom:bound_delta_u}, we have
\begin{align*}
    \Big( \frac{1}{\tau} + \frac{1}{\eta} \Big) \mnorm{u_{\widetilde{C}}^* - u_C^*} &\le \frac{1}{\eta} \mnorm{v_{\widetilde{C}}^* - v_C^*} + \frac{1}{\eta} \mnorm{\widetilde{C} - C} \\
    \Big( \frac{1}{\tau} + \frac{1}{\eta} \Big) \mnorm{v_{\widetilde{C}}^* - v_C^*} &\le \frac{1}{\eta} \mnorm{u_{\widetilde{C}}^* - u_C^*} + \frac{1}{\eta} \mnorm{\widetilde{C} - C},  
\end{align*}
leading to $\mnorm{u_{\widetilde{C}}^* - u_C^*} + \mnorm{v_{\widetilde{C}}^* - v_C^*} \le \frac{2 \tau}{\eta} \mnorm{\widetilde{C} - C}$.

\noindent Hence, we find that
\begin{align*}
    \left| \log \left( \frac{x_{\widetilde{C}}^*}{x_C^*} \right) \right| &= \left| \log \left( \frac{\sum_{i, j=1}^n \exp \Big( \frac{(u_{\widetilde{C}}^*)_i + (v_{\widetilde{C}}^*)_j - \widetilde{C}_{ij}}{\eta} \Big)}{\sum_{i, j=1}^n \exp \Big( \frac{(u_C^*)_i + (v_C^*)_j - C_{ij}}{\eta} \Big)} \right) \right| \\
    &\le \frac{1}{\eta} \mnorm{u_{\widetilde{C}}^* - u_C^*} + \frac{1}{\eta} \mnorm{v_{\widetilde{C}}^* - v_C^*} + \frac{1}{\eta} \mnorm{\widetilde{C} - C}. \\
    &\le \frac{2 \tau + \eta}{\eta^2} \mnorm{\widetilde{C} - C}.
\end{align*}
\end{proof}
\subsection{Robust Semi-Optimal Transport} 
In this case, the constants are
\begin{align*}
    S_x = 1, S_g = \log (n), S_H = 2 \log (n), S_C = \frac{2\tau + \eta}{\eta}.
\end{align*}

\begin{proof}[Proofs of Inequalities]
The inequalities regarding $S_x, S_g$ and $S_H$ comes from the fact that  $\|X^k_{\widetilde{C}}\|_1 = 1$,  inequality \eqref{proof:rsot:bound_g} and inequality \eqref{proof:rsot:bound_h} of Section \ref{appendix:rsot}, respectively. In terms of $S_C$, from equation \eqref{eq:rsot_v_b_s}  we have
\begin{align*}
    g_C(X_C^*) = -\eta - \tau(1-\alpha)+ \langle v_C^*, b^* \rangle, \qquad g_{\widetilde{C}}(X_{\widetilde{C}}^*) = -\eta - \tau(1-\alpha) + \langle v_{\widetilde{C}}^*, b^* \rangle.
\end{align*}
Recall that it is the RSOT problem, thus $b^* = (\Xrsots)^{\top} \ones_n =\bB$, thus  $$\big|g_{\widetilde{C}}(X^*_{\widetilde{C}}) - g_C(X^*_C) \big| = \big| \langle v_{\widetilde{C}}^* - v_C^*, b^* \rangle \big| \le \mnorm{v_{\widetilde{C}}^* - v_C^*} \onorm{b^*} = \mnorm{v_{\widetilde{C}}^* - v_C^*}.$$
\textbf{Upper bound for $\mnorm{v_{\widetilde{C}}^* - v_C^*}$.} Defining $B(u, v; C)$ is a matrix with entries $B(u, v; C)_{ij} = \exp \left( \frac{u_i + v_j - C_{ij}}{\eta} \right)$. The fixed-points  $u_C^*$ and $u_{\widetilde{C}}^*$  satisfy the following equations
\begin{align*}
    \frac{u_C^*}{\tau} = \log a - \log B(u, v; C), \qquad \frac{u_{\widetilde{C}}^*}{\tau} = \log a - \log B(u', v'; C').
\end{align*}
By Lemma \ref{lemma:nystrom:bound_delta_u},
\begin{align}
\label{bound:nystrom:rsot:delta_u}
    \Big( \frac{1}{\tau} + \frac{1}{\eta} \Big) \mnorm{u_{\widetilde{C}}^* - u_C^*} \le \frac{1}{\eta} \mnorm{v_{\widetilde{C}}^* - v_C^*} + \frac{1}{\eta} \mnorm{\widetilde{C} - C}
\end{align}
By the fixed-point theorem, $B(u_C^*, v_C^*; C)^T \ones_n = b$ and $B(u_{\widetilde{C}}^*, v_{\widetilde{C}}^*; \widetilde{C})^T \ones_n = b$, and similarly we obtain
\begin{align}
\label{bound:nystrom:rsot:delta_v}
    \frac{1}{\eta} \mnorm{v_{\widetilde{C}}^* - v_C^*} \le \frac{1}{\eta} \mnorm{u_{\widetilde{C}}^* - u_C^*} + \frac{1}{\eta} \mnorm{\widetilde{C} - C}.
\end{align}
Combining \eqref{bound:nystrom:rsot:delta_u} and \eqref{bound:nystrom:rsot:delta_v}, we have $\mnorm{u_{\widetilde{C}}^* - u_C^*} \le \frac{2 \tau}{\eta} \mnorm{\widetilde{C} - C}$, and consequently $\mnorm{v_{\widetilde{C}}^* - v_C^*} \le \frac{2 \tau + \eta}{\eta} \mnorm{\widetilde{C} - C}$, completing the proof.
\end{proof}



\section{Additional Experiments}
\label{sec:additional_experiments}
\subsection{The Complexity of \textsc{Robust-Sinkhorn} Algorithms on Synthetic Data}
First, we investigate the runtime of Algorithm~\ref{algorithm:rot:sinkhorn} (\textsc{RobustSinkhorn}) for solving ROT, with the same synthetic setting of RSOT described in the main text (which will be repeated here for the sake of completion).
\paragraph{\textit{Synthetic Data.}} We let $n = 100, \tau = 1$, generate entries of $C$ uniformly from the interval $[1, 50]$ and draw entries $a, b$  uniformly from $[0.1, 1]$ then normalizing them to form probability vectors. $\eta$ is set according to Theorem \ref{theorem:rsot}. For each $\varepsilon$ varying from $5\times 10^{-2} $ to $5 \times 10^{-5}$, we calculate the number of theoretical and empirical iterations described above, as well as their ratio.

This experiment is run $10$ times and we report their mean and standard deviation values in Figure \ref{figure:rot:synthetic:runtime}, which shows that ROT lines experience a similar trend to those of RSOT in Section \ref{sec:experiment}, with the ratio decreasing in the direction of $\varepsilon$ toward zero.

\subsection{The Complexity of \textsc{Robust-SemiSinkhorn} and \textsc{Robust-Sinkhorn} Algorithms on Realistic Data}
\paragraph{\textit{MNIST Data.}} We consider each $28 \times 28$ MNIST image as a discrete distribution by flattening it into a $784$-dimensional vector then performing normalization. For any pair of this MNIST distribution, the distance between their support equals to the Manhattan distance between corresponding pixel locations. Here, we let $\tau = 1$ and vary $\varepsilon$ from $10^{-2}$ to $10^{-5}$ (which is relatively small compared to $\frsot(\Xrsot^*) = 1.86 \pm 0.59$ and $\frot(\Xrot^*) = 1.15 \pm 0.33$ in this setting). For each value of $\varepsilon$, the regularized parameter $\eta$ is set accordingly as presented in Theorem \ref{theorem:rsot}. The theoretical and empirical values for the number of necessary iterations, as well as their ratio, are computed similar to the synthetic case, and their mean and standard variation values over 5 random MNIST pairs are reported in Figure \ref{figure:rot:mnist:runtime}.
\begin{figure}[t!]
    \centering
    \includegraphics[width=1\linewidth]{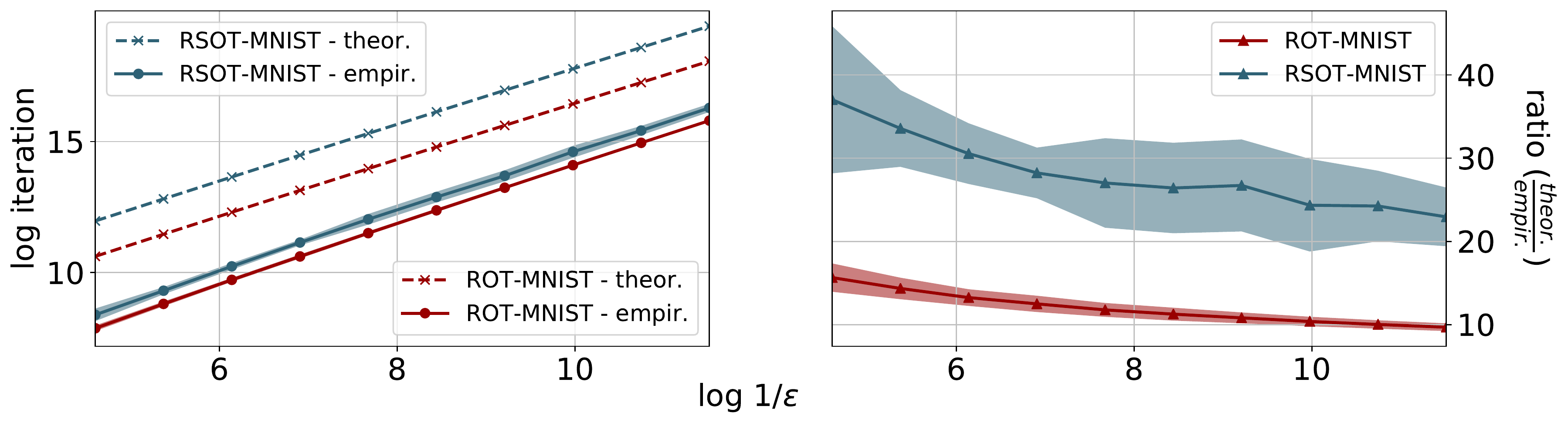}  
    \vspace*{-5mm}
    \caption{Complexity demonstration for \textsc{Robust-SemiSinkhorn} (blue) and \textsc{Robust-Sinkhorn} (red) algorithms used to compute Robust Optimal Transport between MNIST images. All the plots presented in this figure are set up similarly to those in Figure \ref{figure:rot:runtime}.}
    \label{figure:rot:mnist:runtime}
\end{figure}
It can be seen from Figure \ref{figure:rot:mnist:runtime} (compared to Figure \ref{figure:rot:runtime} and \ref{figure:rot:synthetic:runtime}) that the theory-practice relation of the two discussed algorithms (regarding the total iterations needed to reach an $\varepsilon$-approximation) behave quite similarly in both real and synthetic settings: two theoretical and empirical lines in the left plot run almost linearly while coming close to each other as $\varepsilon$ goes toward zero.


\subsection{Robust Comparison between Different Formulations}
In this section, we compare the marginals induced by using different variants of optimal transport in the presence of corrupted measures. With the setting described in Figure \ref{figure:rot_and_uot}, four following formulations are considered:
\begin{figure}[t!]
    \centering
    \includegraphics[width=1\linewidth]{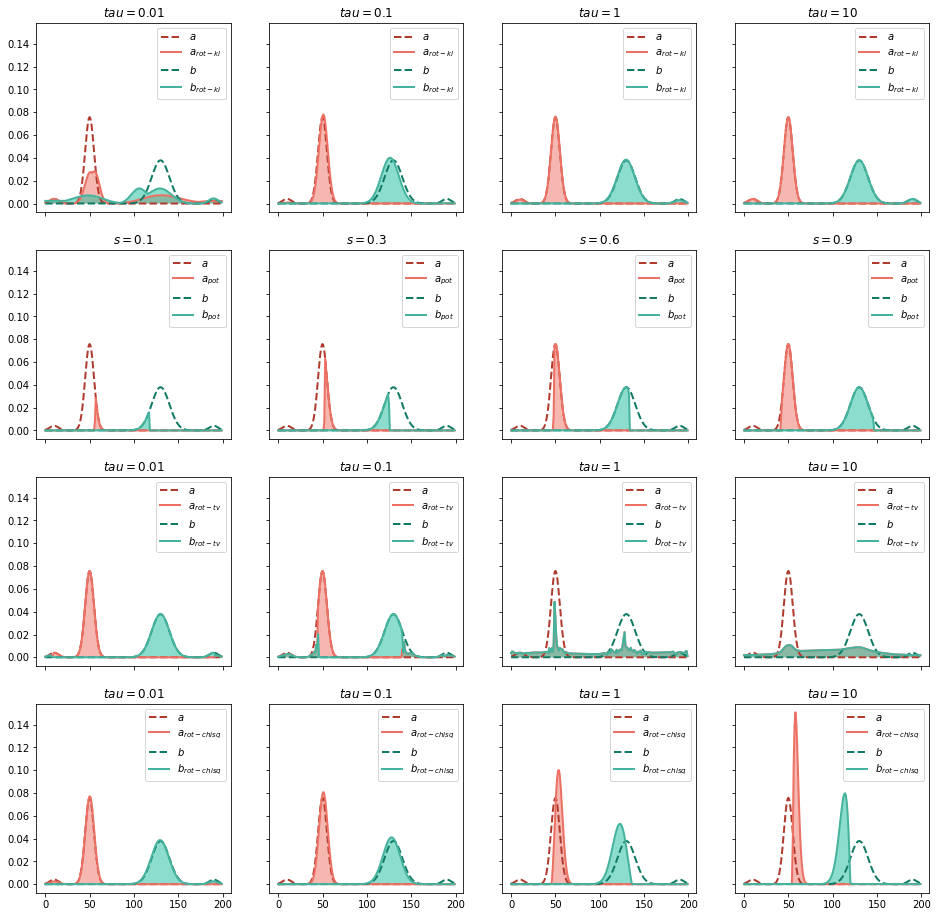} 
    \caption{Comparison between robust optimal transport (ours, using KL divergence), partial optimal transport and robust formulations in \cite{pmlr-v139-mukherjee21a} (using total variation distance) and \cite{Balaji_Robust} (using $\chi^2$-divergence), in that order from the first row to the fourth row, with different hyperparameter settings. The experiment setup is similar to the one in Figure \ref{figure:rot_and_uot}.}
    \label{figure:robust_comparison}
\end{figure}
\begin{itemize}
    \item Robust optimal transport with KL divergence (see Problem \eqref{eq:robust_OT})
    \begin{align*}
        &\min_{X} \quad \langle C, X\rangle \\
        &\text{s.t.} \quad X \ge 0, \onorm{X} = 1, \KL(X\one_n||\aA) \le \tau, \KL(X^{\top}\one_n||\bB) \le \tau,
    \end{align*}
    \item Partial optimal transport \cite{figalli2010optimal}
    \begin{align*}
        &\min_{X} \quad \langle C, X\rangle \\
        &\text{s.t.} \quad X \ge 0, \onorm{X} = s, X\one_n \le \aA, X^{\top}\one_n \le \bB,
    \end{align*}
    \item Robust optimal transport with total variation distance \cite{pmlr-v139-mukherjee21a}
    \begin{align*}
        &\min_{X} \quad \langle C, X\rangle \\
        &\text{s.t.} \quad X \ge 0, \onorm{X} = 1, \mathbf{TV}(X\one_n, \aA) \le \tau, \mathbf{TV}(X^{\top}\one_n,\bB) \le \tau,
    \end{align*}
    \item Robust optimal transport with $\chi^2$ divergence \cite{Balaji_Robust}
    \begin{align*}
        &\min_{X} \quad \langle C, X\rangle \\
        &\text{s.t.} \quad X \ge 0, \onorm{X} = 1, \chi^2(X\one_n, \aA) \le \tau, \chi^2(X^{\top}\one_n, \bB) \le \tau.
    \end{align*}    
\end{itemize}
The results are plotted in Figure \ref{figure:robust_comparison}. It is apparent that all the variants approximate the corrupted measures well with a proper choice of hyperparameter $\tau$ or $s$, and those with $f$-divergence relaxation have different behaviors when $\tau$ goes to infinity.
\begin{figure}[t!]
    \includegraphics[width=\linewidth]{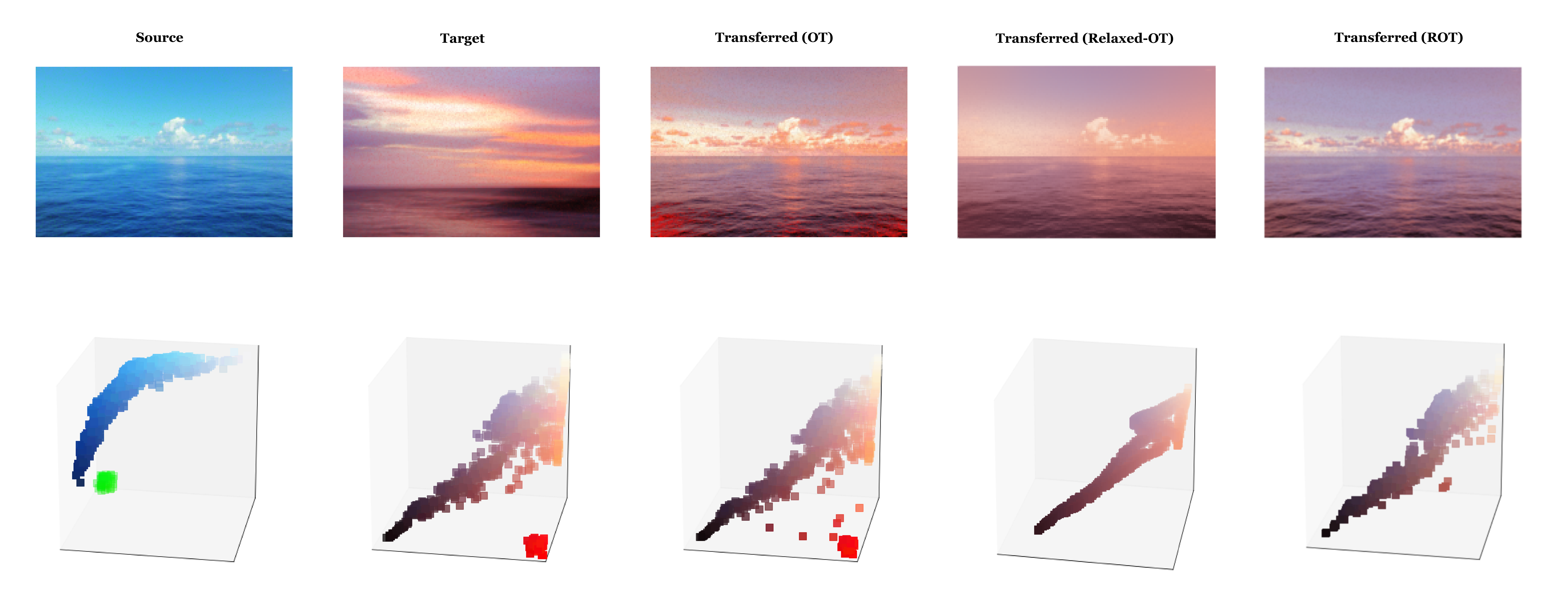}  
    \caption{Demonstration for robust color transfer. The first row, from left to right, consists of source image, target image, and three last ones that are source images with each pixel replaced by its mapped value via standard optimal transport, relaxed optimal transport \cite{rabin2014adaptive} and robust optimal transport (ours) respectively. The second row comprises corresponding (RGB) histograms of images on the first row. Note that the source (or target) image is corrupted by replacing pixels at random positions by green (or red) pixels, resulting in two green and red point clouds in corners in the first two histograms.}
    \label{figure:color_transfer}
\end{figure}
\subsection{Some Applications of Robust Optimal Transport}
In this section we demonstrate the robustness of two discussed versions of Robust Optimal Transport in two applications: color transfer and generative modeling.
\subsubsection{Color Transfer}
Here, the optimal transport problem is conducted between the histograms of two images. Considering a source RGB image of size $h_s \times w_s \times 3$, and the a target RGB image of size $h_t \times w_t \times 3$, we can present all the pixels in these images as point clouds in $3$-dimensional RGB space (see the second row in Figure \ref{figure:color_transfer}). To transfer the color from the target image into the source image, we compute the optimal transportation plan between the two corresponding point clouds and and use it to perform mapping from the source cloud to another point cloud that resembles the target cloud (i.e., transferring from the histogram in the first column to the third and fourth columns in Figure \ref{figure:color_transfer}). As the total number of pixels in source/target image is large, it is a common practice to just sample a subset of pixels from each image, namely $\Ii_{src} = \{x_1, \dots, x_n\}$ and $\Ii_{tar} = \{y_1, \dots, y_m\}$. We consider two discrete measure formed by these two point clouds, $\alpha = \sum_{i} a_i x_i$ and $\beta = \sum_{j} b_j y_i$ and let $\aA = [a_1, \dots, a_n], \bB = [b_1, \dots, b_m]$. To compute the optimal transportation plan, we solve
\begin{align*}
    \text{(for standard optimal transport)} \qquad X^* &= \argmin_{\substack{X \ones_n = \aA, \\ X^T \ones_n = \bB}} \quad \langle C, X \rangle,\\
    \text{(for robust optimal transport)} \qquad X^* &= \argmin_{\substack{X \in \RR_+^{n \times n}, \\ \onorm{X}=1}} \quad \langle C, X \rangle + \tau \KL(X \ones_n \| \aA) + \tau \KL(X^T \ones_n \| \bB ),
\end{align*}
where $C$ is the cost matrix with each entry $C_{ij} := \| x_i - y_j\|^2_2$. This optimal plan $X^*$ is then extended to cover all possible pixels using mapping estimation in \cite{perrot2016mapping}. In the experiment, we let $m = n = 1000, \tau = 1$ and $\aA, \bB$ being uniform mass vectors. Additionally, we approximate solutions of two optimal transport problems above using Sinkhorn algorithms on their entropic formulations with $\eta = 0.001$. To demonstrate the robustness when dealing with outliers in support points, we corrupt both source and target image by randomly changing their pixel intensities to other values (see Figure $\ref{figure:color_transfer}$). It can be seen from the figure that the transferred color histogram induced by the OT solution still contains noisy values (see the bottom-right of the histogram visualization on the third column), while the transferred histogram resulted from ROT is clean as expected. As a consequence, the twilight scene corresponding to OT contains red noises at corners and is not as visually appealing as its ROT counterpart. 

\begin{figure}[t!]
    \includegraphics[width=\linewidth]{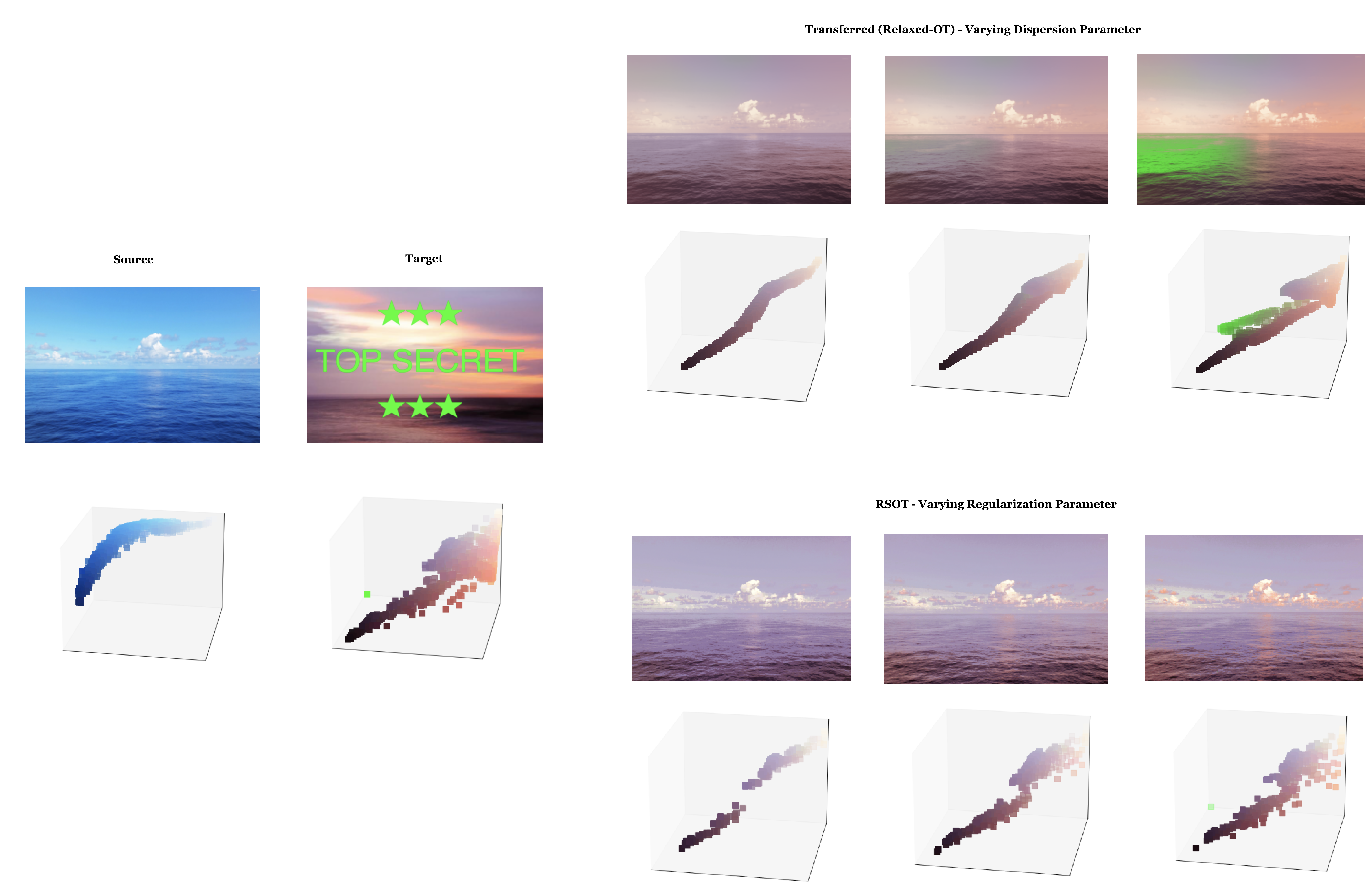}  
    \caption{Comparison between relaxed optimal transport \cite{rabin2014adaptive} and robust optimal transport in the color transfer problem. The setting is the same as in Figure \ref{figure:color_transfer}, but here the "robust" parameters of both methods are varied (from left to right, the dispersion parameter of relaxed OT is set to $0.003, 0.03, 0.3$ respectively, and the parameter $\tau$ of robust OT is set to $0.1, 1, 10$ respectively). It is noticeable that the histogram of transferred image induced by the relaxed OT is not as diverse and exact as the one produced by our robust OT, resulting in a less visually appealing output.}
    \label{figure:color_transfer_2}
\end{figure}
\begin{figure}[t!]
    \includegraphics[width=\linewidth]{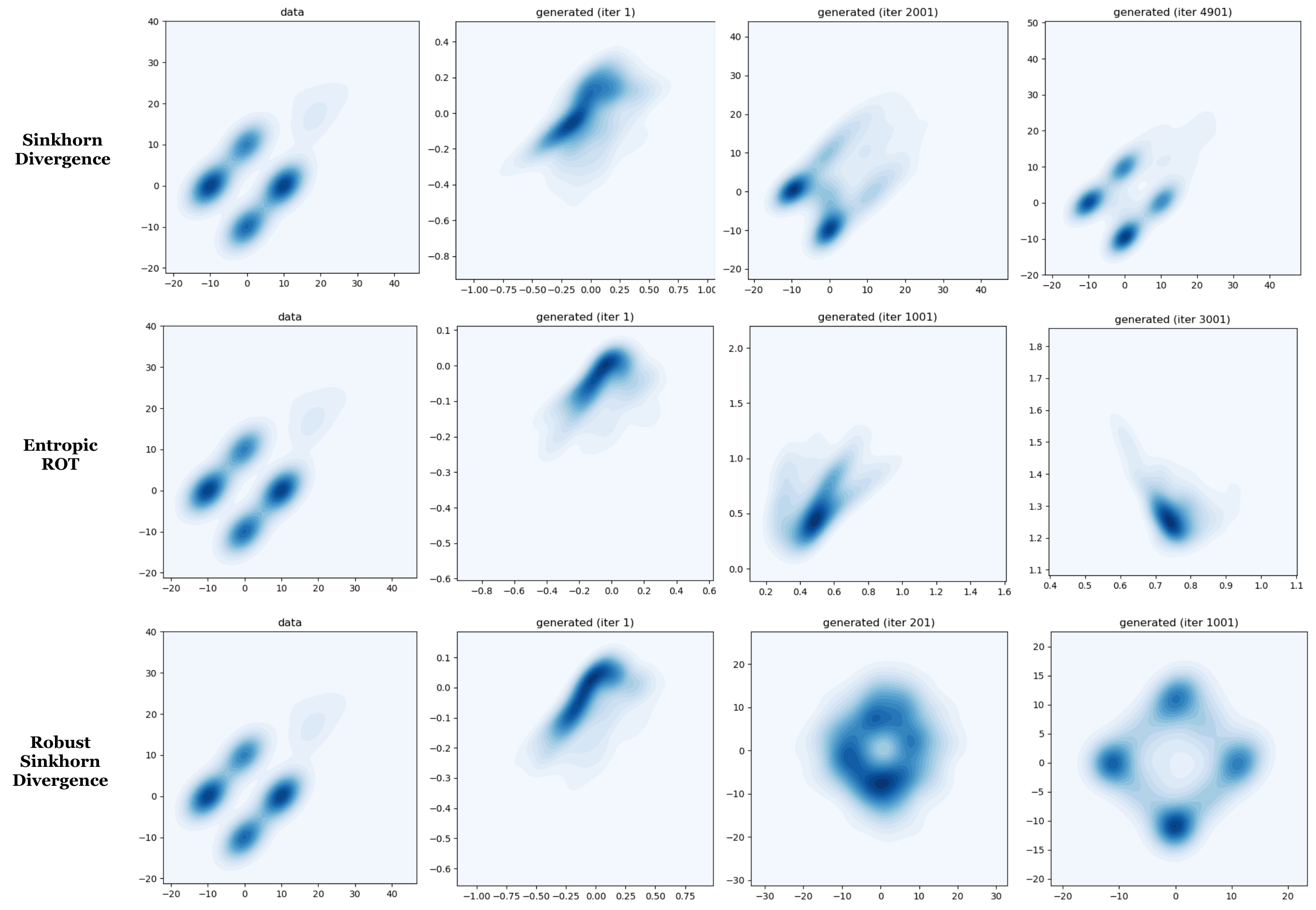}  
    \vspace*{-5mm}
    \caption{Generative modeling with three different objectives: Sinkhorn divergence (first row), entropic ROT (second row) and Robust Sinkhorn divergence (last row). In each image, we show $1000$ points created by first sampling $z \sim \Nn(\zeros_2, \mathbf{I}_2)$ then generating $x_{gen} = g_\theta(z)$. At each row, from left to right, we present generated distributions at several iterations in the chronological order.}
    \label{figure:gan:gaussian}
\end{figure}
\subsubsection{Generative Modeling}
Next, we utilize the robust formulation of optimal transport in the problem of generative modeling. Assume that we have finite samples from a data distribution, which are $x_1, \dots, x_n \sim p_{\text{data}}(x)$, the goal is to find a parametric mapping from a latent space $\Zz$ to the data space $\Xx$, namely $g_\theta: \Zz \to \Xx$, so that the pushforward measure $g_{\theta \# p_\Zz}$ is close to the data distribution $p_{\text{data}}$  as much as possible. This problem can be formulated as to find $\theta^* = \argmin_{\theta} \Dd(p_{\text{data}}, g_{\theta \# p_z})$, where $\Dd$ is a divergence between probability measures. Usually, $p_\Zz$ is taken to be a simple distribution that we can easily sample from, such as an isotropic Gaussian distribution $\Nn(\zeros, \mathbf{I})$, and the divergence $\Dd(p_{\text{data}}, g_{\theta \# p_z})$ is approximated via samples from two distributions, i.e. by $\Dd(\alpha, \beta)$ where $\alpha$ and $\beta$ are two discrete measures supported on $n$ data samples $\{x_i\}_{i=1}^n$ and $m$ generated samples $\{g(z_i): z_i \sim p_\Zz(z)\}_{i=1}^m$, with probability histograms $\aA$ and $\bB$ respectively. We consider three versions of $\Dd$, which are
\begin{itemize}
    \item Sinkhorn divergence in \cite{genevay2018learning}, which reads
    \begin{align*}
        SD_{\eta}(\alpha, \beta) = \Ww_\eta(\alpha, \beta) - \frac{1}{2} \Ww_\eta(\alpha, \alpha) - \frac{1}{2} \Ww_\eta(\beta, \beta),
    \end{align*}
    where $\Ww_\eta(\alpha, \beta)$ is the Wasserstein distance, a special case of optimal transport where the cost comes from a metric,
    \item Entropic robust unconstrained optimal transport in Section \ref{sec:rot}, i.e.
    \begin{align*}
        ROT_\eta(\alpha, \beta) = \min_{\substack{X\in\br^{n\times n},\\\onorm{X} = 1}} \langle C, X \rangle + \tau \KL (X \ones_n \| \aA) + \tau \KL (X^{\top} \ones_n \| \bB) - \eta H(X),
    \end{align*}
    \item Robust Sinkhorn divergence inspired from the above Sinkhorn divergence, which has the form
    \begin{align*}
        RSD_\eta(\alpha, \beta) = ROT_\eta(\alpha, \beta) - \frac{1}{2} ROT_\eta(\alpha, \alpha) - \frac{1}{2} ROT_\eta(\beta, \beta).
    \end{align*}
\end{itemize}

We train different generators corresponding to three different objectives, which are based on three variants of $\Dd$ listed above. Consider that data comes from a mixture of isotropic, two-dimensional Gaussians with four modes located at $(10, 0), (0, 10), (-10, 0)$ and $(0, -10)$. To demonstrate robustness, we corrupt the data by letting $10\%$ of them come from the uniform distribution on $[20, 25]$. We parameterize $g_\theta$ by a fully-connected neural network ($2 \to 64 \to \text{LeakyReLU} \to 128 \to \text{LeakyReLU} \to 2$), and minimize the objective via stochastic gradient descent, where $D(\alpha, \beta)$ at each iteration is computed by sampling a batch of data and generated samples then running $k$ Sinkhorn updates. We set $\eta = 100, \tau = 1, k = 10, \Zz \equiv \RR^2$ and use Adam optimizer \cite{kingma2014adam} with a learning rate of $0.001$. The generated distributions during the training process in three cases of interest are reported in Figure \ref{figure:gan:gaussian}. As shown in this figure, the objective derived from robust optimal transport can help the generator learn to ignore outliers in data distribution (see the third row), while the model based on standard optimal transport still generates noises (see the first row).

In addition to the simple Gaussian setting, we also demonstrate the generative capacity of robust optimal transport on the contaminated set of real MNIST images. Particularly, the dataset is $10\%$-corrupted by random image noises uniformly drawn from $[0, 1]^{28 \times 28}$. The generator is a fully-connected neural network mapping from 16-d Gaussian to $[0, 1]^{784}$ (the full architecture is $16 \to 500 \to \text{Softplus} \to 500 \to \text{Softplus} \to 784 \to \text{Sigmoid}$). We train this network with the same procedure described in the previous paragraph, using the normal and the robust formulation of Sinkhorn divergence as the objective. The generated images are shown in Figure~\ref{figure:gan:mnist}. As expected, while the network trained with the standard Sinkhorn divergence still generates noises (appearing as a mixed version of a MNIST image and a noise image), the network learned with the robust optimal transport ignores the noise and only produce clean digit pictures.

\begin{figure}[t!]
    \includegraphics[width=\linewidth]{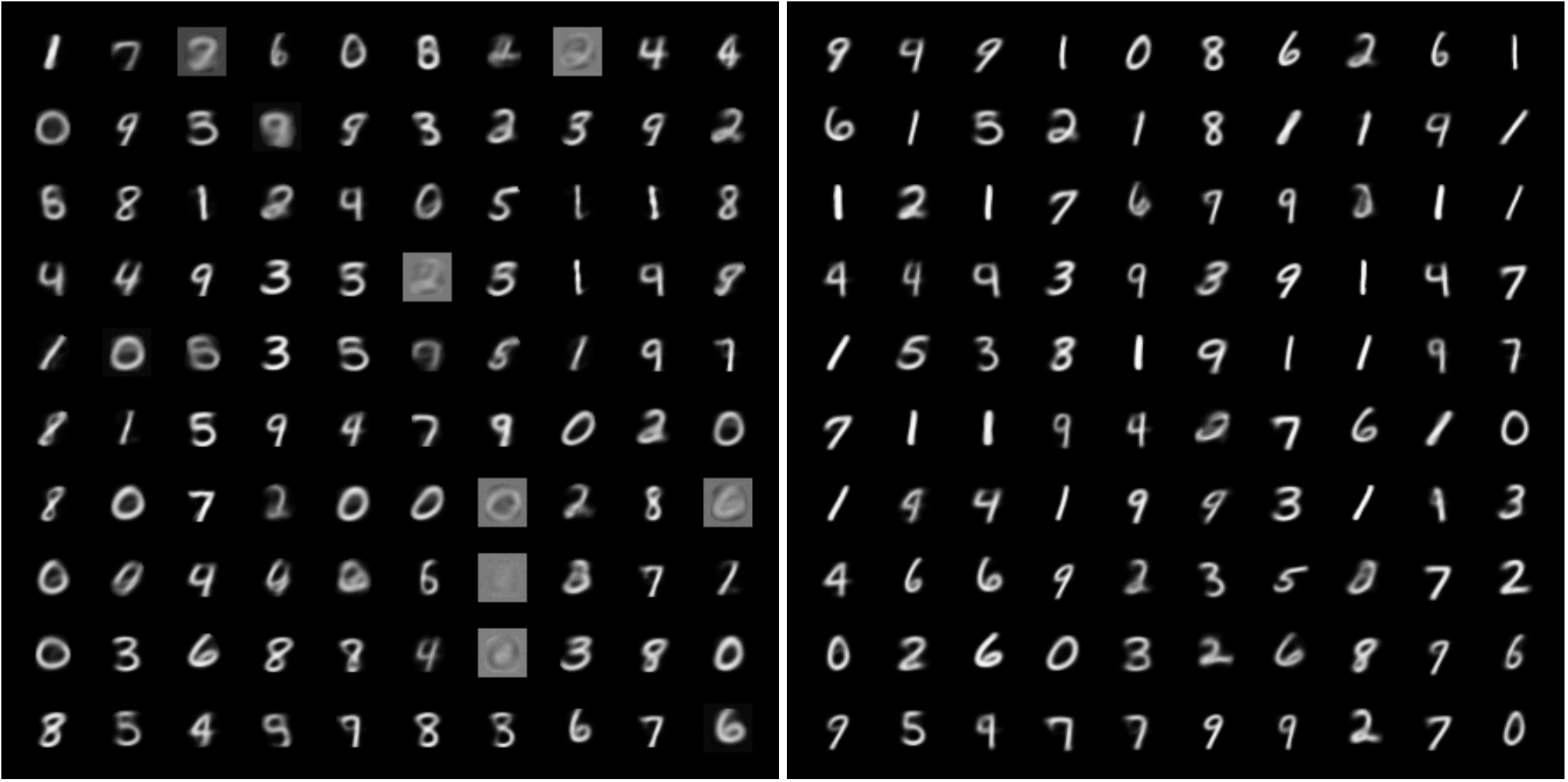}  
    \vspace*{-5mm}
    \caption{Generating contaminated MNIST data. The left and the right figures are the outputs of the generator trained with Sinkhorn divergence and with robust Sinkhorn divergence respectively.}
    \label{figure:gan:mnist}
\end{figure}

\bibliography{refs}

\end{document}